\documentclass{article}

\usepackage[preprint]{neurips_2025}

\usepackage[utf8]{inputenc} 
\usepackage[T1]{fontenc}    
\usepackage{hyperref}       
\usepackage{url}            
\usepackage{booktabs}       
\usepackage{amsfonts}       
\usepackage{nicefrac}       
\usepackage{microtype}      
\usepackage{xcolor}         
\usepackage{subcaption}

\usepackage[bbgreekl]{mathbbol}
\usepackage{amsmath}
\usepackage{yhmath}
\usepackage{stmaryrd}
\usepackage{amssymb}
\usepackage{mathtools}
\usepackage{amsthm}
\usepackage{pgfplots}
\usepgfplotslibrary{colorbrewer}
\usepgfplotslibrary{fillbetween}
\usepgfplotslibrary{groupplots}
\usepackage{makecell}
\usepackage{bbm}
\usepackage{enumerate}
\usepackage{comment}

\definecolor{mycolor1}{rgb}{0, 1, 0}
\definecolor{mycolor2}{rgb}{0.157, 0.784, 0.392}
\definecolor{mycolor3}{rgb}{0.294, 0.706, 0.667}
\definecolor{mycolor4}{rgb}{0.392, 0.627, 0.981}
\definecolor{mycolor5}{rgb}{0.627, 0.392, 0.981}
\definecolor{mycolor6}{rgb}{0.706, 0.294, 0.667}
\definecolor{mycolor7}{rgb}{0.784, 0.157, 0.392}
\definecolor{mycolor8}{rgb}{1, 0, 0}

\newcommand{\R}{\mathbb{R}}
\newcommand{\N}{\mathbb{N}}

\newcommand{\Ell}{\mathcal{L}}

\newcommand{\Id}{\mathcal{I}}

\newcommand\scalemath[2]{\scalebox{#1}{\mbox{\ensuremath{\displaystyle #2}}}}

\DeclareMathOperator{\asd}{\scriptstyle{(\otimes)}}

\DeclareMathOperator{\tr}{tr}
\DeclareMathOperator{\sgn}{sgn}

\newtheorem{proposition}{Proposition}

\newtheorem{theorem}[proposition]{Theorem}

\theoremstyle{definition}
\newtheorem{definition}[proposition]{Definition}
\theoremstyle{remark}

\title{Feature Learning Beyond the Edge of Stability}

\author{%
  D\'avid Terj\'ek \\
  Alfr\'ed R\'enyi Institute of Mathematics \\
  Budapest, Hungary \\
  \texttt{dterjek@renyi.hu} \\
}

\begin{document}

\maketitle

\begin{abstract}
We propose a homogeneous multilayer perceptron parameterization with polynomial hidden layer width pattern and analyze its training dynamics under stochastic gradient descent with depthwise gradient scaling in a general supervised learning scenario. We obtain formulas for the first three Taylor coefficients of the minibatch loss during training that illuminate the connection between sharpness and feature learning, providing in particular a soft rank variant that quantifies the quality of learned hidden layer features. Based on our theory, we design a gradient scaling scheme that in tandem with a quadratic width pattern enables training beyond the edge of stability without loss explosions or numerical errors, resulting in improved feature learning and implicit sharpness regularization as demonstrated empirically.
\end{abstract}

\section{Introduction}\label{intro}
While gradient descent in quadratic optimization converges if and only if the learning rate $\eta$ is chosen such that it satisfies the strict inequality $\frac{1}{2} \eta S < 1$ in relation to the sharpness $S$ (i.e., the largest eigenvalue of the Hessian), it was shown in the seminal work of \citet{Cohenetal2021} that this is not the case for training neural networks, which was not the first time deep learning contradicted classical intuition. In particular, \citet{Cohenetal2021} demonstrated that when a neural network is trained with gradient descent, even if this requirement of stability is met at the beginning, there is progressive sharpening as the sharpness (which depends not only on the loss function, but on the neural network as well) rises until reaching the threshold $\frac{1}{2} \eta S = 1$ where it starts to oscillate, which is referred to as the Edge of Stability (EOS). Since then, many works contributed to the understanding of this phenomenon \citep{Aroraetal2022,Damianetal2023,Cohenetal2025} and its analogue for stochastic gradient descent (SGD) \citep{Wuetal2023,Leeetal2023,Andreyevetal2025}. A related line of research \citep{Lietal2019b,Lewkowyczetal2020,Mohtashamietal2022} concerns training with large learning rates that are unstable already at initialization, leading to better feature learning and generalization performance.

These developments led us to investigate the limits of how large we can push the learning rate to benefit from its implicit regularization effects without divergent training. To this end, we consider a multilayer perceptron (MLP) parameterization with a polynomial layer width pattern and a homogeneous activation trained by SGD with layerwise gradient scaling in a general supervised learning scenario. Without any assumptions that would deviate from practice, we express the first $3$ Taylor coefficients of the minibatch loss during training in terms of a few vectors, a matrix and a third order tensor, all of dimension $l+1$ where $l$ is the number of hidden layers in the MLP. The components of these tensors quantify the scale, the internal structure and the pairwise alignment of minibatches of certain high dimensional tensors, such as activations and gradients of the loss with respect to preactivations. Offering a glimpse into the rich interplay between feature learning and sharpness, our theory lets us design a gradient scaling scheme that is easy to compute and demonstrably leads to significant improvements in large learning rate training when used in tandem with the quadratic hidden layer width pattern of our MLP parameterization.

The organization of the rest of the paper is as follows. We conclude \S~\ref{intro} by discussing related works in \S~\ref{related} and listing our contributions in \S~\ref{contribs}. In \S~\ref{main}, we introduce our MLP parameterization, describe the supervised learning scenario with scaled SGD, present and discuss our main result concerning the Taylor coefficients and then propose the gradient scaling formula and demonstrate its empirical benefits. We conclude by discussing the limitations of our work in \S~\ref{conclusion} along with future directions.

\subsection{Related work}\label{related}
The maximum eigenvalue of the Hessian of the loss is referred to as the sharpness. It is well known that during gradient descent (GD), if the learning rate $\eta>0$ is chosen such that sharpness does not exceed $\frac{2}{\eta}$, then the loss decreases monotonically under mild assumptions. \citet{Cohenetal2021} showed that when neural networks are trained using gradient descent, the progressive sharpening phenomenon is present in the sense that sharpness grows until reaching the Edge of Stability (EOS) regime. At the EOS, sharpness hovers above $\frac{2}{\eta}$ while the loss decreases over time in a non-monotonic manner. \citet{Aroraetal2022} presented a theoretical analysis of two settings where GD provably reaches the EOS and then starts decreasing the sharpness by moving in a direction tangent to the manifold of minimizers. \citet{Damianetal2023} theoretically showed that in a simplified setting, gradient descent implicitly regularizes sharpness when it exceeds $\frac{2}{\eta}$, arguing that an analogous mechanism is behind the EOS. Generalizing these arguments, \citet{Cohenetal2025} studied training at the EOS by proposing central flows, which capture the long term behavior of the loss and are therefore resilient to oscillations.

In the kernel regime, \citet{Lewkowyczetal2020} proposed a theory showing that training with large learning rates such that sharpness exceeds $\frac{2}{\eta}$ already at initialization, a catapult phase can be observed instead of lazy training, inducing feature learning and reducing the sharpness below $\frac{2}{\eta}$, enabling the convergence of training after this initial phase. \citet{Zhuetal2023} showed that these catapults occur in the subspace spanned by the top eigenvectors of the neural tangent kernel and proposed an explanation of how this phenomenon leads to better generalization. \citet{Nocietal2024} found that spectral properties of the Hessian such as the sharpness depend heavily on width in the kernel regime but are largely independent of width in the rich regime, suggesting that the benefits of training at the EOS diminish for wide networks in the kernel regime but not in the rich regime.

\citet{Lietal2019b} proved that in a simplified setting, shallow neural networks generalize better when trained with large learning rates instead of small ones. \citet{Mohtashamietal2022} showed that in some nonconvex optimization problems, GD with a large stepsize follows a trajectory that is qualitatively different than with a small stepsize, escaping local minima and converging instead to a global one.

\citet{Wuetal2023} studied the dynamical stability of SGD focusing on the trace and the Frobenius norm of the Hessian, showing that the implicit regularization effect of SGD is stronger for larger learning rates. \citet{Leeetal2023} and \citet{Andreyevetal2025} showed that if a neural network is trained with SGD instead of full batch GD, the observations of \citet{Cohenetal2021} hold with a different notion of sharpness (referred to as the interaction aware sharpness by \citet{Leeetal2023} and as the batch sharpness by \citet{Andreyevetal2025}). \citet{Leeetal2023} studied the joint scaling of the learning rate and the batch size, while \citet{Andreyevetal2025} investigated the relationship between the usual notion of sharpness and the batch sharpness.

\subsection{Contributions}\label{contribs}
We propose
\begin{itemize}
\item a rich regime MLP parameterization with homogeneous activation functions that can exploit polynomial hidden layer width patterns,
\item an expression of the first three Taylor coefficients of the minibatch loss in terms of a few low dimensional tensors that sheds light on the relationship between sharpness and feature learning and
\item a gradient scaling scheme that is cheap to compute and enables training far beyond the EOS when used in tandem with a quadratic width pattern, demonstrably leading to better feature learning and implicit sharpness regularization.
\end{itemize}

\section{Feature learning beyond the Edge of Stability}\label{main}

Consider the following MLP parameterization. Let $l \in \N+1$ be the number of hidden layers, $\R^{m_0}$ the input space and $\Theta = \Theta_{1:l+1} = \prod_{k=1}^{l+1} \Theta_k$ the parameter space with parameter subspaces $\Theta_k = \R^{m_k \times m_{k-1}}$, input dimension $m_0 \in \N+1$, hidden layer dimensions $m_k = (l-k+1)^r m$ for $k \in [1:l]$ with width scale $m \in \N+1$ and width exponent $r \in \N$ and output dimension $m_{l+1} \in \N+1$. We denote parameters as $\theta = \theta_{1:l+1} = [\theta_k : k \in [1:l+1]] \in \Theta$ with layer matrices $\theta_k \in \Theta_k$. Finally, let $a,b \in \R$ and $\phi : \R \to \R$ be the $(a,b)$-ReLU.

\begin{definition}[$(a,b)$-ReLU]
Given $a,b \in \R$, define the $(a,b)$-ReLU $\phi : \R \to \R$ for all $s \in \R$ as $\phi(s) = as + b\vert s \vert$, so that $\phi'(s) = a + b \sgn(s)$ for all $s \in \R \setminus \{ 0 \}$.
\end{definition}
Unless $b=0$, $\phi$ is not differentiable at $s=0$, but any function $\psi : \R \to \R$ such that $\psi(s) = a + b \sgn(s)$ for all $s \in \R \setminus \{ 0 \}$ and $\psi(0) \in [a-b,a+b]$ can serve as its derivative in some suitable generalized sense. By abuse of notation, we define $\phi': \R \to \R$ as $\phi'(s) = a + b \sgn(s)$ for all $s \in \R$, so that $\phi'(0) = a$. Since the set of $s \in \R$ for which $\phi$ is not twice differentiable at $s$ is of Lebesgue measure $0$ and we initialize with a Gaussian distribution, we assume that not hitting such an $s$ during training has probability $1$ and gloss over this technical detail.

We initialize the initial parameter $\theta_0 = [\theta_{0,k} : k \in [1:l+1]] \in \Theta$ as $\theta_0 \sim \mathcal{N}( 0,\sigma^2 m^{-1} \Id_{\Theta} )$ with $\sigma = (a^2 + b^2)^{-\frac{1}{2}}$ at the Edge of Chaos (EOC) by \citet[Lemma~3]{Hayouetal2019}.

The hyperparameters above and the following forward pass defines the MLP $N : \R^{m_0} \times \Theta \to \R^{m_{l+1}}$.

\begin{definition}[Forward pass]
Given an input $x \in \R^{m_0}$ and a parameter $\theta \in \Theta$, define the forward vectors $f_{k,x,\theta} \in \R^{m_k}$ for $k \in [0:l]$ and the output $N_{x,\theta} = N(x,\theta) \in \R^{m_{l+1}}$ recursively as 
\[
f_{0,x} = x,
f_{k,x,\theta} = m_k^{-\frac{1}{2}} \phi\left( m^{\frac{1}{2}} \theta_k f_{k-1,x,\theta} \right) \text{ for } k \in [1:l] \text{ and }
N_{x,\theta} = \theta_{l+1} f_{l,x,\theta}.
\]
\end{definition}
Note that we have $f_{k,x,\theta} = m_k^{-\frac{1}{2}} D_{\phi'( m^{\frac{1}{2}} \theta_k f_{k-1,x,\theta} )} m^{\frac{1}{2}} \theta_k f_{k-1,x,\theta}$ for $k \in [1:l]$.

The naming below is going to be justified soon.
\begin{definition}[Normalized Update Parameterization ($\nu$P)]
Given $l \in \N+$, $m,m_0,m_{l+1} \in \N+1$, $r \in \N$ and $a,b \in \R$, we refer to the MLP parameterization detailed above as the Normalized Update Parameterization or $\nu$P.
\end{definition}
Note that for $r=0$, $\nu$P is equivalent to the celebrated $\mu$P of \citet{Yangetal2021} at the EOC.

Let $Y$ be a target space and $\ell : \R^{m_{l+1}} \times Y \to \R$ a loss function such that $\ell(\cdot,y) : \R^{m_{l+1}} \to \R$ is differentiable for all $y \in Y$. Denote the gradient of $\ell(z,y)$ with respect to $z \in \R^{m_{l+1}}$ as $\nabla \ell(z,y) \in \R^{m_{l+1}}$ for all $y \in Y$.

\begin{definition}[Backward pass]
Given an input $x \in \R^{m_0}$, an output $y \in Y$ and a parameter $\theta \in \Theta$, define the backward vectors $b_{k,x,y,\theta} \in \R^{m_k}$ for $k \in [1:l+1]$ recursively as 
\[
b_{l+1,x,y,\theta} = \nabla \ell(N_{x,\theta}, y) \text{ and }
b_{k,x,y,\theta} = m_k^{-\frac{1}{2}} D_{\phi'\left( m^{\frac{1}{2}} \theta_k f_{k-1,x,\theta} \right)} m^{\frac{1}{2}} \theta_{k+1}^* b_{k+1,x,y,\theta} \text{ for } k \in [l:1].
\]
\end{definition}
Note that we have $b_{k,x,y,\theta} = \nabla_{\theta_k f_{k-1,x,\theta}} \ell(N_{x,\theta}, y)$ for $k \in [1:l]$.

We are going to analyze the following supervised learning problem. Let $\{ (x_{i,t}, y_{i,t}) : i \in [1:n] \}$ for $t \in \N$ be a sequence of minibatches of size $n \in \N + 1$ drawn i.i.d. from a dataset $\mu \in \mathcal{P}(\R^{m_0} \times Y)$ and define the minibatch losses $\Ell_t : \Theta \to \R$ as 
\[
\Ell_t(\theta) = \frac{1}{n} \sum_{i=1}^n \ell(N_{x_i,\theta}, y_i) \text{ for } t \in \N \text{ and } \theta \in \Theta.
\]
For a learning rate $\eta > 0$ and gradient scales $\xi_{k,t} > 0$ for $k \in [1:l+1]$ and $t \in \N$, we perform SGD on an initial parameter $\theta_0 = [\theta_{k,0} : k \in [1:l+1]] \in \Theta$ with layerwise scaling by letting 
\[
\theta_{t+1} = \theta_t - \eta \left[ \xi_{k,t} \nabla_{\theta_{k,t}} \Ell_t(\theta_t) : k \in [1:l+1] \right] \text{ for } t \in \N.
\]
Denoting $\Xi_t = \bigoplus_{k=1}^{l+1} \xi_{k,t} \Id_{\Theta_k} \in L(\Theta, \Theta)$, note that we have $\theta_{t+1} = \theta_t - \eta \Xi_t \nabla_{\theta_t} \Ell_t(\theta_t)$, i.e., the above amounts to preconditioned SGD.

For all $t \in \N$ and $i \in [1:n]$, denote $f_{k,i,t} = f_{k,x_i,\theta_t}$ for $k \in [0:l]$, $N_{i,t} = N_{x_i,\theta_t}$, $\ell_{i,t} = \ell(N_{i,t}, y_i)$ and $b_{k,i,t} = b_{k,x_i,y_i,\theta_t}$ for $k \in [1:l+1]$. Denoting the preactivation coordinates $\mathbbm{f}_{k,i,t,j} = \langle m^{\frac{1}{2}} \theta_{k,t,j}, f_{k-1,i,t} \rangle$ and the prederivative coordinates $\mathbbm{b}_{k,i,t,j} = \langle m^{\frac{1}{2}} \theta_{k+1,t,j}^*, b_{k+1,i,t} \rangle$ for $k \in [1:l]$, $i \in [1:n]$, $t \in \N$ and $j \in [1:m_k]$, note that $m_k^{\frac{1}{2}} f_{k,i,t,j} = \phi'(\mathbbm{f}_{k,i,t,j}) \mathbbm{f}_{k,i,t,j}$ and $m_k^{\frac{1}{2}} b_{k,i,t,j} = \phi'(\mathbbm{f}_{k,i,t,j}) \mathbbm{b}_{k,i,t,j}$ for $k \in [1:l]$. Denoting the initial parameter rows $\bbtheta_{k,j} = m^{\frac{1}{2}} \theta_{k,0,j} \in \R^{m_{k-1}}$ and columns $\bbtheta_{k,j}^* = m^{\frac{1}{2}} \theta_{k+1,0,j}^* \in \R^{m_{k+1}}$ for $k \in [1:l]$ and $j \in [1:m_k]$ and the learned rows $\mathbb{\Delta}_{k,t,j} = m_k^{\frac{1}{2}} \Delta_{k,t,j} \in \R^{m_{k-1}}$ and columns $\mathbb{\Delta}_{k,t,j}^* = m_k^{\frac{1}{2}} \Delta_{k+1,t,j}^* \in \R^{m_{k+1}}$ for $k \in [1:l]$, $t \in \N$ and $j \in [1:m_k]$ where we denoted the cumulative updates $\Delta_{k,t} = \eta^{-1} (\theta_{k,t} - \theta_{k,0}) \in \Theta_k$ for $k \in [1:l+1]$ and $t \in \N$, we then have
\[
\scalemath{0.87}{
\mathbb{\Delta}_{k,t,j}
= \sum_{t'=0}^{t-1} \xi_{k,t'} \frac{1}{n} \sum_{i'=1}^n \phi'(\mathbbm{f}_{k,i',t',j}) \mathbbm{b}_{k,i',t',j} f_{k-1,i',t'},
\quad
\mathbb{\Delta}_{k,t,j}^*
= \sum_{t'=0}^{t-1} \xi_{k+1,t'} \frac{1}{n} \sum_{i'=1}^n \phi'(\mathbbm{f}_{k,i',t',j}) \mathbbm{f}_{k,i',t',j} b_{k+1,i',t'},
}
\]
\[
\scalemath{0.95}{
\mathbbm{f}_{k,i,t,j}
= \langle \bbtheta_{k,j}, f_{k-1,i,t} \rangle - (l-k+1)^{-\frac{r}{2}} \eta \langle \mathbb{\Delta}_{k,t,j}, f_{k-1,i,t} \rangle
\text{ and }
}
\]
\[
\scalemath{0.95}{
\mathbbm{b}_{k,i,t,j}
= \langle \bbtheta_{k,j}^*, b_{k+1,i,t} \rangle - (l-k+1)^{-\frac{r}{2}} \eta \langle \mathbb{\Delta}_{k,t,j}^*, b_{k+1,i,t} \rangle.
}
\]
This justifies the naming of $\nu$P as the contribution of the learned rows and columns to the evolution of preactivations and prederivatives is normalized according to depth, with the most dampening in the first hidden layer and no dampening in the final hidden layer.

For $i \in [1:n]$ and $t \in \N$, define the cross layer adjoint Jacobians 
\[
\scalemath{0.94}{
j_{k_1,k_2,i,t} = (\partial_{\theta_{k_1,t} f_{k_1-1,i,t}} f_{k_2,i,t})^* \in \R^{m_{k_1} \times m_{k_2}}
\text{ for } k_1 \leq k_2 \in [1:l],
}
\]
the cross layer Hessians
\[
\scalemath{0.94}{
h_{k_1,k_2,i,t} = \partial_{\theta_{k_2,t} f_{k_2-1,i,t}} \nabla_{\theta_{k_1,t} f_{k_1-1,i,t}} \ell_{i,t} \in \R^{m_{k_1} \times m_{k_2}} \text{ for } k_1,k_2 \in [1:l+1]
}
\]
and the cross layer Tressians 
\[
\scalemath{0.94}{
z_{k_1,k_2,k_3,i,t} = \partial_{\theta_{k_3,t} f_{k_2-1,i,t}} \partial_{\theta_{k_2,t} f_{k_2-1,i,t}} \nabla_{\theta_{k_1,t} f_{k_1-1,i,t}} \ell_{i,t} \in \R^{m_{k_1} \times m_{k_2} \times m_{k_3}}
\text{ for } k_1,k_2,k_3 \in [1:l+1].
}
\]
Denoting $\nabla_t = [\nabla_{\theta_{k,t}} \Ell_t(\theta_t) : k \in [1:l+1]] \in \Theta$, $\partial\nabla_t = [\partial_{\theta_{k_2,t}}\nabla_{\theta_{k_1,t}} \Ell_t(\theta_t) : k_1,k_2 \in [1:l+1]] \in \Theta^{\otimes 2}$ and $\partial\partial\nabla_t = [\partial_{\theta_{k_3,t}}\partial_{\theta_{k_2,t}}\nabla_{\theta_{k_1,t}} \Ell_t(\theta_t) : k_1,k_2,k_3 \in [1:l+1]] \in \Theta^{\otimes 3}$, by a third order Taylor approximation of $\Ell_t$ at $\theta_t$ we have $\Ell_t(\theta_{t+1}) - \Ell_t(\theta_t) \approx \dot{\Ell}_t$ with
\[
\scalemath{0.95}{
\dot{\Ell}_t
= - \eta \langle \Xi_t \nabla_t, \nabla_t \rangle + \frac{1}{2} \eta^2 \langle (\Xi_t \nabla_t)^\otimes 2, \partial\nabla_t \rangle - \frac{1}{6} \eta^3 \langle (\Xi_t \nabla_t)^\otimes 3, \partial\partial\nabla_t \rangle.
}
\]
We refer to $S_t = \frac{\langle (\Xi_t \nabla_t)^\otimes 2, \partial\nabla_t \rangle}{\langle \Xi_t \nabla_t, \nabla_t \rangle}$ as the effective sharpness and $\frac{1}{2} \eta S_t = 1$ as the effective EOS. For unscaled SGD, $S_t$ is referred to as the interaction aware sharpness by \citet{Leeetal2023} and as the batch sharpness by \cite{Andreyevetal2025}. Note that while $- \eta \langle \Xi_t \nabla_t, \nabla_t \rangle + \frac{1}{2} \eta^2 \langle (\Xi_t \nabla_t)^\otimes 2, \partial\nabla_t \rangle > 0$ if $\frac{1}{2} \eta S_t > 1$, the loss can still decrease due to higher order terms.

For all $t \in \N$, define the Gram matrices 
\[
\scalemath{0.9}{
F_{k,t} = [\langle f_{k,i_1,t}, f_{k,i_2,t} \rangle : i_1,i_2 \in [1:n]] \text{ for } k \in [0:l],
}
\]
\[
\scalemath{0.9}{
B_{k,t} = [\langle b_{k,i_1,t}, b_{k,i_2,t} \rangle : i_1,i_2 \in [1:n]] \text{ for } k \in [1:l+1],
}
\]
\[
\scalemath{0.9}{
H_{k_1,k_2,t} = [\langle h_{k_1,k_2,i_1,t}, h_{k_1,k_2,i_2,t} \rangle : i_1,i_2 \in [1:n]] \text{ for } k_1,k_2 \in [1:l+1],
}
\]
\[
\scalemath{0.9}{
J_{k_1,k_2,t} = [\langle j_{k_1,k_2,i_1,t}, j_{k_1,k_2,i_2,t} \rangle : i_1,i_2 \in [1:n]]\text{ for } k_1 \leq k_2 \in [1:l+1] \text{ and }
}
\]
\[
\scalemath{0.9}{
Z_{k_1,k_2,k_3,t} = [\langle z_{k_1,k_2,k_3,i_1,t}, z_{k_1,k_2,k_3,i_2,t} \rangle : i_1,i_2 \in [1:n]] \text{ for } k_1,k_2,k_3 \in [1:l+1].
}
\]
\begin{definition}[Soft rank]
Given a Gram matrix $X \in \R^{n \times n}$, let $R(X) = \frac{\tr(X)^2}{\tr(X X)} \in [1,n]$.
\end{definition}
In some sense, $R(X)$ quantifies the geometry of the set of tensors the inner products of which $X$ consists of. In particular, it is maximized if all vectors have the same norm and are orthogonal to each other. As an orthogonal dataset is the easiest to learn from, it is desirable for Gram matrices of hidden layer activations to have high soft rank, giving a measure of quality for learned hidden layer features.

Abusing notation for brevity, we define for all $t \in \N$ and the appropriate depthwise indices the tensors
\[
\scalemath{0.7}{
\binom{H_{k_1,k_2,t}}{B_{k_1,t}, B_{k_2,t}}
= \left[ \langle b_{k_1,i_1,t} \otimes b_{k_2,i_2,t}, h_{k_1,k_2,i_3,t} \rangle : i_1,i_2,i_3 \in [1:n] \right] \in \R^{n \times n \times n},
}
\]
\[
\scalemath{0.7}{
\binom{F_{k_1,t},F_{k_2,t}}{F_{k_1,t}, F_{k_2,t}}
= \left[ \langle f_{k_1,i_1,t} \otimes f_{k_2,i_2,t}, f_{k_1,i_3,t} \otimes f_{k_2,i_3,t} \rangle : i_1,i_2,i_3 \in [1:n] \right] \in \R^{n \times n \times n},
}
\]
\[
\scalemath{0.7}{
\binom{J_{k_1,k_2,t}}{B_{k_1,t}, F_{k_2,t}}
= \left[ \langle b_{k_1,i_1,t} \otimes f_{k_2,i_2,t}, j_{k_1,k_2,i_3,t} \rangle : i_1,i_2,i_3 \in [1:n] \right] \in \R^{n \times n \times n},
}
\]
\[
\scalemath{0.7}{
\binom{F_{k_1,t},B_{k_2,t}}{F_{k_1,t}, B_{k_2,t}}
= \left[ \langle f_{k_1,i_1,t} \otimes b_{k_2,i_2,t}, f_{k_1,i_3,t} \otimes b_{k_2,i_3,t} \rangle : i_1,i_2,i_3 \in [1:n] \right] \in \R^{n \times n \times n},
}
\]
\[
\scalemath{0.7}{
\binom{Z_{k_1,k_2,k_3,t}}{B_{k_1,t},B_{k_2,t},B_{k_3,t}} 
= \left[ \langle b_{k_1,i_1,t} \otimes b_{k_2,i_2,t} \otimes b_{k_3,i_3,t}, z_{k_1,k_2,k_3,i_4,t} \rangle : i_1,i_2,i_3,i_4 \in [1:n] \right] \in \R^{n \times n \times n \times n},
}
\]
\[
\scalemath{0.7}{
\binom{F_{k_1,t},F_{k_2,t},F_{k_3,t}}{F_{k_1,t}, F_{k_2,t}, F_{k_3,t}} 
= \left[ \langle f_{k_1,i_1,t} \otimes f_{k_2,i_2,t} \otimes f_{k_3,i_3,t}, f_{k_1,i_4,t} \otimes f_{k_2,i_4,t} \otimes f_{k_3,i_4,t} \rangle : i_1,i_2,i_3,i_4 \in [1:n] \right] \in \R^{n \times n \times n \times n},
}
\]
\[
\scalemath{0.7}{
\binom{F_{k_1,t},H_{k_2,k_3,t}}{F_{k_1,t}, B_{k_2,t}, B_{k_3,t}} 
= \left[ \langle f_{k_1,i_1,t} \otimes b_{k_2,i_2,t} \otimes b_{k_3,i_3,t}, f_{k_1,i_4,t} \otimes h_{k_2,k_3,i_4,t} \rangle : i_1,i_2,i_3,i_4 \in [1:n] \right] \in \R^{n \times n \times n \times n},
}
\]
\[
\scalemath{0.7}{
\binom{F_{k_1,t},J_{k_2,k_3,t}}{F_{k_1,t}, B_{k_2,t}, F_{k_3,t}} 
= \left[ \langle f_{k_1,i_1,t} \otimes b_{k_2,i_2,t} \otimes f_{k_3,i_3,t}, f_{k_1,i_4,t} \otimes j_{k_2,k_3,i_4,t} \rangle : i_1,i_2,i_3,i_4 \in [1:n] \right] \in \R^{n \times n \times n \times n} \text{ and }
}
\]
\[
\scalemath{0.7}{
\binom{B_{k_1,t},J_{k_2,k_3,t}}{B_{k_1,t}, B_{k_2,t}, F_{k_3,t}} 
= \left[ \langle b_{k_1,i_1,t} \otimes b_{k_2,i_2,t} \otimes f_{k_3,i_3,t}, b_{k_1,i_4,t} \otimes j_{k_2,k_3,i_4,t} \rangle : i_1,i_2,i_3,i_4 \in [1:n] \right] \in \R^{n \times n \times n \times n}.
}
\]

In the following, $\Vert X \Vert$ refers to the Euclidean norm of the tensor $X$. The cosine of a pair of tensors $X_1,X_2$ is defined as $\cos(X_1,X_2) = \Vert X_1 \Vert^{-1} \Vert X_2 \Vert^{-1} \langle X_1,X_2 \rangle \in [-1,1]$ where $\langle X_1,X_2 \rangle$ is the usual dot product. The cosine of a triplet of Gram matrices $X_1,X_2,X_3 \in \R^{n \times n}$ is going make an appearance, which is defined as $\cos(X_1,X_2,X_3) = \Vert X_1 \Vert_3^{-1} \Vert X_2 \Vert_3^{-1} \Vert X_3 \Vert_3^{-1} \tr(X_1 X_2 X_3) \in [0,1]$ with $\Vert X \Vert_3 = (\sum_{i_1=1}^n \sum_{i_2=1}^n \vert X_{i_1,i_2} \vert^3)^{\frac{1}{3}}$. We now state our main result, expressing the three Taylor coefficients above in terms of a few low dimensional tensors of dimension $l+1$. Its proof along with other results can be found in Appendix~\ref{proofs}.

\begin{theorem}[Taylor Coefficients]\label{thm:taylor_coeffs}
For all $t \in \N$, we have 
\[
\scalemath{0.7}{
\langle \Xi_t \nabla_t, \nabla_t \rangle = \langle D_{\xi_t} \tau_t, T^{(1)}_t \rangle,
\quad
\langle (\Xi_t \nabla_t)^\otimes 2, \partial\nabla_t \rangle = \langle (D_{\xi_t} \tau_t)^{\otimes 2}, T^{(2)}_t \rangle
\quad \text{ and } \quad
\langle (\Xi_t \nabla_t)^\otimes 3, \partial\partial\nabla_t \rangle = \langle (D_{\xi_t} \tau_t)^{\otimes 3}, T^{(3)}_t \rangle
}
\]
with $\xi_t, \tau_t, T^{(1)}_t \in \R^{l+1}$, $T^{(2)}_t \in \R^{(l+1) \times (l+1)}$ and $T^{(3)}_t \in \R^{(l+1) \times (l+1) \times (l+1)}$ defined as 
\[
\scalemath{0.7}{
\xi_t = [\xi_{k,t} : k \in [1:l+1]],
\quad
\tau_t = \left[ \frac{ \tr(\frac{1}{n} F_{k-1,t})^{\frac{1}{2}} \tr(\frac{1}{n} B_{k,t})^{\frac{1}{2}} }{ R(F_{k-1,t})^{\frac{1}{4}} R(B_{k,t})^{\frac{1}{4}} } : k \in [1:l+1] \right],
}
\]
\[
\scalemath{0.7}{
T^{(1)}_t = \left[ \frac{ \tr(\frac{1}{n} F_{k-1,t})^{\frac{1}{2}} \tr(\frac{1}{n} B_{k,t})^{\frac{1}{2}} }{ R(F_{k-1,t})^{\frac{1}{4}} R(B_{k,t})^{\frac{1}{4}} } \cos(F_{k-1,t}, B_{k,t}) : k \in [1:l+1] \right],
}
\]
\begin{multline*}
\scalemath{0.7}{
T^{(2)}_{t,k,k}
= \frac{ \tr(\frac{1}{n} F_{k-1,t}^{\circ 2})^{\frac{1}{2}} \tr(\frac{1}{n} H_{k,k,t})^{\frac{1}{2}} }{ R(F_{k-1,t}^{\circ 2})^{\frac{1}{4}} R(H_{k,k,t})^{\frac{1}{4}} } 
\cos\left( \binom{H_{k,k,t}}{B_{k,t},B_{k,t}}, \binom{F_{k-1,t},F_{k-1,t}}{F_{k-1,t},F_{k-1,t}} \right)
} \\ \scalemath{0.7}{
\cos\left( \frac{1}{n} \sum_{i=1}^n h_{k,k,i,t}^{\otimes 2}, \frac{1}{n^2} \sum_{i_1=1}^n \sum_{i_2=1}^n (b_{k,i_1,t} \otimes b_{k,i_2,t})^{\otimes 2} \right)^{\frac{1}{2}} 
\cos\left( \frac{1}{n} \sum_{i=1}^n f_{k-1,i,t}^{\otimes 4}, \frac{1}{n^2} \sum_{i_1=1}^n \sum_{i_2=1}^n (f_{k-1,i_1,t} \otimes f_{k-1,i_2,t})^{\otimes 2} \right)^{\frac{1}{2}} \text{ for } k \in [1:l+1],
}
\end{multline*}
\begin{multline*}
\scalemath{0.7}{
T^{(2)}_{t,k_1,k_2}
= \frac{ \tr(\frac{1}{n} F_{k_1-1,t}^{\circ 2})^{\frac{1}{4}} \tr(\frac{1}{n} F_{k_2-1,t}^{\circ 2})^{\frac{1}{4}} \tr(\frac{1}{n} H_{k_1,k_2,t})^{\frac{1}{2}} }{ R(F_{k_1-1,t}^{\circ 2})^{\frac{1}{8}} R(F_{k_2-1,t}^{\circ 2})^{\frac{1}{8}} R(H_{k_1,k_2,t})^{\frac{1}{4}} } 
\cos\left( \binom{H_{k_1,k_2,t}}{B_{k_1,t}, B_{k_2,t}}, \binom{F_{k_1-1,t},F_{k_2-1,t}}{F_{k_1-1,t},F_{k_2-1,t}} \right) 
\cos(F_{k_1-1,t}^{\circ 2}, F_{k_2-1,t}^{\circ 2})^{\frac{1}{4}}
} \\ \scalemath{0.7}{
\cos\left( \frac{1}{n} \sum_{i=1}^n h_{k_1,k_2,i,t}^{\otimes 2}, \frac{1}{n^2} \sum_{i_1=1}^n \sum_{i_2=1}^n (b_{k_1,i_1,t} \otimes b_{k_2,i_2,t})^{\otimes 2} \right)^{\frac{1}{2}} 
\cos\left( \frac{1}{n} \sum_{i=1}^n (f_{k_1-1,i,t} \otimes f_{k_2-1,i,t})^{\otimes 2}, \frac{1}{n^2} \sum_{i_1=1}^n \sum_{i_2=1}^n (f_{k_1-1,i_1,t} \otimes f_{k_2-1,i_2,t})^{\otimes 2} \right)^{\frac{1}{2}} 
}
\end{multline*}
\begin{multline*}
\scalemath{0.7}{
+ \frac{ \tr(\frac{1}{n} F_{k_1-1,t}^{\circ 2})^{\frac{1}{4}} \tr(\frac{1}{n} B_{k_2,t}^{\circ 2})^{\frac{1}{4}} \tr(\frac{1}{n} J_{k_1,k_2-1,t})^{\frac{1}{2}} }{ R(F_{k_1-1,t}^{\circ 2})^{\frac{1}{8}} R(B_{k_2,t}^{\circ 2})^{\frac{1}{8}} R(J_{k_1,k_2-1,t})^{\frac{1}{4}} } 
\cos\left( \binom{J_{k_1,k_2-1,t}}{B_{k_1,t},F_{k_2-1,t}}, \binom{F_{k_1-1,t},B_{k_2,t}}{F_{k_1-1,t},B_{k_2,t}} \right) 
\cos(F_{k_1-1,t}^{\circ 2}, B_{k_2,t}^{\circ 2})^{\frac{1}{4}}
} \\ \scalemath{0.7}{
\cos\left( \frac{1}{n} \sum_{i=1}^n j_{k_1,k_2-1,i,t}^{\otimes 2}, \frac{1}{n^2} \sum_{i_1=1}^n \sum_{i_2=1}^n (b_{k_1,i_1,t} \otimes f_{k_2-1,i_2,t})^{\otimes 2} \right)^{\frac{1}{2}} 
\cos\left( \frac{1}{n} \sum_{i=1}^n (f_{k_1-1,i,t} \otimes b_{k_2,i,t})^{\otimes 2}, \frac{1}{n^2} \sum_{i_1=1}^n \sum_{i_2=1}^n (f_{k_1-1,i_1,t} \otimes b_{k_2,i_2,t})^{\otimes 2} \right)^{\frac{1}{2}} 
} \\ \scalemath{0.7}{
\text{ for } k_1 < k_2 \in [1:l+1] \text{ and } T^{(2)}_{t,k_1,k_2} = T^{(2)}_{t,k_2,k_1} \text{ for } k_1 > k_2 \in [1:l+1],
}
\end{multline*}
\begin{multline*}
\scalemath{0.7}{
T^{(3)}_{t,k,k,k}
= \frac{ \tr(\frac{1}{n} F_{k-1,t}^{\circ 3})^{\frac{1}{2}} \tr(\frac{1}{n} Z_{k,k,k,t})^{\frac{1}{2}} }{ R(F_{k-1,t}^{\circ 3})^{\frac{1}{4}} R(Z_{k,k,k,t})^{\frac{1}{4}} } 
\cos\left( \binom{Z_{k,k,k,t}}{B_{k,t},B_{k,t},B_{k,t}}, \binom{F_{k-1,t},F_{k-1,t},F_{k-1,t}}{F_{k-1,t},F_{k-1,t},F_{k-1,t}} \right) 
} \\ \scalemath{0.7}{
\cos\left( \frac{1}{n} \sum_{i=1}^n z_{k,k,k,i,t}^{\otimes 2}, \frac{1}{n^3} \sum_{i_1=1}^n \sum_{i_2=1}^n \sum_{i_3=1}^n (b_{k,i_1,t} \otimes b_{k,i_2,t} \otimes b_{k,i_3,t})^{\otimes 2} \right)^{\frac{1}{2}} 
} \\ \scalemath{0.7}{
\cos\left( \frac{1}{n} \sum_{i=1}^n f_{k-1,i,t}^{\otimes 6}, \frac{1}{n^3} \sum_{i_1=1}^n \sum_{i_2=1}^n \sum_{i_3=1}^n (f_{k-1,i_1,t} \otimes f_{k-1,i_2,t} \otimes f_{k-1,i_3,t})^{\otimes 2} \right)^{\frac{1}{2}} \text{ for } k \in [1:l+1],
}
\end{multline*}
\begin{multline*}
\scalemath{0.7}{
T^{(3)}_{t,k_1,k_1,k_2}
= T^{(3)}_{t,k_1,k_2,k_1}
= T^{(3)}_{t,k_2,k_1,k_1}
= \frac{ \tr(\frac{1}{n} F_{k_1-1,t}^{\circ 3})^{\frac{1}{3}} \tr(\frac{1}{n} F_{k_2-1,t}^{\circ 3})^{\frac{1}{6}} \tr(\frac{1}{n} Z_{k_1,k_1,k_2,t})^{\frac{1}{2}} }{ R(F_{k_1-1,t}^{\circ 3})^{\frac{1}{6}} R(F_{k_2-1,t}^{\circ 3})^{\frac{1}{12}} R(Z_{k_1,k_1,k_2,t})^{\frac{1}{4}} } 
} \\ \scalemath{0.7}{
\cos\left( \binom{Z_{k_1,k_1,k_2,t}}{B_{k_1,t},B_{k_1,t},B_{k_2,t}}, \binom{F_{k_1-1,t},F_{k_1-1,t},F_{k_2-1,t}}{F_{k_1-1,t},F_{k_1-1,t},F_{k_2-1,t}} \right) 
\cos(F_{k_1-1,t}^{\circ 2}, F_{k_1-1,t}^{\circ 2}, F_{k_2-1,t}^{\circ 2})^{\frac{1}{4}} 
} \\ \scalemath{0.7}{
\cos\left( \frac{1}{n} \sum_{i=1}^n z_{k_1,k_1,k_2,i,t}^{\otimes 2}, \frac{1}{n^3} \sum_{i_1=1}^n \sum_{i_2=1}^n \sum_{i_3=1}^n (b_{k_1,i_1,t} \otimes b_{k_1,i_2,t} \otimes b_{k_2,i_3,t})^{\otimes 2} \right)^{\frac{1}{2}} 
} \\ \scalemath{0.7}{
\cos\left( \frac{1}{n} \sum_{i=1}^n (f_{k_1-1,i,t}^{\otimes 2} \otimes f_{k_2-1,i,t})^{\otimes 2}, \frac{1}{n^3} \sum_{i_1=1}^n \sum_{i_2=1}^n \sum_{i_3=1}^n (f_{k_1-1,i_1,t} \otimes f_{k_1-1,i_2,t} \otimes f_{k_2-1,i_3,t})^{\otimes 2} \right)^{\frac{1}{2}}
}
\end{multline*}
\begin{multline*}
\scalemath{0.7}{
+ 2 \frac{ \tr(\frac{1}{n} F_{k_1-1,t}^{\circ 2})^{\frac{1}{2}} \tr(\frac{1}{n} H_{k_1,k_2,t}^{\circ 2})^{\frac{1}{4}} \tr(\frac{1}{n} J_{k_1,k_2-1,t}^{\circ 2})^{\frac{1}{4}} }{ R(F_{k_1-1,t}^{\circ 2})^{\frac{1}{4}} R(H_{k_1,k_2,t}^{\circ 2})^{\frac{1}{8}} R(J_{k_1,k_2-1,t}^{\circ 2})^{\frac{1}{8}} } 
} \\ \scalemath{0.7}{
\cos\left( \binom{F_{k_1-1,t},H_{k_1,k_2,t}}{F_{k_1-1,t},B_{k_1,t},B_{k_2,t}}, 
\binom{F_{k_1-1,t},J_{k_1,k_2-1,t}}{B_{k_1,t},F_{k_1-1,t},F_{k_2-1,t}} \right) 
\cos(F_{k_1-1,t}^{\circ 2}, H_{k_1,k_2,t}^{\circ 2})^{\frac{1}{4}} 
\cos(F_{k_1-1,t}^{\circ 2}, J_{k_1,k_2-1,t}^{\circ 2})^{\frac{1}{4}}
} \\ \scalemath{0.7}{
\cos\left( \frac{1}{n} \sum_{i=1}^n (f_{k_1-1,i,t} \otimes h_{k_1,k_2,i,t})^{\otimes 2}, \frac{1}{n^3} \sum_{i_1=1}^n \sum_{i_2=1}^n \sum_{i_3=1}^n (f_{k_1-1,i_1,t} \otimes b_{k_1,i_2,t} \otimes b_{k_2,i_3,t})^{\otimes 2} \right)^{\frac{1}{2}} 
} \\ \scalemath{0.7}{
\cos\left( \frac{1}{n} \sum_{i=1}^n (f_{k_1-1,i,t} \otimes j_{k_1,k_2-1,i,t})^{\otimes 2}, \frac{1}{n^3} \sum_{i_1=1}^n \sum_{i_2=1}^n \sum_{i_3=1}^n (f_{k_1-1,i_1,t} \otimes b_{k_1,i_2,t} \otimes f_{k_2-1,i_3,t})^{\otimes 2} \right)^{\frac{1}{2}}  \text{ for } k_1 < k_2 \in [1:l+1],
}
\end{multline*}
\begin{multline*}
\scalemath{0.7}{
T^{(3)}_{t,k_1,k_1,k_2}
= T^{(3)}_{t,k_1,k_2,k_1}
= T^{(3)}_{t,k_2,k_1,k_1}
= \frac{ \tr(\frac{1}{n} F_{k_1-1,t}^{\circ 3})^{\frac{1}{3}} \tr(\frac{1}{n} F_{k_2-1,t}^{\circ 3})^{\frac{1}{6}} \tr(\frac{1}{n} Z_{k_1,k_1,k_2,t})^{\frac{1}{2}} }{ R(F_{k_1-1,t}^{\circ 3})^{\frac{1}{6}} R(F_{k_2-1,t}^{\circ 3})^{\frac{1}{12}} R(Z_{k_1,k_1,k_2,t})^{\frac{1}{4}} } 
} \\ \scalemath{0.7}{
\cos\left( \binom{Z_{k_1,k_1,k_2,t}}{B_{k_1,t},B_{k_1,t},B_{k_2,t}}, 
\binom{F_{k_1-1,t},F_{k_1-1,t},F_{k_2-1,t}}{F_{k_1-1,t},F_{k_1-1,t},F_{k_2-1,t}} \right) 
\cos(F_{k_1-1,t}^{\circ 2}, F_{k_1-1,t}^{\circ 2}, F_{k_2-1,t}^{\circ 2})^{\frac{1}{4}}
} \\ \scalemath{0.7}{
\cos\left( \frac{1}{n} \sum_{i=1}^n z_{k_1,k_1,k_2,i,t}^{\otimes 2}, \frac{1}{n^3} \sum_{i_1=1}^n \sum_{i_2=1}^n \sum_{i_3=1}^n (b_{k_1,i_1,t} \otimes b_{k_1,i_2,t} \otimes b_{k_2,i_3,t})^{\otimes 2} \right)^{\frac{1}{2}} 
} \\ \scalemath{0.7}{
\cos\left( \frac{1}{n} \sum_{i=1}^n (f_{k_1-1,i,t}^{\otimes 2} \otimes f_{k_2-1,i,t})^{\otimes 2}, \frac{1}{n^3} \sum_{i_1=1}^n \sum_{i_2=1}^n \sum_{i_3=1}^n (f_{k_1-1,i_1,t} \otimes f_{k_1-1,i_2,t} \otimes f_{k_2-1,i_3,t})^{\otimes 2} \right)^{\frac{1}{2}} 
} 
\end{multline*}
\begin{multline*}
\scalemath{0.7}{
+ 2 \frac{ \tr(\frac{1}{n} F_{k_1-1,t}^{\circ 2})^{\frac{1}{4}} \tr(\frac{1}{n} F_{k_2-1,t}^{\circ 2})^{\frac{1}{4}} \tr(\frac{1}{n} H_{k_1,k_1,t}^{\circ 2})^{\frac{1}{4}} \tr(\frac{1}{n} J_{k_2,k_1-1,t}^{\circ 2})^{\frac{1}{4}} }{ R(F_{k_1-1,t}^{\circ 2})^{\frac{1}{8}} R(F_{k_2-1,t}^{\circ 2})^{\frac{1}{8}} R(H_{k_1,k_1,t}^{\circ 2})^{\frac{1}{8}} R(J_{k_2,k_1-1,t}^{\circ 2})^{\frac{1}{8}} } 
} \\ \scalemath{0.7}{
\cos\left( \binom{F_{k_2-1,t},H_{k_1,k_1,t}}{F_{k_2-1,t},B_{k_1,t},B_{k_1,t}}, 
\binom{F_{k_1-1,t},J_{k_2,k_1-1,t}}{B_{k_2,t},F_{k_1-1,t},F_{k_1-1,t}} \right) 
\cos(F_{k_2-1,t}^{\circ 2}, H_{k_1,k_1,t}^{\circ 2})^{\frac{1}{4}} 
\cos(F_{k_1-1,t}^{\circ 2}, J_{k_2,k_1-1,t}^{\circ 2})^{\frac{1}{4}}
} \\ \scalemath{0.7}{
\cos\left( \frac{1}{n} \sum_{i=1}^n (f_{k_2-1,i,t} \otimes h_{k_1,k_1,i,t})^{\otimes 2}, \frac{1}{n^3} \sum_{i_1=1}^n \sum_{i_2=1}^n \sum_{i_3=1}^n (f_{k_2-1,i_1,t} \otimes b_{k_1,i_2,t} \otimes b_{k_1,i_3,t})^{\otimes 2} \right)^{\frac{1}{2}} 
} \\ \scalemath{0.7}{
\cos\left( \frac{1}{n} \sum_{i=1}^n (f_{k_1-1,i,t} \otimes j_{k_2,k_1-1,i,t})^{\otimes 2}, \frac{1}{n^3} \sum_{i_1=1}^n \sum_{i_2=1}^n \sum_{i_3=1}^n (f_{k_1-1,i_1,t} \otimes b_{k_2,i_2,t} \otimes f_{k_1-1,i_3,t})^{\otimes 2} \right)^{\frac{1}{2}} \text{ for } k_1 > k_2 \in [1:l+1] \text{ and }
}
\end{multline*}
\begin{multline*}
\scalemath{0.7}{
T^{(3)}_{t,k_1,k_2,k_3}
= T^{(3)}_{t,k_1,k_3,k_2}
= T^{(3)}_{t,k_3,k_2,k_1}
= T^{(3)}_{t,k_2,k_1,k_3}
= \frac{ \tr(\frac{1}{n} F_{k_1-1,t}^{\circ 3})^{\frac{1}{6}} \tr(\frac{1}{n} F_{k_2-1,t}^{\circ 3})^{\frac{1}{6}} \tr(\frac{1}{n} F_{k_3-1,t}^{\circ 3})^{\frac{1}{6}} \tr(\frac{1}{n} Z_{k_1,k_2,k_3,t})^{\frac{1}{2}} }{ R(F_{k_1-1,t}^{\circ 3})^{\frac{1}{12}} R(F_{k_2-1,t}^{\circ 3})^{\frac{1}{12}} R(F_{k_3-1,t}^{\circ 3})^{\frac{1}{12}} R(Z_{k_1,k_2,k_3,t})^{\frac{1}{4}} } 
} \\ \scalemath{0.7}{
\cos\left( \binom{Z_{k_1,k_2,k_3,t}}{B_{k_1,t},B_{k_2,t},B_{k_3,t}}, 
\binom{F_{k_1-1,t},F_{k_2-1,t},F_{k_3-1,t}}{F_{k_1-1,t},F_{k_2-1,t},F_{k_3-1,t}} \right) 
\cos(F_{k_1-1,t}^{\circ 2}, F_{k_2-1,t}^{\circ 2}, F_{k_3-1,t}^{\circ 2})^{\frac{1}{4}}
} \\ \scalemath{0.7}{
\cos\left( \frac{1}{n} \sum_{i=1}^n z_{k_1,k_2,k_3,i,t}^{\otimes 2}, \frac{1}{n^3} \sum_{i_1=1}^n \sum_{i_2=1}^n \sum_{i_3=1}^n (b_{k_1,i_1,t} \otimes b_{k_2,i_2,t} \otimes b_{k_3,i_3,t})^{\otimes 2} \right)^{\frac{1}{2}} 
} \\ \scalemath{0.7}{
\cos\left( \frac{1}{n} \sum_{i=1}^n (f_{k_1-1,i,t} \otimes f_{k_2-1,i,t} \otimes f_{k_3-1,i,t})^{\otimes 2}, \frac{1}{n^3} \sum_{i_1=1}^n \sum_{i_2=1}^n \sum_{i_3=1}^n (f_{k_1-1,i_1,t} \otimes f_{k_2-1,i_2,t} \otimes f_{k_3-1,i_3,t})^{\otimes 2} \right)^{\frac{1}{2}} 
}
\end{multline*}
\begin{multline*}
\scalemath{0.7}{
+ \frac{ \tr(\frac{1}{n} F_{k_1-1,t}^{\circ 2})^{\frac{1}{4}} \tr(\frac{1}{n} F_{k_2-1,t}^{\circ 2})^{\frac{1}{4}} \tr(\frac{1}{n} H_{k_2,k_3,t}^{\circ 2})^{\frac{1}{4}} \tr(\frac{1}{n} J_{k_1,k_3-1,t}^{\circ 2})^{\frac{1}{4}} }{ R(F_{k_1-1,t}^{\circ 2})^{\frac{1}{8}} R(F_{k_2-1,t}^{\circ 2})^{\frac{1}{8}} R(H_{k_2,k_3,t}^{\circ 2})^{\frac{1}{8}} R(J_{k_1,k_3-1,t}^{\circ 2})^{\frac{1}{8}} } 
} \\ \scalemath{0.7}{
\cos\left( \binom{F_{k_1-1,t},H_{k_2,k_3,t}}{F_{k_1-1,t},B_{k_2,t},B_{k_3,t}}, 
\binom{F_{k_2-1,t},J_{k_1,k_3-1,t}}{B_{k_1,t},F_{k_2-1,t},F_{k_3-1,t}} \right) 
\cos(F_{k_1-1,t}^{\circ 2}, H_{k_2,k_3,t}^{\circ 2})^{\frac{1}{4}} 
\cos(F_{k_2-1,t}^{\circ 2}, J_{k_1,k_3-1,t}^{\circ 2})^{\frac{1}{4}}
} \\ \scalemath{0.7}{
\cos\left( \frac{1}{n} \sum_{i=1}^n (f_{k_1-1,i,t} \otimes h_{k_2,k_3,i,t})^{\otimes 2}, \frac{1}{n^3} \sum_{i_1=1}^n \sum_{i_2=1}^n \sum_{i_3=1}^n (f_{k_1-1,i_1,t} \otimes b_{k_2,i_2,t} \otimes b_{k_3,i_3,t})^{\otimes 2} \right)^{\frac{1}{2}} 
} \\ \scalemath{0.7}{
\cos\left( \frac{1}{n} \sum_{i=1}^n (f_{k_2-1,i,t} \otimes j_{k_1,k_3-1,i,t})^{\otimes 2}, \frac{1}{n^3} \sum_{i_1=1}^n \sum_{i_2=1}^n \sum_{i_3=1}^n (f_{k_2-1,i_1,t} \otimes b_{k_1,i_2,t} \otimes f_{k_3-1,i_3,t})^{\otimes 2} \right)^{\frac{1}{2}} 
}
\end{multline*}
\begin{multline*}
\scalemath{0.7}{
+ \frac{ \tr(\frac{1}{n} F_{k_1-1,t}^{\circ 2})^{\frac{1}{4}} \tr(\frac{1}{n} F_{k_3-1,t}^{\circ 2})^{\frac{1}{4}} \tr(\frac{1}{n} H_{k_2,k_3,t}^{\circ 2})^{\frac{1}{4}} \tr(\frac{1}{n} J_{k_1,k_2-1,t}^{\circ 2})^{\frac{1}{4}} }{ R(F_{k_1-1,t}^{\circ 2})^{\frac{1}{8}} R(F_{k_3-1,t}^{\circ 2})^{\frac{1}{8}} R(H_{k_2,k_3,t}^{\circ 2})^{\frac{1}{8}} R(J_{k_1,k_2-1,t}^{\circ 2})^{\frac{1}{8}} }
} \\ \scalemath{0.7}{
\cos\left( \binom{F_{k_1-1,t},H_{k_2,k_3,t}}{F_{k_1-1,t},B_{k_2,t},B_{k_3,t}}, 
\binom{F_{k_3-1,t},J_{k_1,k_2-1,t}}{B_{k_1,t},F_{k_2-1,t},F_{k_3-1,t}} \right) 
\cos(F_{k_1-1,t}^{\circ 2}, H_{k_2,k_3,t}^{\circ 2})^{\frac{1}{4}} 
\cos(F_{k_3-1,t}^{\circ 2}, J_{k_1,k_2-1,t}^{\circ 2})^{\frac{1}{4}} 
} \\ \scalemath{0.7}{
\cos\left( \frac{1}{n} \sum_{i=1}^n (f_{k_1-1,i,t} \otimes h_{k_2,k_3,i,t})^{\otimes 2}, \frac{1}{n^3} \sum_{i_1=1}^n \sum_{i_2=1}^n \sum_{i_3=1}^n (f_{k_1-1,i_1,t} \otimes b_{k_2,i_2,t} \otimes b_{k_3,i_3,t})^{\otimes 2} \right)^{\frac{1}{2}} 
} \\ \scalemath{0.7}{
\cos\left( \frac{1}{n} \sum_{i=1}^n (f_{k_3-1,i,t} \otimes j_{k_1,k_2-1,i,t})^{\otimes 2}, \frac{1}{n^3} \sum_{i_1=1}^n \sum_{i_2=1}^n \sum_{i_3=1}^n (f_{k_3-1,i_1,t} \otimes b_{k_1,i_2,t} \otimes f_{k_2-1,i_3,t})^{\otimes 2} \right)^{\frac{1}{2}}
}
\end{multline*}
\begin{multline*}
\scalemath{0.7}{
+ \frac{ \tr(\frac{1}{n} F_{k_1-1,t}^{\circ 2})^{\frac{1}{4}} \tr(\frac{1}{n} F_{k_2-1,t}^{\circ 2})^{\frac{1}{4}} \tr(\frac{1}{n} H_{k_1,k_3,t}^{\circ 2})^{\frac{1}{4}} \tr(\frac{1}{n} J_{k_2,k_3-1,t}^{\circ 2})^{\frac{1}{4}} }{ R(F_{k_1-1,t}^{\circ 2})^{\frac{1}{8}} R(F_{k_2-1,t}^{\circ 2})^{\frac{1}{8}} R(H_{k_1,k_3,t}^{\circ 2})^{\frac{1}{8}} R(J_{k_2,k_3-1,t}^{\circ 2})^{\frac{1}{8}} } 
} \\ \scalemath{0.7}{
\cos\left( \binom{F_{k_2-1,t},H_{k_1,k_3,t}}{B_{k_1,t},F_{k_2-1,t},B_{k_3,t}}, 
\binom{F_{k_1-1,t},J_{k_2,k_3-1,t}}{F_{k_1-1,t},B_{k_2,t},F_{k_3-1,t}} \right) 
\cos(F_{k_2-1,t}^{\circ 2}, H_{k_1,k_3,t}^{\circ 2})^{\frac{1}{4}} 
\cos(F_{k_1-1,t}^{\circ 2}, J_{k_2,k_3-1,t}^{\circ 2})^{\frac{1}{4}}
} \\ \scalemath{0.7}{
\cos\left( \frac{1}{n} \sum_{i=1}^n (f_{k_2-1,i,t} \otimes h_{k_1,k_3,i,t})^{\otimes 2}, \frac{1}{n^3} \sum_{i_1=1}^n \sum_{i_2=1}^n \sum_{i_3=1}^n (f_{k_2-1,i_1,t} \otimes b_{k_1,i_2,t} \otimes b_{k_3,i_3,t})^{\otimes 2} \right)^{\frac{1}{2}} 
} \\ \scalemath{0.7}{
\cos\left( \frac{1}{n} \sum_{i=1}^n (f_{k_1-1,i,t} \otimes j_{k_2,k_3-1,i,t})^{\otimes 2}, \frac{1}{n^3} \sum_{i_1=1}^n \sum_{i_2=1}^n \sum_{i_3=1}^n (f_{k_1-1,i_1,t} \otimes b_{k_2,i_2,t} \otimes f_{k_3-1,i_3,t})^{\otimes 2} \right)^{\frac{1}{2}} 
}
\end{multline*}
\begin{multline*}
\scalemath{0.7}{
+ \frac{ \tr(\frac{1}{n} F_{k_1-1,t}^{\circ 2})^{\frac{1}{4}} \tr(\frac{1}{n} B_{k_3,t}^{\circ 2})^{\frac{1}{4}} \tr(\frac{1}{n} J_{k_1,k_2-1,t}^{\circ 2})^{\frac{1}{4}} \tr(\frac{1}{n} J_{k_2,k_3-1,t}^{\circ 2})^{\frac{1}{4}} }{ R(F_{k_1-1,t}^{\circ 2})^{\frac{1}{8}} R(B_{k_3,t}^{\circ 2})^{\frac{1}{8}} R(J_{k_1,k_2-1,t}^{\circ 2})^{\frac{1}{8}} R(J_{k_2,k_3-1,t}^{\circ 2})^{\frac{1}{8}} } 
} \\ \scalemath{0.7}{
\cos\left( \binom{B_{k_3,t},J_{k_1,k_2-1,t}}{B_{k_1,t},F_{k_2-1,t},B_{k_3,t}}, 
\binom{F_{k_1-1,t},J_{k_2,k_3-1,t}}{F_{k_1-1,t},B_{k_2,t},F_{k_3-1,t}} \right) 
\cos(B_{k_3,t}^{\circ 2}, J_{k_1,k_2-1,t}^{\circ 2})^{\frac{1}{4}} 
\cos(F_{k_1-1,t}^{\circ 2}, J_{k_2,k_3-1,t}^{\circ 2})^{\frac{1}{4}}
} \\ \scalemath{0.7}{
\cos\left( \frac{1}{n} \sum_{i=1}^n (b_{k_3,i,t} \otimes j_{k_1,k_2-1,i,t})^{\otimes 2}, \frac{1}{n^3} \sum_{i_1=1}^n \sum_{i_2=1}^n \sum_{i_3=1}^n (b_{k_3,i_1,t} \otimes b_{k_1,i_2,t} \otimes f_{k_3-1,i_3,t})^{\otimes 2} \right)^{\frac{1}{2}} 
} \\ \scalemath{0.7}{
\cos\left( \frac{1}{n} \sum_{i=1}^n (f_{k_1-1,i,t} \otimes j_{k_2,k_3-1,i,t})^{\otimes 2}, \frac{1}{n^3} \sum_{i_1=1}^n \sum_{i_2=1}^n \sum_{i_3=1}^n (f_{k_1-1,i_1,t} \otimes b_{k_2,i_2,t} \otimes f_{k_3-1,i_3,t})^{\otimes 2} \right)^{\frac{1}{2}} \text{ for } k_1 < k_2 < k_3 \in [1:l+1].
}
\end{multline*}
\end{theorem}

In the result above, the gradient scales are collected in $\xi_t$, while $\tau_t$ contains the average squared norms of forward and backward vectors and the soft ranks of forward and backward Gram matrices. The first order Taylor tensor $T^{(1)}_t$ consists of the elements of $\tau_t$ multiplied by the cosine alignments between forward and backward Gram matrices. Second order quantities are absorbed by the second order Taylor tensor $T^{(2)}_t$. The simplest of these are the average fourth powers of norms of forward and backward vectors, the average Frobenius norms of cross layer Hessians and Jacobians and the soft ranks of second Hadamard powers of forward and backward Gram matrices, cross layer Hessian gram matrices and cross layer Jacobian Gram matrices. Cosine alignments of second Hadamard powers of forward and backward Gram matrices also appear, as well as cosines of the third order tensors consisting of inner products which were defined above the result. Unlike in the first order Taylor tensor, there are also cosine alignments of certain empirical covariance tensors of random matrices. In this case, randomness is considered with respect to minibatch sampling, so that larger minibatches lead to these empirical covariance tensors being better approximations of the actual covariance tensor. The trend continues in the third order Taylor tensor $T^{(3)}_t$, which scales with the average fourth or sixth powers of norms of forward and backward vectors, the average squared Frobenius norms of cross layer Hessians and Jacobians, the average Euclidean norms of cross layer Tressians, as well as the soft ranks of second or third Hadamard powers of forward and backward Gram matrices, the soft ranks of second Hadamard powers of Hessian and Jacobian Gram matrices and the soft ranks of Tressian Gram matrices. There are also cosines of fourth order tensors consisting of inner products and cosines of empirical covariance tensors of random third order tensors.

Without claiming to fully make sense of the above formulas, one takeaway is that the Taylor coefficients are determined by variables that quantify either the \emph{scale}, the \emph{structure} or the pairwise \emph{alignment} of minibatches of certain tensors. Scale variables are powers of Euclidean norms of a given type of tensor averaged over the minibatch. Structural variables are the soft ranks, which measure the quality of different representations of the same minibatch. Unlike scale, structure also depends on the relations between tensors of the same type in a minibatch. We argue that representation learning is not only about the activations representing the inputs, since hidden layer gradients can be seen as learned representations of output space gradients and they also heavily affect learning according to the above result. This holds analogously to hidden layer representations of output space Hessians and Tressians. In the third category on the other hand, cosine alignments quantify the relationship between tensors of different types, often coupled with each other. Among the cosines, distinct subcategories are alignments of tensors that consist of inner products (including Gram matrices) and alignments of covariance tensors. As inner products factor into norms and cosines, a tensor of inner products describes the geometry of a combination of different representations of the same minibatch. The cosine of a pair of such tensors measures the alignment between the geometries of two sets of representations (in different spaces) of the same minibatch. On the other hand, a covariance tensor describes the internal probabilistic structure of a random tensor. In the above cases, the empirical covariance tensors corresponding to a minibatch approximate the true covariances of random tensors which are pushforwards of some tensor power of the probability measure that is the dataset. Cosines of such tensors measure the alignment between pushforwards of the minibatch and its second or third tensor power (which in this case are in the same space).

Based on the above theorem, we propose the gradient scaling scheme
\begin{equation}\label{eq:grad_scale}
\scalemath{0.9}{
\xi_{k,t} = \tr\left( \frac{1}{n} F_{k-1,t} \right)^{-\frac{1}{2}} \tr\left( \frac{1}{n} B_{k,t} \right)^{-\frac{1}{2}} \text{ for } k \in [1:l+1] \text{ and } t \in \N,
}
\end{equation}
canceling the forward and backward norms in $\tau_t$. The formula is cheap to compute as the forward and backward vectors are always available (the latter can be accessed e.g. by backward hooks), the additional cost being that the minibatch average of their norms has to be computed. Another benefit is that \eqref{eq:grad_scale} also cancels the forward and backward norms in the cumulative updates by the following result, which shows that $\Vert \Delta_{k,t} \Vert$ scales with the average squared norms of past forward and backward vectors, the soft ranks and cosine alignments of the corresponding forward and backward Gram matrices and the cosine alignments of the empirical forward backward cross covariances across time.

\begin{proposition}[Squared Frobenius Norm of Cumulative Updates]
For all $t \in \N$, we have
\begin{multline*}
\scalemath{0.8}{
\Vert \Delta_{k,t} \Vert^2
= \sum_{t_1=0}^{t-1} \sum_{t_2=0}^{t-1} \xi_{k,t_1} \xi_{k,t_2} \frac{ \tr(\frac{1}{n} F_{k-1,t_1})^{\frac{1}{2}} \tr(\frac{1}{n} F_{k-1,t_2})^{\frac{1}{2}} \tr(\frac{1}{n} B_{k,t_1})^{\frac{1}{2}} \tr(\frac{1}{n} B_{k,t_2})^{\frac{1}{2}} }{ R(F_{k-1,t_1})^{\frac{1}{4}} R(F_{k-1,t_2})^{\frac{1}{4}} R(B_{k,t_1})^{\frac{1}{4}} R(B_{k,t_2})^{\frac{1}{4}} } 
} \\ \scalemath{0.8}{
\cos(F_{k-1,t_1}, B_{k,t_1})^{\frac{1}{2}} \cos(F_{k-1,t_2}, B_{k,t_2})^{\frac{1}{2}}
\cos\left( \frac{1}{n} \sum_{i=1}^n b_{k,i,t_1} \otimes f_{k-1,i,t_1}, \frac{1}{n} \sum_{i=1}^n b_{k,i,t_2} \otimes f_{k-1,i,t_2} \right)
\text{ for } k \in [1:l+1].
}
\end{multline*}
\end{proposition}

To demonstrate the usefulness of our results, we train $\nu$P ReLU MLPs with $l=8$ hidden layers on MNIST with and without the proposed gradient scaling \eqref{eq:grad_scale} for width exponents $r \in [0:2]$ along a range of learning rates and plot the results in terms of minibatch loss $\Ell_t(\theta_t)$ (on a logarithmic scale), effective sharpness $S_t$, stability ratio $\frac{1}{2} \eta S_t$ and last hidden layer soft rank $R(F_{l,t})$ in Figure~\ref{fig:plot_eos}. Each run took $\approx 8$ minutes on an NVIDIA A100 SXM4 40GB GPU. As expected, larger learning rates mostly lead to larger soft ranks, but to issues as well. Without \eqref{eq:grad_scale}, for $r=0$ a color disappearing before $t=2000$ indicates numerical errors. Width exponents $r=1$ and $r=2$ help avoiding this and lead to somewhat lower effective sharpness, but there are still occasional catapult events disrupting the flow of training even though the process recovers afterwards. On the other hand, with \eqref{eq:grad_scale} we can train with much larger learning rates without such difficulties, the only issue being that the loss does not descend if $\eta$ is too large. Setting $r=2$ with gradient scaling, we observe the best performance with the loss successfully descending in all cases and increasing learning rates leading consistently to smaller effective sharpness and higher soft rank of the last hidden layer representation, while the stability ratio shows that training takes place far beyond the effective EOS. We set the width scale $m$ such that all settings have parameter counts on the same order of magnitude, but the quadratic case has the smallest. Therefore it can be seen as an economic way of choosing hidden layer widths.

\begin{figure}
\begin{center}
\begin{tabular}{c}
\begin{subfigure}[b]{1\textwidth}
\includegraphics{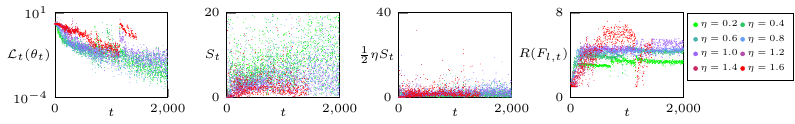}
\vspace*{-2mm}
\caption{Unscaled SGD with $r=0$, $m=8192$, $\dim(\Theta)=476266496$}
\end{subfigure}\\
\vspace*{-2mm}
\begin{subfigure}[b]{1\textwidth}
\includegraphics{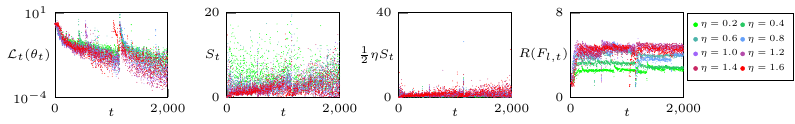}
\vspace*{-2mm}
\caption{Unscaled SGD with $r=1$, $m=2048$, $\dim(\Theta)=717508608$}
\end{subfigure}\\
\vspace*{-2mm}
\begin{subfigure}[b]{1\textwidth}
\includegraphics{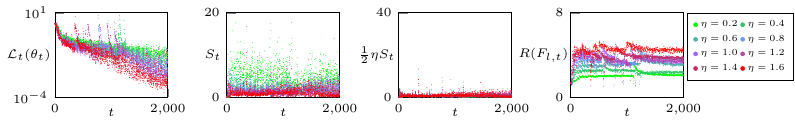}
\vspace*{-2mm}
\caption{Unscaled SGD with $r=2$, $m=256$, $\dim(\Theta)=431229440$}
\end{subfigure}\\
\vspace*{-2mm}
\begin{subfigure}[b]{1\textwidth}
\includegraphics{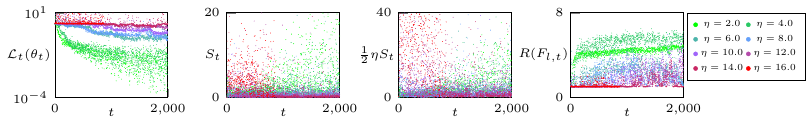}
\vspace*{-2mm}
\caption{Scaled SGD with \eqref{eq:grad_scale}, $r=0$, $m=8192$, $\dim(\Theta)=476266496$}
\end{subfigure}\\
\vspace*{-2mm}
\begin{subfigure}[b]{1\textwidth}
\includegraphics{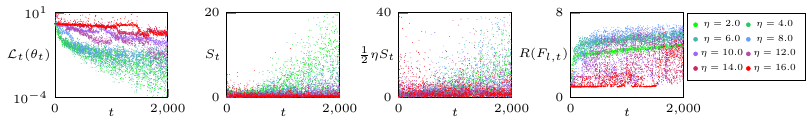}
\vspace*{-2mm}
\caption{Scaled SGD with \eqref{eq:grad_scale}, $r=1$, $m=2048$, $\dim(\Theta)=717508608$}
\end{subfigure}\\
\vspace*{-2mm}
\begin{subfigure}[b]{1\textwidth}
\includegraphics{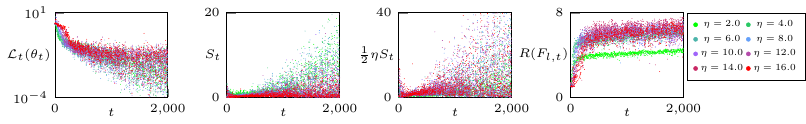}
\vspace*{-2mm}
\caption{Scaled SGD with \eqref{eq:grad_scale}, $r=2$, $m=256$, $\dim(\Theta)=431229440$}
\end{subfigure}
\end{tabular}
\vskip 0.05in
\caption{Training $\nu$P MLPs with $l=8$ and $(a,b)=(\frac{1}{2},\frac{1}{2})$ to minimize the classification loss $\ell(x,y) = \log(\langle e^z, \mathbbm{1}_n \rangle) - z_y$ on MNIST with minibatch size $n=2^8$.}
\label{fig:plot_eos}
\end{center}
\vskip -0.2in
\end{figure}

\section{Limitations and future directions}\label{conclusion}
The main limitation of our work is that it only treats MLPs. Extending $\nu$P and the gradient scaling scheme \eqref{eq:grad_scale} to other architectures could potentially enable practitioners to train better models with a more economic allocation of compute resources. While we attempted to interpret the formulas we obtained for the Taylor tensors, probably more insight can be gained from their study. In particular, an important future direction is to empirically observe the evolution of the scale, structure and alignment quantities that seem to govern the training dynamics. On the theoretical side, we intend to study the behavior of the soft ranks, which quantify representation learning and are much easier to analyze than most other spectral quantities. We believe that studying the theory of artificial neural networks by factoring objects of interest into norms, soft ranks, cosines and perhaps some other components promises to be a fruitful approach towards understanding the success of deep learning.

\begin{ack}
D\'avid Terj\'ek was supported by the Ministry of Innovation and Technology NRDI Office within the framework of the Artificial Intelligence National Laboratory (RRF-2.3.1-21-2022-00004).
\end{ack}

\newpage

\bibliographystyle{apalike}
\bibliography{neurips_2025}

\begin{thebibliography}{}

\bibitem[Andreyev and Beneventano, 2025]{Andreyevetal2025}
Andreyev, A. and Beneventano, P. (2025).
\newblock Edge of stochastic stability: Revisiting the edge of stability for
  sgd.

\bibitem[Arora et~al., 2022]{Aroraetal2022}
Arora, S., Li, Z., and Panigrahi, A. (2022).
\newblock Understanding gradient descent on edge of stability in deep learning.
\newblock {\em ArXiv}, abs/2205.09745.

\bibitem[Cohen et~al., 2025]{Cohenetal2025}
Cohen, J., Damian, A., Talwalkar, A., Kolter, J.~Z., and Lee, J.~D. (2025).
\newblock Understanding optimization in deep learning with central flows.
\newblock In {\em The Thirteenth International Conference on Learning
  Representations}.

\bibitem[Cohen et~al., 2021]{Cohenetal2021}
Cohen, J., Kaur, S., Li, Y., Kolter, J.~Z., and Talwalkar, A. (2021).
\newblock Gradient descent on neural networks typically occurs at the edge of
  stability.
\newblock In {\em International Conference on Learning Representations}.

\bibitem[Damian et~al., 2023]{Damianetal2023}
Damian, A., Nichani, E., and Lee, J.~D. (2023).
\newblock Self-stabilization: The implicit bias of gradient descent at the edge
  of stability.
\newblock In {\em The Eleventh International Conference on Learning
  Representations}.

\bibitem[Hayou et~al., 2019]{Hayouetal2019}
Hayou, S., Doucet, A., and Rousseau, J. (2019).
\newblock On the impact of the activation function on deep neural networks
  training.
\newblock In Chaudhuri, K. and Salakhutdinov, R., editors, {\em Proceedings of
  the 36th International Conference on Machine Learning}, volume~97 of {\em
  Proceedings of Machine Learning Research}, pages 2672--2680. PMLR.

\bibitem[Lee and Jang, 2023]{Leeetal2023}
Lee, S. and Jang, C. (2023).
\newblock A new characterization of the edge of stability based on a sharpness
  measure aware of batch gradient distribution.
\newblock In {\em International Conference on Learning Representations}.

\bibitem[Lewkowycz et~al., 2020]{Lewkowyczetal2020}
Lewkowycz, A., Bahri, Y., Dyer, E., Sohl-Dickstein, J., and Gur-Ari, G. (2020).
\newblock The large learning rate phase of deep learning: the catapult
  mechanism.

\bibitem[Li et~al., 2019]{Lietal2019b}
Li, Y., Wei, C., and Ma, T. (2019).
\newblock Towards explaining the regularization effect of initial large
  learning rate in training neural networks.
\newblock In {\em Neural Information Processing Systems}.

\bibitem[Mohtashami et~al., 2022]{Mohtashamietal2022}
Mohtashami, A., Jaggi, M., and Stich, S.~U. (2022).
\newblock Special properties of gradient descent with large learning rates.
\newblock In {\em International Conference on Machine Learning}.

\bibitem[Noci et~al., 2024]{Nocietal2024}
Noci, L., Meterez, A., Hofmann, T., and Orvieto, A. (2024).
\newblock Super consistency of neural network landscapes and learning rate
  transfer.
\newblock In {\em The Thirty-eighth Annual Conference on Neural Information
  Processing Systems}.

\bibitem[Wu and Su, 2023]{Wuetal2023}
Wu, L. and Su, W.~J. (2023).
\newblock The implicit regularization of dynamical stability in stochastic
  gradient descent.
\newblock In Krause, A., Brunskill, E., Cho, K., Engelhardt, B., Sabato, S.,
  and Scarlett, J., editors, {\em Proceedings of the 40th International
  Conference on Machine Learning}, volume 202 of {\em Proceedings of Machine
  Learning Research}, pages 37656--37684. PMLR.

\bibitem[Yang and Hu, 2021]{Yangetal2021}
Yang, G. and Hu, E.~J. (2021).
\newblock Tensor programs iv: Feature learning in infinite-width neural
  networks.
\newblock In Meila, M. and Zhang, T., editors, {\em Proceedings of the 38th
  International Conference on Machine Learning}, volume 139 of {\em Proceedings
  of Machine Learning Research}, pages 11727--11737. PMLR.

\bibitem[Zhu et~al., 2023]{Zhuetal2023}
Zhu, L., Liu, C., Radhakrishnan, A., and Belkin, M. (2023).
\newblock Catapults in sgd: spikes in the training loss and their impact on
  generalization through feature learning.
\newblock {\em ArXiv}, abs/2306.04815.

\end{thebibliography}


\newpage

\appendix

\section{Proofs and additional results}\label{proofs}

Given $i, j \in \N$, we define the tuple $[i:j] = (i,i+1,\cdots,j-1,j)$ (which is the empty tuple $()$ if $i > j$). For any $m, n \in \N$, we denote by $m\N+n$ the set $\{mr+n : r \in \N\}$. The space of bounded linear operators from some vector space $G$ to another vector space $H$ is denoted $L(G,H)$. We denote the identity on a vector space $H$ by $\Id_H$. The adjoint of a linear operator $A \in L(G,H)$ is the unique linear operator $A^* \in L(H,G)$ such that $\langle A x, y \rangle = \langle x, A^* y \rangle$ for all $x \in G$ and $y \in H$. Given $x \in \R^n$, we define the corresponding diagonal matrix $D_x \in \R^{n \times n}$ as ${D_x}_{i_1,i_2} = x_i$ if $i_1 = i_2 = i$ and $0$ otherwise for all $i_1,i_2 \in [1:n]$. For $n \in \N+1$, we denote the $n$-dimensional constant $1$ vector by $\mathbbm{1}_n = [ 1 : i \in [1:n]] \in \R^n$. Given a function $F : G \to H$ and $x \in G$, we say that $F$ is differentiable at $x$ if it is Fr\'echet differentiable at $x$, i.e., if there exists a bounded linear operator $\partial F(x) \in L(G,H)$, which we refer to as the Jacobian of $F$ at $x$, satisfying $\lim_{\substack{y \to x \\ x \neq y}}\frac{\Vert F(y) - F(x) - \partial F(x) (y-x) \Vert}{\Vert y - x \Vert}=0$. Such a function is differentiable if it is differentiable at all $x \in G$.

We introduce some notation that helps working with higher order tensors. We permute the tensor dimensions by listing the dimension indices in the desired order between parenthesis in the upper index. Grouping of the indices means that the appearance of the tensor should be viewed as if the grouped dimensions are flattened. For example, for a tensor $X \in \R^{a \times b \times c \times d}$, writing $X^{(a,c),(b,d)}$ implies that $X^{(a,c),(b,d)}$ acts as a linear operator from $\R^{b \times d}$ to $\R^{a \times c}$, i.e., a kind of matrix that has matrices as its rows and columns. The norm $\Vert X \Vert$ always denotes the Euclidean norm of a tensor $X$ and $\Vert X \Vert_3 = (\sum_{i_1=1}^{a_1} \cdots \sum_{i_d=1}^{a_d} \vert X_{i_1,\cdots,i_d} \vert^3)^{\frac{1}{3}}$ for $X \in \R^{a_1 \times \cdots \times a_d}$. For $X \in \R^{a_1 \times \cdots \times a_d \times n}$, the soft rank $R(X)$ denotes the soft rank of the corresponding Gram matrix $(X^{(1,\cdots,d),d+1})^* X^{(1,\cdots,d),d+1} \in \R^{n \times n}$, i.e., $R(X) = R((X^{(1,\cdots,d),d+1})^* X^{(1,\cdots,d),d+1}) = \frac{\Vert X \Vert^4}{\Vert (X^{(1,\cdots,d),d+1})^* X^{(1,\cdots,d),d+1} \Vert^2}$.

Define for all $t \in \N$ the forward and backward matrices 
\[
f_{k,t} = \left[ n^{-\frac{1}{2}} f_{k,i,t} : i \in [1:n] \right] \in \R^{m_k \times n} \text{ for } k \in [0:l]
\]
and
\[
b_{k,t} = \left[ n^{-\frac{1}{2}} b_{k,i,t} : i \in [1:n] \right] \in \R^{m_k \times n} \text{ for } k \in [1:l+1]
\]
and denote the layerwise gradients $\nabla_{k,t} = \nabla_{\theta_{k,t}} \Ell_t(\theta_t) \in \Theta_k$ for all $k \in [1:l+1]$ and $t \in \N$. The result below tells us that $\nabla_{k,t}$ is the empirical cross covariance (uncentered, also called cross correlation) matrix of the $k-1$th forward vectors and the $k$th backward vectors considered as random vectors with respect to the dataset $\mu$.

\begin{proposition}[Factorization of Gradients]\label{prop:nabla_k_t_factorization}
For all $k \in [1:l+1]$ and $t \in \N$, we have
\[
\nabla_{k,t}
= b_{k,t} f_{k-1,t}^* = \frac{1}{n} \sum_{i=1}^n b_{k,i,t} \otimes f_{k-1,i,t}
\]
\end{proposition}
\begin{proof}
As $\nabla_{\theta_{k,t}} \ell_{i,t} = (\nabla_{\theta_{k,t} f_{k-1,i,t}} \ell_{i,t}) \otimes f_{k-1,i,t} = (\nabla_{\theta_{k,t} f_{k-1,i,t}} \ell_{i,t}) \otimes f_{k-1,i,t} = b_{k,i,t} \otimes f_{k-1,i,t}$ by the chain rule, we have $\nabla_{\theta_{k,t}} \Ell_t(\theta_t) = \frac{1}{n} \sum_{i=1}^n \nabla_{\theta_{k,t}} \ell_{i,t} = \frac{1}{n} \sum_{i=1}^n b_{k,i,t} \otimes f_{k-1,i,t} = b_{k,t} f_{k-1,t}^*$.
\end{proof}

We show that the Frobenius norm of a gradient scales with the average norms of the forward and backward vectors, the soft ranks of the forward and backward matrices and the cosine alignment of the corresponding Gram matrices.

\begin{proposition}[Squared Frobenius Norm of Gradients]\label{prop:nabla_k_t_norm}
For all $k \in [1:l+1]$ and $t \in \N$, we have
\[
\Vert \nabla_{k,t} \Vert^2
= \frac{ \Vert f_{k-1,t} \Vert^2 \Vert b_{k,t} \Vert^2 }{ R(f_{k-1,t})^{\frac{1}{2}} R(b_{k,t})^{\frac{1}{2}} } \cos(f_{k-1,t}^* f_{k-1,t}, b_{k,t}^* b_{k,t}).
\]
\end{proposition}
\begin{proof}
By Proposition~\ref{prop:nabla_k_t_factorization} and the cyclic property of the trace, we have
\begin{multline*}
\Vert \nabla_{k,t} \Vert_F^2
= \tr(f_{k-1,t} b_{k,t}^* b_{k,t} f_{k-1,t}^*) 
= \tr(f_{k-1,t}^* f_{k-1,t} b_{k,t}^* b_{k,t}) \\
= \Vert f_{k-1,t}^* f_{k-1,t} \Vert_F \Vert b_{k,t}^* b_{k,t} \Vert_F \cos(f_{k-1,t}^* f_{k-1,t}, b_{k,t}^* b_{k,t}) \\
= \Vert f_{k-1,t} \Vert_F^2 \Vert b_{k,t} \Vert_F^2 R(f_{k-1,t})^{-\frac{1}{2}} R(b_{k,t})^{-\frac{1}{2}} \cos(f_{k-1,t}^* f_{k-1,t}, b_{k,t}^* b_{k,t}).
\end{multline*}
\end{proof}

\begin{proposition}[Squared Frobenius Norm of Cumulative Updates]
For all $k \in [1:l+1]$ and $t \in \N$, we have
\begin{multline*}
\Vert \Delta_{k,t} \Vert^2
= \sum_{t_1=0}^{t-1} \sum_{t_2=0}^{t-1} \xi_{k,t_1} \xi_{k,t_2} \frac{ \Vert f_{k-1,t_1} \Vert \Vert f_{k-1,t_2} \Vert \Vert b_{k,t_1} \Vert \Vert b_{k,t_2} \Vert }{ R(f_{k-1,t_1})^{\frac{1}{4}} R(f_{k-1,t_2})^{\frac{1}{4}} R(b_{k,t_1})^{\frac{1}{4}} R(b_{k,t_2})^{\frac{1}{4}} } \\
\cos(b_{k,t_1} f_{k-1,t_1}^*, b_{k,t_2} f_{k-1,t_2}^*) \cos(f_{k-1,t_1}^* f_{k-1,t_1}, b_{k,t_1}^* b_{k,t_1})^{\frac{1}{2}} \cos(f_{k-1,t_2}^* f_{k-1,t_2}, b_{k,t_2}^* b_{k,t_2})^{\frac{1}{2}}
\end{multline*}
\end{proposition}
\begin{proof}
Noting that $\Vert \Delta_{k,t} \Vert^2 = \sum_{t_1=0}^{t-1} \sum_{t_2=0}^{t-1} \xi_{k,t_1} \xi_{k,t_2} \tr(\nabla_{k,t_1}^* \nabla_{k,t_2})$ and $\tr(\nabla_{k,t_1}^* \nabla_{k,t_2}) = \Vert \nabla_{k,t_1} \Vert \Vert \nabla_{k,t_2}) \Vert \cos(\nabla_{k,t_1}, \nabla_{k,t_2})$, we get the result from Proposition~\ref{prop:nabla_k_t_factorization} and Proposition~\ref{prop:nabla_k_t_norm}.
\end{proof}

For $k_1,k_2 \in [1:l+1]$, $i \in [1:n]$ and $t \in \N$, note that 
\[
h_{k_1,k_2,i,t}  = (\partial_{\theta_{k_1,t} f_{k_1-1,i,t}} N_{i,t})^* (\partial \nabla \ell_{i,t}) (\partial_{\theta_{k_2,t} f_{k_2-1,i,t}} N_{i,t}),
\]
so that we can write the Hessians recursively as
\[
h_{l+1,l+1,i,t} = \partial \nabla \ell_{i,t},
\]
\[
h_{k_1,k_2,i,t} 
= m_{k_1}^{-\frac{1}{2}} D_{\phi'\left( m^{\frac{1}{2}} \theta_{k_1} f_{k_1-1,i,t} \right)} m^{\frac{1}{2}} \theta_{k_1+1}^* h_{k_1+1,k_2,i,t} \text{ for } k_1 \in [1:l], k_2 \in [1:l+1]
\]
and
\[
h_{k_1,k_2,i,t} = \left( m_{k_2}^{-\frac{1}{2}} D_{\phi'\left( m^{\frac{1}{2}} \theta_{k_2} f_{k_2-1,i,t} \right)} m^{\frac{1}{2}} \theta_{k_2+1}^* h_{k_1,k_2+1,i,t}^* \right)^* \text{ for } k_1 \in [1:l+1], k_2 \in [1:l].
\]
Analogously for the Tressians as 
\[
h_{k_1,k_2,i,t}  = ((\partial_{\theta_{k_1,t} f_{k_1-1,i,t}} N_{i,t}) \otimes (\partial_{\theta_{k_2,t} f_{k_2-1,i,t}} N_{i,t}) \otimes (\partial_{\theta_{k_3,t} f_{k_3-1,i,t}} N_{i,t}))^{(2,4,6),(1,3,5)} (\partial\partial\nabla\ell_{i,t}).
\]
We can also write the cross layer adjoint Jacobians recursively as
\[
j_{k,k,i,t} = m^{\frac{1}{2}} m_k^{-\frac{1}{2}} D_{\phi'(m^{\frac{1}{2}} \theta_{k,t} f_{k-1,i,t})} \text{ for } k \in [1:l]
\]
and
\[
j_{k_1,k_2,i,t}
= \left( m_{k_2}^{-\frac{1}{2}} D_{\phi'(m^{\frac{1}{2}} \theta_{k_2,t} f_{k_2-1,i,t})} m^{\frac{1}{2}} \theta_{k_2,t} j_{k_1,k_2-1,i,t}^* \right)^* \text{ for } k_1 < k_2 \in [1:l].
\]
Define for all $t \in \N$ the tensors
\[
h_{k_1,k_2,t}
= \left[ n^{-\frac{1}{2}} h_{k_1,k_2,i,t} : i \in [1:n] \right] \in \R^{(m_{k_1} \times m_{k_2}) \times n} \text{ for } k_1,k_2 \in [1:l+1],
\]
\[
j_{k_1,k_2,t}
= \left[ n^{-\frac{1}{2}} j_{k_1,k_2,i,t} : i \in [1:n] \right] \in \R^{(m_{k_1} \times m_{k_2}) \times n} \text{ for } k_1 \in [1:l], k_2 \in [k_1:l],
\]
\[
f_{k_1,t} \asd f_{k_2,t}
= \left[ n^{-\frac{1}{2}} f_{k_1,i,t} \otimes f_{k_2,i,t} : i \in [1:n] \right] \in \R^{(m_{k_1} \times m_{k_2}) \times n} \text{ for } k_1,k_2 \in [0:l]
\]
and
\[
f_{k_1,t} \asd b_{k_2,t}
= \left[ n^{-\frac{1}{2}} f_{k_1,i,t} \otimes b_{k_2,i,t} : i \in [1:n] \right] \in \R^{(m_{k_1} \times m_{k_2}) \times n} \text{ for } k_1 \in [0:l], k_2 \in [1:l+1].
\]
Denote the layerwise Hessians $\partial\nabla_{k_1,k_2,t} = \partial_{\theta_{k_2,t}} \nabla_{\theta_{k_1,t}} \Ell_t(\theta_t) \in \R^{(m_{k_1} \times m_{k_1-1}) \times (m_{k_2} \times m_{k_2-1})}$ for all $k_1,k_2 \in [1:l+1]$ and $t \in \N$. We now show that such a Hessian consists of two empirical cross covariance tensors, one of the cross layer Hessians and the tensor products of the forward vectors of the given layers and another of the cross layer Jacobians and the tensor products of the forward and backward vectors of the given layers (which does not appear on the diagonal), with all four considered as random matrices with respect to the dataset.

\begin{proposition}[Factorization of Hessians]\label{prop:partial_nabla_k_1_k_2_t_factorization}
For all $t \in \N$, we have 
\begin{multline*}
\partial\nabla_{k,k,t}
= (h_{k,k,t} (f_{k-1,t}^{\asd 2})^*)^{((1,3),(2,4))} 
= \frac{1}{n} \sum_{i=1}^n (h_{k,k,i,t} \otimes f_{k-1,i,t}^{\otimes 2})^{((1,3),(2,4))} \text{ for } k \in [1:l+1],
\end{multline*}
\begin{multline*}
\partial\nabla_{k_1,k_2,t}
= (h_{k_1,k_2,t} (f_{k_1-1,t} \asd f_{k_2-1,t})^*)^{((1,3),(2,4))} 
+ (j_{k_1,k_2-1,t} (f_{k_1-1,t} \asd b_{k_2,t})^*)^{((1,3),(4,2))} \\
= \frac{1}{n} \sum_{i=1}^n (h_{k_1,k_2,i,t} \otimes f_{k_1-1,i,t} \otimes f_{k_2-1,i,t})^{((1,3),(2,4))} 
+ (j_{k_1,k_2-1,i,t} \otimes f_{k_1-1,i,t} \otimes b_{k_2,i,t})^{((1,3),(4,2))}
\end{multline*}
for $k_1 < k_2 \in [1:l+1]$ and $\partial\nabla_{k_1,k_2,t} = \partial\nabla_{k_2,k_1,t}^{((3,4),(1,2))}$ for $k_1 > k_2 \in [1:l+1]$.
\end{proposition}
\begin{proof}
Note that we can express the gradient corresponding to an individual datapoint as
\[
\nabla_{\theta_{k_1,t}} \ell_{i,t}
= ((\partial_{\theta_{k_1,t} f_{k_1-1,i,t}} N_{i,t})^* (\nabla \ell_{i,t})) \otimes f_{k_1-1,i,t}.
\]
In order to apply the product rule, we need to differentiate all three individual components above with respect to $\theta_{k_2,t}$. First, noting that
\begin{multline*}
\partial_{\theta_{k_1,t} f_{k_1-1,i,t}} N_{i,t}
= m^{\frac{1}{2}} \theta_{l+1,t} m_l^{-\frac{1}{2}} D_{\phi'(m^{\frac{1}{2}} \theta_{l,t} f_{l-1,i,t})} \cdots m^{\frac{1}{2}} \theta_{k_1+1,t} m_{k_1}^{-\frac{1}{2}} D_{\phi'(m^{\frac{1}{2}} \theta_{k_1,t} f_{k_1-1,i,t})} \\
= ( m^{\frac{1}{2}} \theta_{l+1,t} m_l^{-\frac{1}{2}} D_{\phi'(m^{\frac{1}{2}} \theta_{l,t} f_{l-1,i,t})} \cdots m^{\frac{1}{2}} \theta_{k_2+1,t} m_{k_2}^{-\frac{1}{2}} D_{\phi'(m^{\frac{1}{2}} \theta_{k_2,t} f_{k_2-1,i,t})} ) \theta_{k_2,t} \\
( m^{\frac{1}{2}} m_{k_2-1}^{-\frac{1}{2}} D_{\phi'(m^{\frac{1}{2}} \theta_{k_2-1,t} f_{k_2-2,i,t})} \cdots m^{\frac{1}{2}} \theta_{k_1+1,t} m_{k_1}^{-\frac{1}{2}} D_{\phi'(m^{\frac{1}{2}} \theta_{k_1,t} f_{k_1-1,i,t})} ) \\
= (\partial_{\theta_{k_2,t} f_{k_2-1,i,t}} N_{i,t}) \theta_{k_2,t} (\partial_{\theta_{k_1,t} f_{k_1-1,i,t}} f_{k_2-1,i,t})
\end{multline*}
if $k_1 < k_2$, as $\phi''(s)=0$ for almost every $s \in \R$ we have
\[
\partial_{\theta_{k_2,t}} (\partial_{\theta_{k_1,t} f_{k_1-1,i,t}} N_{i,t})^*
= ((\partial_{\theta_{k_2,t} f_{k_2-1,i,t}} N_{i,t}) \otimes (\partial_{\theta_{k_1,t} f_{k_1-1,i,t}} f_{k_2-1,i,t}))^{(4,1),(2,3)}.
\]
We also have
\begin{multline*}
\partial_{\theta_{k_2,t}} \nabla \ell_{i,t}
= (\partial_{\theta_{k_2,t} f_{k_2-1,i,t}} \nabla \ell_{i,t}) \otimes f_{k_2-1,i,t} \\
= ((\partial\nabla \ell_{i,t}) (\partial_{\theta_{k_2,t} f_{k_2-1,i,t}} N_{i,t})) \otimes f_{k_2-1,i,t}
= h_{l+1,k_2,i,t} \otimes f_{k_2-1,i,t}
\end{multline*}
and if $k_1 > k_2$ then
\[
\partial_{\theta_{k_2,t}} f_{k_1-1,i,t}
= (\partial_{\theta_{k_2,t} f_{k_2-1,i,t}} f_{k_1-1,i,t}) \otimes f_{k_2-1,i,t}.
\]
By the product rule, we then have
\begin{multline*}
\partial_{\theta_{k_2,t}} \nabla_{\theta_{k_1,t}} \ell_{i,t}
= (((\partial_{\theta_{k_1,t} f_{k_1-1,i,t}} N_{i,t})^* (\partial \nabla \ell_{i,t}) (\partial_{\theta_{k_2,t} f_{k_2-1,i,t}} N_{i,t})) \otimes f_{k_1-1,i,t} \otimes f_{k_2-1,i,t})^{(1,3),(2,4)} \\
+ \begin{cases}
((\partial_{\theta_{k_1,t} f_{k_1-1,i,t}} f_{k_2-1,i,t}) \otimes f_{k_1-1,i,t} \otimes ((\partial_{\theta_{k_2,t} f_{k_2-1,i,t}} N_{i,t})^* (\nabla \ell_{i,t})))^{(1,3),(4,2)} \text{ if } k_1 < k_2, \\
((\partial_{\theta_{k_2,t} f_{k_2-1,i,t}} f_{k_1-1,i,t}) \otimes f_{k_2-1,i,t} \otimes ((\partial_{\theta_{k_1,t} f_{k_1-1,i,t}} N_{i,t})^* (\nabla \ell_{i,t})))^{(4,1),(2,3)} \text{ if } k_1 > k_2 \text{ and } \\
0 \text{ if } k_1=k_2
\end{cases} \\
= (h_{k_1,k_2,i,t} \otimes f_{k_1-1,i,t} \otimes f_{k_2-1,i,t})^{(1,3),(2,4)}
+ \begin{cases}
(j_{k_1,k_2-1,i,t} \otimes f_{k_1-1,i,t} \otimes b_{k_2,i,t})^{(1,3),(4,2)} \text{ if } k_1 < k_2, \\
(j_{k_2,k_1-1,i,t} \otimes f_{k_2-1,i,t} \otimes b_{k_1,i,t})^{(4,1),(2,3)} \text{ if } k_1 > k_2 \text{ and } \\
0 \text{ if } k_1=k_2.
\end{cases}.
\end{multline*}
We get the claim since $\partial\nabla_{k_1,k_2,t} = \partial_{\theta_{k_2,t}} \nabla_{\theta_{k_1,t}} \Ell_t  = \frac{1}{n} \sum_{i=1}^n \partial_{\theta_{k_2,t}} \nabla_{\theta_{k_1,t}} \ell_{i,t}$.
\end{proof}

\begin{proposition}[Factorization of Tressians]\label{prop:partial_partial_nabla_k_1_k_2_k_3_t_factorization}
For all $t \in \N$, we have 
\[
\partial\partial\nabla_{k,k,k,t}
= \frac{1}{n} \sum_{i=1}^n (z_{k,k,k,i,t} \otimes f_{k-1,i,t}^{\otimes 3})^{((1,4),(2,5)),(3,6)} \text{ for } k \in [1:l+1],
\]
\begin{multline*}
\partial\partial\nabla_{k_1,k_1,k_2,t}
= \frac{1}{n} \sum_{i=1}^n (z_{k_1,k_1,k_2,i,t} \otimes f_{k_1-1,i,t}^{\otimes 2} \otimes f_{k_2-1,i,t})^{((1,4),(2,5)),(3,6)} \\
+ 2 (h_{k_1,k_2,i,t} \otimes j_{k_1,k_2-1,i,t} \otimes f_{k_1-1,i,t}^{\otimes 2})^{((1,5),(3,6)),(2,4)}, \\
\partial\partial\nabla_{k_1,k_2,k_1,t} = \partial\partial\nabla_{k_1,k_1,k_2,t}^{((1,2),(5,6)),(3,4)}
\text{ and } \\
\partial\partial\nabla_{k_2,k_1,k_1,t} = \partial\partial\nabla_{k_1,k_1,k_2,t}^{((5,6),(1,2)),(3,4)}
\text{ for } k_1 < k_2 \in [1:l+1],
\end{multline*}

\begin{multline*}
\partial\partial\nabla_{k_1,k_1,k_2,t}
= \frac{1}{n} \sum_{i=1}^n (z_{k_1,k_1,k_2,i,t} \otimes f_{k_1-1,i,t}^{\otimes 2} \otimes f_{k_2-1,i,t})^{((1,4),(2,5)),(3,6)} \\
+ 2 (h_{k_1,k_1,i,t} \otimes j_{k_2,k_1-1,i,t} \otimes f_{k_1-1,i,t} \otimes f_{k_2-1,i,t})^{((1,5),(3,6)),(2,4)}, \\
\partial\partial\nabla_{k_1,k_2,k_1,t} = \partial\partial\nabla_{k_1,k_1,k_2,t}^{((1,2),(5,6)),(3,4)}
\text{ and } \\
\partial\partial\nabla_{k_2,k_1,k_1,t} = \partial\partial\nabla_{k_1,k_1,k_2,t}^{((5,6),(1,2)),(3,4)}
\text{ for } k_1 > k_2 \in [1:l+1] \text{ and }
\end{multline*}
\begin{multline*}
\partial\nabla_{k_1,k_2,k_3,t}
= \frac{1}{n} \sum_{i=1}^n (z_{k_1,k_2,k_3,i,t} \otimes f_{k_1-1,i,t} \otimes f_{k_2-1,i,t} \otimes f_{k_3-1,i,t})^{((1,4),(2,5)),(3,6)} \\
+ (h_{k_1,k_3,i,t} \otimes j_{k_2,k_3-1,i,t} \otimes f_{k_1-1,i,t} \otimes f_{k_2-1,i,t})^{((1,5),(3,6)),(2,4)} \\
+ (h_{k_2,k_3,i,t} \otimes j_{k_1,k_3-1,i,t} \otimes f_{k_1-1,i,t} \otimes f_{k_2-1,i,t})^{((3,5),(1,6)),(2,4)} \\
+ (h_{k_2,k_3,i,t} \otimes j_{k_1,k_2-1,i,t} \otimes f_{k_1-1,i,t} \otimes f_{k_3-1,i,t})^{((3,5),(1,4)),(2,6)} \\
+ (j_{k_1,k_2-1,i,t} \otimes j_{k_2,k_3-1,i,t} \otimes f_{k_1-1,i,t} \otimes b_{k_3,i,t})^{((1,5),(3,2)),(6,4)}, \\
\partial\partial\nabla_{k_1,k_3,k_2,t} = \partial\partial\nabla_{k_1,k_2,k_3,t}^{((1,2),(5,6)),(3,4)}, \\
\partial\partial\nabla_{k_3,k_2,k_1,t} = \partial\partial\nabla_{k_1,k_2,k_3,t}^{((5,6),(3,4)),(1,2)} \text{ and } \\
\partial\partial\nabla_{k_2,k_1,k_3,t} = \partial\partial\nabla_{k_1,k_2,k_3,t}^{((3,4),(1,2)),(5,6)} \text{ for } k_1 < k_2 < k_3 \in [1:l+1].
\end{multline*}

\end{proposition}
\begin{proof}
We focus on the case $k_1 < k_2 < k_3 \in [1:l+1]$, and the rest will follow by reindexing and dropping terms. Note that by Proposition~\ref{prop:partial_nabla_k_1_k_2_t_factorization} we have
\begin{multline*}
\partial_{\theta_{k_2,t}} \nabla_{\theta_{k_1,t}} \ell_{i,t} \\
= (((\partial_{\theta_{k_1,t} f_{k_1-1,i,t}} N_{i,t})^* (\partial\nabla \ell_{i,t}) (\partial_{\theta_{k_2,t} f_{k_2-1,i,t}} N_{i,t})) \otimes f_{k_1-1,i,t} \otimes f_{k_2-1,i,t})^{((1,3),(2,4))} \\
+ (j_{k_1,k_2-1,i,t} \otimes f_{k_1-1,i,t} \otimes ((\partial_{\theta_{k_2,t} f_{k_2-1,i,t}} N_{i,t})^* (\nabla \ell_{i,t}))^{((1,3),(4,2))}.
\end{multline*}
As $k_3 > k_2 > k_1$, we only have to differentiate the terms $\partial_{\theta_{k_1,t} f_{k_1-1,i,t}} N_{i,t}$, $\partial_{\theta_{k_2,t} f_{k_2-1,i,t}} N_{i,t}$ (which appears in two terms), $\partial\nabla \ell_{i,t}$ and $\nabla \ell_{i,t}$. Then we apply the product rule to have $\partial_{\theta_{k_3,t}} \partial_{\theta_{k_2,t}} \nabla_{\theta_{k_1,t}} \ell_{i,t}$ as a sum of five terms. Summing these up over $i \in [1:n]$ gives us the claim. We differentiate the listed terms as already done in the proof of Proposition~\ref{prop:partial_nabla_k_1_k_2_t_factorization}, the only exception being $\partial\nabla \ell_{i,t}$. For this we have
\begin{multline*}
\partial_{\theta_{k_3,t}} \partial\nabla \ell_{i,t}
= (\partial_{\theta_{k_3,t} f_{k_3-1,i,t}} \partial\nabla \ell_{i,t}) \otimes f_{k_3-1,i,t} \\
= ((\partial_{\theta_{k_3,t} f_{k_3-1,i,t}} N_{i,t})^* (\partial\partial\nabla \ell_{i,t})) \otimes f_{k_3-1,i,t}
= z_{k_3,l+1,l+1,i,t} \otimes f_{k_3-1,i,t}.
\end{multline*}
The claim then follows from the product rule.
\end{proof}

\begin{proposition}[First Order Term]\label{prop:first_order_term}
For all $t \in \N$, we have
\[
\langle \Xi_t \nabla_t, \nabla_t \rangle
= \langle D_{\xi_t} \tau_t, T^{(1)}_t \rangle
\]
with $\xi_t, \tau_t, T^{(1)}_t \in \R^{l+1}$ defined as $\xi_t = [\xi_{k,t} : k \in [1:l+1]]$,
\[
\tau_t = \left[ \frac{ \Vert f_{k-1,t} \Vert \Vert b_{k,t} \Vert }{ R(f_{k-1,t})^{\frac{1}{4}} R(b_{k,t})^{\frac{1}{4}} } : k \in [1:l+1] \right]
\]
and
\[
T^{(1)}_t = \left[ \frac{ \Vert f_{k-1,t} \Vert \Vert b_{k,t} \Vert }{ R(f_{k-1,t})^{\frac{1}{4}} R(b_{k,t})^{\frac{1}{4}} } \cos(f_{k-1,t}^* f_{k-1,t}, b_{k,t}^* b_{k,t}) : k \in [1:l+1] \right].
\]
\end{proposition}
\begin{proof}
Noting that $\langle \Xi_t \nabla_t, \nabla_t \rangle = \sum_{k=1}^{l+1} \xi_{k,t} \Vert \nabla_{k,t} \Vert^2$, the claim follows from Proposition~\ref{prop:nabla_k_t_norm}.
\end{proof}

\begin{proposition}[Second Order Term]\label{prop:second_order_term}
For all $t \in \N$, we have
\[
\langle (\Xi_t \nabla_t)^\otimes 2, \partial\nabla_t \rangle
= \langle (D_{\xi_t} \tau_t)^{\otimes 2}, T^{(2)}_t \rangle
\]
with $T^{(2)}_t \in \R^{(l+1) \times (l+1)}$ defined as
\begin{multline*}
T^{(2)}_{t,k,k}
= \frac{ \Vert f_{k-1,t}^{\asd 2} \Vert \Vert h_{k,k,t} \Vert }{ R(f_{k-1,t}^{\asd 2})^{\frac{1}{4}} R(h_{k_1,k_2,t})^{\frac{1}{4}} } 
\cos\left( h_{k,k,t}^* (b_{k,t}^{\otimes 2})^{((1,3),(2,4))}, (f_{k-1,t}^{\asd 2})^* (f_{k-1,t}^{\otimes 2})^{((1,3),(2,4))} \right) \\
\cos\left( h_{k,k,t} h_{k,k,t}^*, ((b_{k,t} b_{k,t}^*)^{\otimes 2})^{((1,3),(2,4))} \right)^{\frac{1}{2}} \\
\cos\left( (f_{k-1,t}^{\asd 2}) (f_{k-1,t}^{\asd 2})^*, ((f_{k-1,t} f_{k-1,t}^*)^{\otimes 2})^{((1,3),(2,4))} \right)^{\frac{1}{2}} \text{ for } k \in [1:l+1],
\end{multline*}
\begin{multline*}
T^{(2)}_{t,k_1,k_2}
= \frac{ \Vert f_{k_1-1,t}^{\asd 2} \Vert^{\frac{1}{2}} \Vert f_{k_2-1,t}^{\asd 2} \Vert^{\frac{1}{2}} \Vert h_{k_1,k_2,t} \Vert }{ R(f_{k_1-1,t}^{\asd 2})^{\frac{1}{8}} R(f_{k_2-1,t}^{\asd 2})^{\frac{1}{8}} R(h_{k_1,k_2,t})^{\frac{1}{4}} } \\
\cos\left( h_{k_1,k_2,t}^* (b_{k_1,t} \otimes b_{k_2,t})^{((1,3),(2,4))}, (f_{k_1-1,t} \asd f_{k_2-1,t})^* (f_{k_1-1,t} \otimes f_{k_2-1,t})^{((1,3),(2,4))} \right) \\
\cos\left( h_{k_1,k_2,t} h_{k_1,k_2,t}^*, ((b_{k_1,t} b_{k_1,t}^*) \otimes (b_{k_2,t} b_{k_2,t}^*))^{((1,3),(2,4))} \right)^{\frac{1}{2}} \\
\cos\left( (f_{k_1-1,t} \asd f_{k_2-1,t}) (f_{k_1-1,t} \asd f_{k_2-1,t})^*, ((f_{k_1-1,t} f_{k_1-1,t}^*) \otimes (f_{k_2-1,t} f_{k_2-1,t}^*))^{((1,3),(2,4))} \right)^{\frac{1}{2}} \\
\cos((f_{k_1-1,t}^{\asd 2})^* (f_{k_1-1,t}^{\asd 2}), (f_{k_2-1,t}^{\asd 2})^* (f_{k_2-1,t}^{\asd 2}))^{\frac{1}{4}} \\
+ \frac{ \Vert f_{k_1-1,t}^{\asd 2} \Vert^{\frac{1}{2}} \Vert b_{k_2,t}^{\asd 2} \Vert^{\frac{1}{2}} \Vert j_{k_1,k_2-1,t} \Vert }{ R(f_{k_1-1,t}^{\asd 2})^{\frac{1}{8}} R(b_{k_2,t}^{\asd 2})^{\frac{1}{8}} R(j_{k_1,k_2-1,t})^{\frac{1}{4}} } \\
\cos\left( j_{k_1,k_2-1,t}^* (b_{k_1,t} \otimes f_{k_2-1,t})^{((1,3),(2,4))}, (f_{k_1-1,t} \asd b_{k_2,t})^* (f_{k_1-1,t} \otimes b_{k_2,t})^{((1,3),(2,4))} \right) \\
\cos\left( j_{k_1,k_2-1,t} j_{k_1,k_2-1,t}^*, ((b_{k_1,t} b_{k_1,t}^*) \otimes (f_{k_2-1,t} f_{k_2-1,t}^*))^{((1,3),(2,4))} \right)^{\frac{1}{2}} \\
\cos\left( (f_{k_1-1,t} \asd b_{k_2,t}) (f_{k_1-1,t} \asd b_{k_2,t})^*, ((f_{k_1-1,t} f_{k_1-1,t}^*) \otimes (b_{k_2,t} b_{k_2,t}^*))^{((1,3),(2,4))} \right)^{\frac{1}{2}} \\
\cos((f_{k_1-1,t}^{\asd 2})^* (f_{k_1-1,t}^{\asd 2}), (b_{k_2,t}^{\asd 2})^* (b_{k_2,t}^{\asd 2}))^{\frac{1}{4}} \text{ for } k_1 < k_2 \in [1:l+1]
\end{multline*}
and $T^{(2)}_{t,k_1,k_2} = T^{(2)}_{t,k_2,k_1}$ for $k_1 > k_2 \in [1:l+1]$.
\end{proposition}
\begin{proof}
We clearly have $\langle (\Xi_t \nabla_t)^\otimes 2, \partial\nabla_t \rangle = \sum_{k_1=1}^{l+1} \sum_{k_2=1}^{l+1} \xi_{k_1,t} \xi_{k_2,t} \langle \nabla_{k_1,t} \otimes \nabla_{k_2,t}, \partial\nabla_{k_1,k_2,t} \rangle$. Substituting from Proposition~\ref{prop:nabla_k_t_factorization} and Proposition~\ref{prop:partial_nabla_k_1_k_2_t_factorization}, we have
\begin{multline*}
\langle \nabla_{k_1,t} \otimes \nabla_{k_2,t}, \partial\nabla_{k_1,k_2,t} \rangle \\
= \frac{1}{n^3} \sum_{i_1=1}^n \sum_{i_2=1}^n \sum_{i_3=1}^n \langle b_{k_1,i_1,t} \otimes b_{k_2,i_2,t}, h_{k_1,k_2,i_3,t} \rangle \langle f_{k_1-1,i_1,t} \otimes f_{k_2-1,i_2,t}, f_{k_1-1,i_3,t} \otimes f_{k_2-1,i_3,t} \rangle \\
+ \begin{cases}
\frac{1}{n^3} \sum_{i_1=1}^n \sum_{i_2=1}^n \sum_{i_3=1}^n \langle b_{k_1,i_1,t} \otimes f_{k_2-1,i_2,t}, j_{k_1,k_2-1,i_3,t} \rangle \langle f_{k_1-1,i_1,t} \otimes b_{k_2,i_2,t}, f_{k_1-1,i_3,t} \otimes b_{k_2,i_3,t} \rangle \\
\text{ if } k_1 < k_2, \\
\frac{1}{n^3} \sum_{i_1=1}^n \sum_{i_2=1}^n \sum_{i_3=1}^n \langle b_{k_2,i_2,t} \otimes f_{k_1-1,i_1,t}, j_{k_2,k_1-1,i_3,t} \rangle \langle f_{k_2-1,i_2,t} \otimes b_{k_1,i_1,t}, f_{k_2-1,i_3,t} \otimes b_{k_1,i_3,t} \rangle \\
\text{ if } k_1 > k_2 \text{ and } \\
0 \text{ if } k_1=k_2
\end{cases} \\
= \langle h_{k_1,k_2,t}^* (b_{k_1,t} \otimes b_{k_2,t})^{((1,3),(2,4))}, (f_{k_1-1,t} \asd f_{k_2-1,t})^* (f_{k_1-1,t} \otimes f_{k_2-1,t})^{((1,3),(2,4))} \rangle \\
+ \begin{cases}
\langle j_{k_1,k_2-1,t}^* (b_{k_1,t} \otimes f_{k_2-1,t})^{((1,3),(2,4))}, (f_{k_1-1,t} \asd b_{k_2,t})^* (f_{k_1-1,t} \otimes b_{k_2,t})^{((1,3),(2,4))} \rangle \text{ if } k_1 < k_2, \\
\langle j_{k_2,k_1-1,t}^* (b_{k_2,t} \otimes f_{k_1-1,t})^{((1,3),(2,4))}, (f_{k_2-1,t} \asd b_{k_1,t})^* (f_{k_2-1,t} \otimes b_{k_1,t})^{((1,3),(2,4))} \rangle \text{ if } k_1 > k_2 \text{ and } \\
0 \text{ if } k_1=k_2.
\end{cases}
\end{multline*}
We rewrite the first term as
\begin{multline*}
\langle h_{k_1,k_2,t}^* (b_{k_1,t} \otimes b_{k_2,t})^{((1,3),(2,4))}, (f_{k_1-1,t} \asd f_{k_2-1,t})^* (f_{k_1-1,t} \otimes f_{k_2-1,t})^{((1,3),(2,4))} \rangle \\
= \Vert h_{k_1,k_2,t}^* (b_{k_1,t} \otimes b_{k_2,t})^{((1,3),(2,4))} \Vert \Vert (f_{k_1-1,t} \asd f_{k_2-1,t})^* (f_{k_1-1,t} \otimes f_{k_2-1,t})^{((1,3),(2,4))} \Vert \\
\cos(h_{k_1,k_2,t}^* (b_{k_1,t} \otimes b_{k_2,t})^{((1,3),(2,4))}, (f_{k_1-1,t} \asd f_{k_2-1,t})^* (f_{k_1-1,t} \otimes f_{k_2-1,t})^{((1,3),(2,4))})
\end{multline*}
and then the two norms as
\begin{multline*}
\Vert h_{k_1,k_2,t}^* (b_{k_1,t} \otimes b_{k_2,t})^{((1,3),(2,4))} \Vert^2 \\
= \langle h_{k_1,k_2,t} h_{k_1,k_2,t}^*, ((b_{k_1,t} b_{k_1,t}^*) \otimes (b_{k_2,t} b_{k_2,t}^*))^{((1,3),(2,4))} \rangle \\
= \Vert h_{k_1,k_2,t} h_{k_1,k_2,t}^* \Vert \Vert b_{k_1,t} b_{k_1,t}^* \Vert \Vert b_{k_2,t} b_{k_2,t}^* \Vert \\ \cos(h_{k_1,k_2,t} h_{k_1,k_2,t}^*, ((b_{k_1,t} b_{k_1,t}^*) \otimes (b_{k_2,t} b_{k_2,t}^*))^{((1,3),(2,4))}) \\
= \Vert h_{k_1,k_2,t} \Vert^2 \Vert b_{k_1,t} \Vert^2 \Vert b_{k_2,t} \Vert^2 R(h_{k_1,k_2,t})^{-\frac{1}{2}} R(b_{k_1,t})^{-\frac{1}{2}} R(b_{k_2,t})^{-\frac{1}{2}} \\ \cos(h_{k_1,k_2,t} h_{k_1,k_2,t}^*, ((b_{k_1,t} b_{k_1,t}^*) \otimes (b_{k_2,t} b_{k_2,t}^*))^{((1,3),(2,4))})
\end{multline*}
and
\begin{multline*}
\Vert (f_{k_1-1,t} \asd f_{k_2-1,t})^* (f_{k_1-1,t} \otimes f_{k_2-1,t})^{((1,3),(2,4))} \Vert^2 \\
= \langle (f_{k_1-1,t} \asd f_{k_2-1,t}) (f_{k_1-1,t} \asd f_{k_2-1,t})^*, ((f_{k_1-1,t} f_{k_1-1,t}^*) \otimes (f_{k_2-1,t} f_{k_2-1,t}^*))^{((1,3),(2,4))} \rangle \\
= \Vert (f_{k_1-1,t} \asd f_{k_2-1,t}) (f_{k_1-1,t} \asd f_{k_2-1,t})^* \Vert \Vert f_{k_1-1,t} f_{k_1-1,t}^* \Vert \Vert f_{k_2-1,t} f_{k_2-1,t}^* \Vert \\
\cos((f_{k_1-1,t} \asd f_{k_2-1,t}) (f_{k_1-1,t} \asd f_{k_2-1,t})^*, ((f_{k_1-1,t} f_{k_1-1,t}^*) \otimes (f_{k_2-1,t} f_{k_2-1,t}^*))^{((1,3),(2,4))}) \\
= \frac{ \Vert f_{k_1-1,t} \Vert^2 \Vert f_{k_2-1,t} \Vert^2 }{ R(f_{k_1-1,t})^{\frac{1}{2}} R(f_{k_2-1,t})^{\frac{1}{2}} } \frac{ \Vert f_{k_1-1,t}^{\asd 2} \Vert \Vert f_{k_2-1,t}^{\asd 2} \Vert }{ R(f_{k_1-1,t}^{\asd 2})^{\frac{1}{4}} R(f_{k_2-1,t}^{\asd 2})^{\frac{1}{4}} } \\
\cos((f_{k_1-1,t} \asd f_{k_2-1,t}) (f_{k_1-1,t} \asd f_{k_2-1,t})^*, ((f_{k_1-1,t} f_{k_1-1,t}^*) \otimes (f_{k_2-1,t} f_{k_2-1,t}^*))^{((1,3),(2,4))}) \\
\cos((f_{k_1-1,t}^{\asd 2})^* (f_{k_1-1,t}^{\asd 2}), (f_{k_2-1,t}^{\asd 2})^* (f_{k_2-1,t}^{\asd 2}))^{\frac{1}{2}},
\end{multline*}
where we used that
\begin{multline*}
\Vert (f_{k_1-1,t} \asd f_{k_2-1,t}) (f_{k_1-1,t} \asd f_{k_2-1,t})^* \Vert^2 
= \langle (f_{k_1-1,t}^{\asd 2})^* (f_{k_1-1,t}^{\asd 2}), (f_{k_2-1,t}^{\asd 2})^* (f_{k_2-1,t}^{\asd 2}) \rangle \\
= \Vert (f_{k_1-1,t}^{\asd 2})^* (f_{k_1-1,t}^{\asd 2}) \Vert \Vert (f_{k_2-1,t}^{\asd 2})^* (f_{k_2-1,t}^{\asd 2}) \Vert 
\cos((f_{k_1-1,t}^{\asd 2})^* (f_{k_1-1,t}^{\asd 2}), (f_{k_2-1,t}^{\asd 2})^* (f_{k_2-1,t}^{\asd 2})) \\
= \Vert f_{k_1-1,t}^{\asd 2} \Vert^2 \Vert f_{k_2-1,t}^{\asd 2} \Vert^2 R(f_{k_1-1,t}^{\asd 2})^{-\frac{1}{2}} R(f_{k_2-1,t}^{\asd 2})^{-\frac{1}{2}} 
\cos((f_{k_1-1,t}^{\asd 2})^* (f_{k_1-1,t}^{\asd 2}), (f_{k_2-1,t}^{\asd 2})^* (f_{k_2-1,t}^{\asd 2})).
\end{multline*}

We rewrite the second term for the case $k_1 < k_2$ as
\begin{multline*}
\langle j_{k_1,k_2-1,t}^* (b_{k_1,t} \otimes f_{k_2-1,t})^{((1,3),(2,4))}, (f_{k_1-1,t} \asd b_{k_2,t})^* (f_{k_1-1,t} \otimes b_{k_2,t})^{((1,3),(2,4))} \rangle \\
= \Vert j_{k_1,k_2-1,t}^* (b_{k_1,t} \otimes f_{k_2-1,t})^{((1,3),(2,4))} \Vert \Vert (f_{k_1-1,t} \asd b_{k_2,t})^* (f_{k_1-1,t} \otimes b_{k_2,t})^{((1,3),(2,4))} \Vert \\
\cos(j_{k_1,k_2-1,t}^* (b_{k_1,t} \otimes f_{k_2-1,t})^{((1,3),(2,4))}, (f_{k_1-1,t} \asd b_{k_2,t})^* (f_{k_1-1,t} \otimes b_{k_2,t})^{((1,3),(2,4))})
\end{multline*}
and again the two norms as
\begin{multline*}
\Vert j_{k_1,k_2-1,t}^* (b_{k_1,t} \otimes f_{k_2-1,t})^{((1,3),(2,4))} \Vert^2 \\
= \langle j_{k_1,k_2-1,t} j_{k_1,k_2-1,t}^*, ((b_{k_1,t} b_{k_1,t}^*) \otimes (f_{k_2-1,t} f_{k_2-1,t}^*))^{((1,3),(2,4))} \rangle \\
= \Vert j_{k_1,k_2-1,t} j_{k_1,k_2-1,t}^* \Vert \Vert b_{k_1,t} b_{k_1,t}^* \Vert \Vert f_{k_2-1,t} f_{k_2-1,t}^* \Vert \\
\cos(j_{k_1,k_2-1,t} j_{k_1,k_2-1,t}^*, ((b_{k_1,t} b_{k_1,t}^*) \otimes (f_{k_2-1,t} f_{k_2-1,t}^*))^{((1,3),(2,4))}) \\
= \Vert j_{k_1,k_2-1,t} \Vert^2 \Vert b_{k_1,t} \Vert^2 \Vert f_{k_2-1,t} \Vert^2 R(j_{k_1,k_2-1,t})^{-\frac{1}{2}} R(b_{k_1,t})^{-\frac{1}{2}} R(f_{k_2-1,t})^{-\frac{1}{2}} \\
\cos(j_{k_1,k_2-1,t} j_{k_1,k_2-1,t}^*, ((b_{k_1,t} b_{k_1,t}^*) \otimes (f_{k_2-1,t} f_{k_2-1,t}^*))^{((1,3),(2,4))})
\end{multline*}
and
\begin{multline*}
\Vert (f_{k_1-1,t} \asd b_{k_2,t})^* (f_{k_1-1,t} \otimes b_{k_2,t})^{((1,3),(2,4))} \Vert^2 \\
= \langle (f_{k_1-1,t} \asd b_{k_2,t}) (f_{k_1-1,t} \asd b_{k_2,t})^*, ((f_{k_1-1,t} f_{k_1-1,t}^*) \otimes (b_{k_2,t} b_{k_2,t}^*))^{((1,3),(2,4))} \rangle \\
= \Vert (f_{k_1-1,t} \asd b_{k_2,t}) (f_{k_1-1,t} \asd b_{k_2,t})^* \Vert \Vert f_{k_1-1,t} f_{k_1-1,t}^* \Vert \Vert b_{k_2,t} b_{k_2,t}^* \Vert \\
\cos((f_{k_1-1,t} \asd b_{k_2,t}) (f_{k_1-1,t} \asd b_{k_2,t})^*, ((f_{k_1-1,t} f_{k_1-1,t}^*) \otimes (b_{k_2,t} b_{k_2,t}^*))^{((1,3),(2,4))}) \\
= \frac{ \Vert f_{k_1-1,t} \Vert^2 \Vert b_{k_2,t} \Vert^2 }{ R(f_{k_1-1,t})^{\frac{1}{2}} R(b_{k_2,t})^{\frac{1}{2}} } \frac{ \Vert f_{k_1-1,t}^{\asd 2} \Vert \Vert b_{k_2,t}^{\asd 2} \Vert }{ R(f_{k_1-1,t}^{\asd 2})^{\frac{1}{4}} R(b_{k_2,t}^{\asd 2})^{\frac{1}{4}} } \\
\cos((f_{k_1-1,t} \asd b_{k_2,t}) (f_{k_1-1,t} \asd b_{k_2,t})^*, ((f_{k_1-1,t} f_{k_1-1,t}^*) \otimes (b_{k_2,t} b_{k_2,t}^*))^{((1,3),(2,4))}) \\
\cos((f_{k_1-1,t}^{\asd 2})^* (f_{k_1-1,t}^{\asd 2}), (b_{k_2,t}^{\asd 2})^* (b_{k_2,t}^{\asd 2}))^{\frac{1}{2}},
\end{multline*}
where we used that
\begin{multline*}
\Vert (f_{k_1-1,t} \asd b_{k_2,t}) (f_{k_1-1,t} \asd b_{k_2,t})^* \Vert^2
= \langle (f_{k_1-1,t}^{\asd 2})^* (f_{k_1-1,t}^{\asd 2}), (b_{k_2,t}^{\asd 2})^* (b_{k_2,t}^{\asd 2}) \rangle \\
= \Vert (f_{k_1-1,t}^{\asd 2})^* (f_{k_1-1,t}^{\asd 2}) \Vert \Vert (b_{k_2,t}^{\asd 2})^* (b_{k_2,t}^{\asd 2}) \Vert 
\cos((f_{k_1-1,t}^{\asd 2})^* (f_{k_1-1,t}^{\asd 2}), (b_{k_2,t}^{\asd 2})^* (b_{k_2,t}^{\asd 2})) \\
= \Vert f_{k_1-1,t}^{\asd 2} \Vert^2 \Vert b_{k_2,t}^{\asd 2} \Vert^2 R(f_{k_1-1,t}^{\asd 2})^{-\frac{1}{2}} R(b_{k_2,t}^{\asd 2})^{-\frac{1}{2}} 
\cos((f_{k_1-1,t}^{\asd 2})^* (f_{k_1-1,t}^{\asd 2}), (b_{k_2,t}^{\asd 2})^* (b_{k_2,t}^{\asd 2})).
\end{multline*}
We get the $k_1 > k_2$ case by symmetry.

Summarizing the above, we then have
\[
\langle (\Xi_t \nabla_t)^\otimes 2, \partial\nabla_t \rangle
= \sum_{k_1=1}^{l+1} \sum_{k_2=1}^{l+1} \xi_{k_1,t} \xi_{k_2,t} \langle \nabla_{k_1,t} \otimes \nabla_{k_2,t}, \partial\nabla_{k_1,k_2,t} \rangle
\]
with
\begin{multline*}
\langle \nabla_{k_1,t} \otimes \nabla_{k_2,t}, \partial\nabla_{k_1,k_2,t} \rangle \\
= \frac{ \Vert f_{k_1-1,t} \Vert \Vert f_{k_2-1,t} \Vert \Vert b_{k_1,t} \Vert \Vert b_{k_2,t} \Vert \Vert h_{k_1,k_2,t} \Vert \Vert f_{k_1-1,t}^{\asd 2} \Vert^{\frac{1}{2}}  \Vert f_{k_2-1,t}^{\asd 2} \Vert^{\frac{1}{2}} }{ R(f_{k_1-1,t})^{\frac{1}{4}} R(f_{k_2-1,t})^{\frac{1}{4}} R(b_{k_1,t})^{\frac{1}{4}} R(b_{k_2,t})^{\frac{1}{4}} R(h_{k_1,k_2,t})^{\frac{1}{4}} R(f_{k_1-1,t}^{\asd 2})^{\frac{1}{8}} R(f_{k_2-1,t}^{\asd 2})^{\frac{1}{8}} } \\
\cos\left( h_{k_1,k_2,t}^* (b_{k_1,t} \otimes b_{k_2,t})^{((1,3),(2,4))}, (f_{k_1-1,t} \asd f_{k_2-1,t})^* (f_{k_1-1,t} \otimes f_{k_2-1,t})^{((1,3),(2,4))} \right) \\
\cos\left( h_{k_1,k_2,t} h_{k_1,k_2,t}^*, ((b_{k_1,t} b_{k_1,t}^*) \otimes (b_{k_2,t} b_{k_2,t}^*))^{((1,3),(2,4))} \right)^{\frac{1}{2}} \\
\cos\left( (f_{k_1-1,t} \asd f_{k_2-1,t}) (f_{k_1-1,t} \asd f_{k_2-1,t})^*, ((f_{k_1-1,t} f_{k_1-1,t}^*) \otimes (f_{k_2-1,t} f_{k_2-1,t}^*))^{((1,3),(2,4))} \right)^{\frac{1}{2}} \\
\cos((f_{k_1-1,t}^{\asd 2})^* (f_{k_1-1,t}^{\asd 2}), (f_{k_2-1,t}^{\asd 2})^* (f_{k_2-1,t}^{\asd 2}))^{\frac{1}{4}} \\
+ \begin{cases}
\frac{ \Vert f_{k_1-1,t} \Vert \Vert f_{k_2-1,t} \Vert \Vert b_{k_1,t} \Vert \Vert b_{k_2,t} \Vert\Vert j_{k_1,k_2-1,t} \Vert \Vert f_{k_1-1,t}^{\asd 2} \Vert^{\frac{1}{2}} \Vert b_{k_2,t}^{\asd 2} \Vert^{\frac{1}{2}} }{ R(f_{k_1-1,t})^{\frac{1}{4}} R(f_{k_2-1,t})^{\frac{1}{4}} R(b_{k_1,t})^{\frac{1}{4}} R(b_{k_2,t})^{\frac{1}{4}} R(j_{k_1,k_2-1,t})^{\frac{1}{4}} R(f_{k_1-1,t}^{\asd 2})^{\frac{1}{8}} R(b_{k_2,t}^{\asd 2})^{\frac{1}{8}} } \\
\cos\left( j_{k_1,k_2-1,t}^* (b_{k_1,t} \otimes f_{k_2-1,t})^{((1,3),(2,4))}, (f_{k_1-1,t} \asd b_{k_2,t})^* (f_{k_1-1,t} \otimes b_{k_2,t})^{((1,3),(2,4))} \right) \\
\cos\left( j_{k_1,k_2-1,t} j_{k_1,k_2-1,t}^*, ((b_{k_1,t} b_{k_1,t}^*) \otimes (f_{k_2-1,t} f_{k_2-1,t}^*))^{((1,3),(2,4))} \right)^{\frac{1}{2}} \\
\cos\left( (f_{k_1-1,t} \asd b_{k_2,t}) (f_{k_1-1,t} \asd b_{k_2,t})^*, ((f_{k_1-1,t} f_{k_1-1,t}^*) \otimes (b_{k_2,t} b_{k_2,t}^*))^{((1,3),(2,4))} \right)^{\frac{1}{2}} \\
\cos((f_{k_1-1,t}^{\asd 2})^* (f_{k_1-1,t}^{\asd 2}), (b_{k_2,t}^{\asd 2})^* (b_{k_2,t}^{\asd 2}))^{\frac{1}{4}} \text{ if } k_1 < k_2, \\
\frac{ \Vert f_{k_1-1,t} \Vert \Vert f_{k_2-1,t} \Vert \Vert b_{k_1,t} \Vert \Vert b_{k_2,t} \Vert \Vert j_{k_2,k_1-1,t} \Vert \Vert f_{k_2-1,t}^{\asd 2} \Vert^{\frac{1}{2}} \Vert b_{k_1,t}^{\asd 2} \Vert^{\frac{1}{2}} }{ R(f_{k_1-1,t})^{\frac{1}{4}} R(f_{k_2-1,t})^{\frac{1}{4}} R(b_{k_1,t})^{\frac{1}{4}} R(b_{k_2,t})^{\frac{1}{4}} R(j_{k_2,k_1-1,t})^{\frac{1}{4}} R(f_{k_2-1,t}^{\asd 2})^{\frac{1}{8}} R(b_{k_1,t}^{\asd 2})^{\frac{1}{8}} } \\
\cos\left( j_{k_2,k_1-1,t}^* (b_{k_2,t} \otimes f_{k_1-1,t})^{((1,3),(2,4))}, (f_{k_2-1,t} \asd b_{k_1,t})^* (f_{k_2-1,t} \otimes b_{k_1,t})^{((1,3),(2,4))} \right) \\
\cos\left( j_{k_2,k_1-1,t} j_{k_2,k_1-1,t}^*, ((b_{k_2,t} b_{k_2,t}^*) \otimes (f_{k_1-1,t} f_{k_1-1,t}^*))^{((1,3),(2,4))} \right)^{\frac{1}{2}} \\
\cos\left( (f_{k_2-1,t} \asd b_{k_1,t}) (f_{k_2-1,t} \asd b_{k_1,t})^*, ((f_{k_2-1,t} f_{k_2-1,t}^*) \otimes (b_{k_1,t} b_{k_1,t}^*))^{((1,3),(2,4))} \right)^{\frac{1}{2}} \\
\cos((f_{k_2-1,t}^{\asd 2})^* (f_{k_2-1,t}^{\asd 2}), (b_{k_1,t}^{\asd 2})^* (b_{k_1,t}^{\asd 2}))^{\frac{1}{4}} \text{ if } k_1 > k_2 \text{ and } \\
0 \text{ if } k_1 = k_2.
\end{cases}
\end{multline*}
\end{proof}

\begin{proposition}[Third Order Term]\label{prop:third_order_term}
For all $t \in \N$, we have
\[
\langle (\Xi_t \nabla_t)^\otimes 3, \partial\partial\nabla_t \rangle
= \langle (D_{\xi_t} \tau_t)^{\otimes 3}, T^{(3)}_t \rangle
\]
with $T^{(3)}_t \in \R^{(l+1) \times (l+1) \times (l+1)}$ defined as
\begin{multline*}
T^{(3)}_{t,k,k,k}
= \frac{ \Vert f_{k-1,t}^{\asd 3} \Vert \Vert z_{k,k,k,t} \Vert }{ R(f_{k-1,t}^{\asd 3})^{\frac{1}{4}} R(z_{k,k,k,t})^{\frac{1}{4}} } 
\cos(z_{k,k,k,t}^* (b_{k,t}^{\otimes 3})^{(1,3,5),(2,4,6)}, (f_{k-1,t}^{\asd 3})^* (f_{k-1,t}^{\otimes 3})^{(1,3,5),(2,4,6)}) \\
\cos(z_{k,k,k,t} z_{k,k,k,t}^*, ((b_{k,t} b_{k,t}^*)^{\otimes 3})^{(1,3,5),(2,4,6)})^{\frac{1}{2}} \\
\cos((f_{k-1,t}^{\asd 3}) (f_{k-1,t}^{\asd 3})^*, ((f_{k-1,t} f_{k-1,t}^*)^{\otimes 3})^{(1,3,5),(2,4,6)})^{\frac{1}{2}} \text{ for } k \in [1:l+1],
\end{multline*}
\begin{multline*}
T^{(3)}_{t,k_1,k_1,k_2}
= \frac{ \Vert f_{k_1-1,t}^{\asd 3} \Vert^{\frac{2}{3}} \Vert f_{k_2-1,t}^{\asd 3} \Vert^{\frac{1}{3}} \Vert z_{k_1,k_1,k_2,t} \Vert }{ R(f_{k_1-1,t}^{\asd 3})^{\frac{1}{6}} R(f_{k_2-1,t}^{\asd 3})^{\frac{1}{12}} R(z_{k_1,k_1,k_2,t})^{\frac{1}{4}} } \\
\cos(z_{k_1,k_1,k_2,t}^* (b_{k_1,t}^{\otimes 2} \otimes b_{k_2,t})^{(1,3,5),(2,4,6)}, (f_{k_1-1,t}^{\asd 2} \asd f_{k_2-1,t})^* (f_{k_1-1,t}^{\otimes 2} \otimes f_{k_2-1,t})^{(1,3,5),(2,4,6)}) \\
\cos(z_{k_1,k_1,k_2,t} z_{k_1,k_1,k_2,t}^*, ((b_{k_1,t} b_{k_1,t}^*)^{\otimes 2} \otimes (b_{k_2,t} b_{k_2,t}^*))^{(1,3,5),(2,4,6)})^{\frac{1}{2}} \\
\cos((f_{k_1-1,t}^{\asd 2} \asd f_{k_2-1,t}) (f_{k_1-1,t}^{\asd 2} \asd f_{k_2-1,t})^*, ((f_{k_1-1,t} f_{k_1-1,t}^*)^{\otimes 2} \otimes (f_{k_2-1,t} f_{k_2-1,t}^*))^{(1,3,5),(2,4,6)})^{\frac{1}{2}} \\
\cos((f_{k_1-1,t}^{\asd 2})^* (f_{k_1-1,t}^{\asd 2}), (f_{k_1-1,t}^{\asd 2})^* (f_{k_1-1,t}^{\asd 2}), (f_{k_2-1,t}^{\asd 2})^* (f_{k_2-1,t}^{\asd 2}))^{\frac{1}{4}} 
\end{multline*}
\begin{multline*}
+ 2 \frac{ \Vert f_{k_1-1,t}^{\asd 2} \Vert \Vert h_{k_1,k_2,t}^{\asd 2} \Vert^{\frac{1}{2}} \Vert j_{k_1,k_2-1,t}^{\asd 2} \Vert^{\frac{1}{2}} }{ R(f_{k_1-1,t}^{\asd 2})^{\frac{1}{4}} R(h_{k_1,k_2,t}^{\asd 2})^{\frac{1}{8}} R(j_{k_1,k_2-1,t}^{\asd 2})^{\frac{1}{8}} } \\
\cos( (f_{k_1-1,t} \asd h_{k_1,k_2,t})^* (f_{k_1-1,t} \otimes b_{k_1,t} \otimes b_{k_2,t})^{(1,3,5),(2,4,6)}, \\
(f_{k_1-1,t} \asd j_{k_1,k_2-1,t})^* (b_{k_1,t} \otimes f_{k_1-1,t} \otimes f_{k_2-1,t})^{(3, 1, 5),(4, 2, 6)}) \\
\cos((f_{k_1-1,t} \asd h_{k_1,k_2,t}) (f_{k_1-1,t} \asd h_{k_1,k_2,t})^*, \\
((f_{k_1-1,t} f_{k_1-1,t}^*) \otimes (b_{k_1,t} b_{k_1,t}^*) \otimes (b_{k_2,t} b_{k_2,t}^*))^{(1,3,5),(2,4,6)})^{\frac{1}{2}} \\
\cos((f_{k_1-1,t} \asd j_{k_1,k_2-1,t}) (f_{k_1-1,t} \asd j_{k_1,k_2-1,t})^*, \\
((b_{k_1,t} b_{k_1,t}^*) \otimes (f_{k_1-1,t} f_{k_1-1,t}^*) \otimes (f_{k_2-1,t} f_{k_2-1,t}^*))^{(3, 1, 5),(4, 2, 6)})^{\frac{1}{2}} \\
\cos((f_{k_1-1,t}^{\asd 2})^* (f_{k_1-1,t}^{\asd 2}), (h_{k_1,k_2,t}^{\asd 2})^* (h_{k_1,k_2,t}^{\asd 2}))^{\frac{1}{4}} \\
\cos((f_{k_1-1,t}^{\asd 2})^* (f_{k_1-1,t}^{\asd 2}), (j_{k_1,k_2-1,t}^{\asd 2})^* (j_{k_1,k_2-1,t}^{\asd 2}))^{\frac{1}{4}},
\end{multline*}
$T^{(3)}_{t,k_1,k_2,k_1} = T^{(3)}_{t,k_1,k_1,k_2}$ and $T^{(3)}_{t,k_2,k_1,k_1} = T^{(3)}_{t,k_1,k_1,k_2}$ for $k_1 < k_2 \in [1:l+1]$,
\begin{multline*}
T^{(3)}_{k_1,k_1,k_2}
= \frac{ \Vert f_{k_1-1,t}^{\asd 3} \Vert^{\frac{2}{3}} \Vert f_{k_2-1,t}^{\asd 3} \Vert^{\frac{1}{3}} \Vert z_{k_1,k_1,k_2,t} \Vert }{ R(f_{k_1-1,t}^{\asd 3})^{\frac{1}{6}} R(f_{k_2-1,t}^{\asd 3})^{\frac{1}{12}} R(z_{k_1,k_1,k_2,t})^{\frac{1}{4}} } \\
\cos(z_{k_1,k_1,k_2,t}^* (b_{k_1,t}^{\otimes 2} \otimes b_{k_2,t})^{(1,3,5),(2,4,6)}, \\
(f_{k_1-1,t}^{\asd 2} \asd f_{k_2-1,t})^* (f_{k_1-1,t}^{\otimes 2} \otimes f_{k_2-1,t})^{(1,3,5),(2,4,6)}) \\
\cos(z_{k_1,k_1,k_2,t} z_{k_1,k_1,k_2,t}^*, ((b_{k_1,t} b_{k_1,t}^*)^{\otimes 2} \otimes (b_{k_2,t} b_{k_2,t}^*))^{(1,3,5),(2,4,6)})^{\frac{1}{2}} \\
\cos((f_{k_1-1,t}^{\asd 2} \asd f_{k_2-1,t}) (f_{k_1-1,t}^{\asd 2} \asd f_{k_2-1,t})^*, \\ 
((f_{k_1-1,t} f_{k_1-1,t}^*)^{\otimes 2} \otimes (f_{k_2-1,t} f_{k_2-1,t}^*))^{(1,3,5),(2,4,6)})^{\frac{1}{2}} \\
\cos((f_{k_1-1,t}^{\asd 2})^* (f_{k_1-1,t}^{\asd 2}), (f_{k_1-1,t}^{\asd 2})^* (f_{k_1-1,t}^{\asd 2}), (f_{k_2-1,t}^{\asd 2})^* (f_{k_2-1,t}^{\asd 2}))^{\frac{1}{4}} 
\end{multline*}
\begin{multline*}
+ 2 \frac{ \Vert f_{k_1-1,t}^{\asd 2} \Vert^{\frac{1}{2}} \Vert f_{k_2-1,t}^{\asd 2} \Vert^{\frac{1}{2}} \Vert h_{k_1,k_1,t}^{\asd 2} \Vert^{\frac{1}{2}} \Vert j_{k_2,k_1-1,t}^{\asd 2} \Vert^{\frac{1}{2}} }{ R(f_{k_1-1,t}^{\asd 2})^{\frac{1}{8}} R(f_{k_2-1,t}^{\asd 2})^{\frac{1}{8}} R(h_{k_1,k_1,t}^{\asd 2})^{\frac{1}{8}} R(j_{k_2,k_1-1,t}^{\asd 2})^{\frac{1}{8}} } \\
\cos( (f_{k_2-1,t} \asd h_{k_1,k_1,t})^* (f_{k_2-1,t} \otimes b_{k_1,t}^{\otimes })^{(1,3,5),(2,4,6)}, \\
(f_{k_1-1,t} \asd j_{k_2,k_1-1,t})^* (b_{k_2,t} \otimes f_{k_1-1,t}^{\otimes 2})^{(3, 1, 5),(4, 2, 6)}) \\
\cos((f_{k_2-1,t} \asd h_{k_1,k_1,t}) (f_{k_2-1,t} \asd h_{k_1,k_1,t})^*, \\
((f_{k_2-1,t} f_{k_2-1,t}^*) \otimes (b_{k_1,t} b_{k_1,t}^*)^{\otimes 2})^{(1,3,5),(2,4,6)})^{\frac{1}{2}} \\
\cos((f_{k_1-1,t} \asd j_{k_2,k_1-1,t}) (f_{k_1-1,t} \asd j_{k_2,k_1-1,t})^*, \\
((b_{k_2,t} b_{k_2,t}^*) \otimes (f_{k_1-1,t} f_{k_1-1,t}^*)^{\otimes 2})^{(3, 1, 5),(4, 2, 6)})^{\frac{1}{2}} \\
\cos((f_{k_2-1,t}^{\asd 2})^* (f_{k_2-1,t}^{\asd 2}), (h_{k_1,k_1,t}^{\asd 2})^* (h_{k_1,k_1,t}^{\asd 2}))^{\frac{1}{4}} \\
\cos((f_{k_1-1,t}^{\asd 2})^* (f_{k_1-1,t}^{\asd 2}), (j_{k_2,k_1-1,t}^{\asd 2})^* (j_{k_2,k_1-1,t}^{\asd 2}))^{\frac{1}{4}},
\end{multline*}
$T^{(3)}_{t,k_1,k_2,k_1} = T^{(3)}_{t,k_1,k_1,k_2}$ and $T^{(3)}_{t,k_2,k_1,k_1} = T^{(3)}_{t,k_1,k_1,k_2}$ for $k_1 > k_2 \in [1:l+1]$,
\begin{multline*}
T^{(3)}_{t,k_1,k_2,k_3}
= \frac{ \Vert f_{k_1-1,t}^{\asd 3} \Vert^{\frac{1}{3}} \Vert f_{k_2-1,t}^{\asd 3} \Vert^{\frac{1}{3}} \Vert f_{k_3-1,t}^{\asd 3} \Vert^{\frac{1}{3}} \Vert z_{k_1,k_2,k_3,t} \Vert }{ R(f_{k_1-1,t}^{\asd 3})^{\frac{1}{12}} R(f_{k_2-1,t}^{\asd 3})^{\frac{1}{12}} R(f_{k_3-1,t}^{\asd 3})^{\frac{1}{12}} R(z_{k_1,k_2,k_3,t})^{\frac{1}{4}} } \\
\cos(z_{k_1,k_2,k_3,t}^* (b_{k_1,t} \otimes b_{k_2,t} \otimes b_{k_3,t})^{(1,3,5),(2,4,6)}, \\
(f_{k_1-1,t} \asd f_{k_2-1,t} \asd f_{k_3-1,t})^* (f_{k_1-1,t} \otimes f_{k_2-1,t} \otimes f_{k_3-1,t})^{(1,3,5),(2,4,6)}) \\
\cos(z_{k_1,k_2,k_3,t} z_{k_1,k_2,k_3,t}^*, ((b_{k_1,t} b_{k_1,t}^*) \otimes (b_{k_2,t} b_{k_2,t}^*) \otimes (b_{k_3,t} b_{k_3,t}^*))^{(1,3,5),(2,4,6)})^{\frac{1}{2}} \\
\cos((f_{k_1-1,t} \asd f_{k_2-1,t} \asd f_{k_3-1,t}) (f_{k_1-1,t} \asd f_{k_2-1,t} \asd f_{k_3-1,t})^*, \\ 
((f_{k_1-1,t} f_{k_1-1,t}^*) \otimes (f_{k_2-1,t} f_{k_2-1,t}^*) \otimes (f_{k_3-1,t} f_{k_3-1,t}^*))^{(1,3,5),(2,4,6)})^{\frac{1}{2}} \\
\cos((f_{k_1-1,t}^{\asd 2})^* (f_{k_1-1,t}^{\asd 2}), (f_{k_2-1,t}^{\asd 2})^* (f_{k_2-1,t}^{\asd 2}), (f_{k_3-1,t}^{\asd 2})^* (f_{k_3-1,t}^{\asd 2}))^{\frac{1}{4}} 
\end{multline*}
\begin{multline*}
+ \frac{ \Vert f_{k_1-1,t}^{\asd 2} \Vert^{\frac{1}{2}} \Vert f_{k_2-1,t}^{\asd 2} \Vert^{\frac{1}{2}} \Vert h_{k_2,k_3,t}^{\asd 2} \Vert^{\frac{1}{2}} \Vert j_{k_1,k_3-1,t}^{\asd 2} \Vert^{\frac{1}{2}} }{ R(f_{k_1-1,t}^{\asd 2})^{\frac{1}{8}} R(f_{k_2-1,t}^{\asd 2})^{\frac{1}{8}} R(h_{k_2,k_3,t}^{\asd 2})^{\frac{1}{8}} R(j_{k_1,k_3-1,t}^{\asd 2})^{\frac{1}{8}} } \\
\cos( (f_{k_1-1,t} \asd h_{k_2,k_3,t})^* (f_{k_1-1,t} \otimes b_{k_2,t} \otimes b_{k_3,t})^{(1,3,5),(2,4,6)}, \\
(f_{k_2-1,t} \asd j_{k_1,k_3-1,t})^* (b_{k_1,t} \otimes f_{k_2-1,t} \otimes f_{k_3-1,t})^{(3, 1, 5),(4, 2, 6)}) \\
\cos((f_{k_1-1,t} \asd h_{k_2,k_3,t}) (f_{k_1-1,t} \asd h_{k_2,k_3,t})^*, \\
((f_{k_1-1,t} f_{k_1-1,t}^*) \otimes (b_{k_2,t} b_{k_2,t}^*) \otimes (b_{k_3,t} b_{k_3,t}^*))^{(1,3,5),(2,4,6)})^{\frac{1}{2}} \\
\cos((f_{k_2-1,t} \asd j_{k_1,k_3-1,t}) (f_{k_2-1,t} \asd j_{k_1,k_3-1,t})^*, \\
((b_{k_1,t} b_{k_1,t}^*) \otimes (f_{k_2-1,t} f_{k_2-1,t}^*) \otimes (f_{k_3-1,t} f_{k_3-1,t}^*))^{(3, 1, 5),(4, 2, 6)})^{\frac{1}{2}} \\
\cos((f_{k_1-1,t}^{\asd 2})^* (f_{k_1-1,t}^{\asd 2}), (h_{k_2,k_3,t}^{\asd 2})^* (h_{k_2,k_3,t}^{\asd 2}))^{\frac{1}{4}} \\
\cos((f_{k_2-1,t}^{\asd 2})^* (f_{k_2-1,t}^{\asd 2}), (j_{k_1,k_3-1,t}^{\asd 2})^* (j_{k_1,k_3-1,t}^{\asd 2}))^{\frac{1}{4}} 
\end{multline*}
\begin{multline*}
+ \frac{ \Vert f_{k_1-1,t}^{\asd 2} \Vert^{\frac{1}{2}} \Vert f_{k_3-1,t}^{\asd 2} \Vert^{\frac{1}{2}} \Vert h_{k_2,k_3,t}^{\asd 2} \Vert^{\frac{1}{2}} \Vert j_{k_1,k_2-1,t}^{\asd 2} \Vert^{\frac{1}{2}} }{ R(f_{k_1-1,t}^{\asd 2})^{\frac{1}{8}} R(f_{k_3-1,t}^{\asd 2})^{\frac{1}{8}} R(h_{k_2,k_3,t}^{\asd 2})^{\frac{1}{8}} R(j_{k_1,k_2-1,t}^{\asd 2})^{\frac{1}{8}} } \\
\cos( (f_{k_1-1,t} \asd h_{k_2,k_3,t})^* (f_{k_1-1,t} \otimes b_{k_2,t} \otimes b_{k_3,t})^{(1,3,5),(2,4,6)}, \\
(f_{k_3-1,t} \asd j_{k_1,k_2-1,t})^* (b_{k_1,t} \otimes f_{k_2-1,t} \otimes f_{k_3-1,t})^{(5, 1, 3),(6, 2, 4)}) \\
\cos((f_{k_1-1,t} \asd h_{k_2,k_3,t}) (f_{k_1-1,t} \asd h_{k_2,k_3,t})^*, \\
((f_{k_1-1,t} f_{k_1-1,t}^*) \otimes (b_{k_2,t} b_{k_2,t}^*) \otimes (b_{k_3,t} b_{k_3,t}^*))^{(1,3,5),(2,4,6)})^{\frac{1}{2}} \\
\cos((f_{k_3-1,t} \asd j_{k_1,k_2-1,t}) (f_{k_3-1,t} \asd j_{k_1,k_2-1,t})^*, \\
((b_{k_1,t} b_{k_1,t}^*) \otimes (f_{k_2-1,t} f_{k_2-1,t}^*) \otimes (f_{k_3-1,t} f_{k_3-1,t}^*))^{(3, 1, 5),(4, 2, 6)})^{\frac{1}{2}} \\
\cos((f_{k_1-1,t}^{\asd 2})^* (f_{k_1-1,t}^{\asd 2}), (h_{k_2,k_3,t}^{\asd 2})^* (h_{k_2,k_3,t}^{\asd 2}))^{\frac{1}{4}} \\
\cos((f_{k_3-1,t}^{\asd 2})^* (f_{k_3-1,t}^{\asd 2}), (j_{k_1,k_2-1,t}^{\asd 2})^* (j_{k_1,k_2-1,t}^{\asd 2}))^{\frac{1}{4}} 
\end{multline*}
\begin{multline*}
+ \frac{ \Vert f_{k_1-1,t}^{\asd 2} \Vert^{\frac{1}{2}} \Vert f_{k_2-1,t}^{\asd 2} \Vert^{\frac{1}{2}} \Vert h_{k_1,k_3,t}^{\asd 2} \Vert^{\frac{1}{2}} \Vert j_{k_2,k_3-1,t}^{\asd 2} \Vert^{\frac{1}{2}} }{ R(f_{k_1-1,t}^{\asd 2})^{\frac{1}{8}} R(f_{k_2-1,t}^{\asd 2})^{\frac{1}{8}} R(h_{k_1,k_3,t}^{\asd 2})^{\frac{1}{8}} R(j_{k_2,k_3-1,t}^{\asd 2})^{\frac{1}{8}} } \\
\cos( (f_{k_2-1,t} \asd h_{k_1,k_3,t})^* (b_{k_1,t} \otimes f_{k_2-1,t} \otimes b_{k_3,t})^{(3,1,5),(4,2,6)}, \\
(f_{k_1-1,t} \asd j_{k_2,k_3-1,t})^* (f_{k_1-1,t} \otimes b_{k_2,t} \otimes f_{k_3-1,t})^{(1, 3, 5),(2, 4, 6)}) \\
\cos((f_{k_2-1,t} \asd h_{k_1,k_3,t}) (f_{k_2-1,t} \asd h_{k_1,k_3,t})^*, \\
((b_{k_1,t} b_{k_1,t}^*) \otimes (f_{k_2-1,t} f_{k_2-1,t}^*) \otimes (b_{k_3,t} b_{k_3,t}^*))^{(1,3,5),(2,4,6)})^{\frac{1}{2}} \\
\cos((f_{k_1-1,t} \asd j_{k_2,k_3-1,t}) (f_{k_1-1,t} \asd j_{k_2,k_3-1,t})^*, \\
((f_{k_1-1,t} f_{k_1-1,t}^*) \otimes (b_{k_2,t} b_{k_2,t}^*) \otimes (f_{k_3-1,t} f_{k_3-1,t}^*))^{(3, 1, 5),(4, 2, 6)})^{\frac{1}{2}} \\
\cos((f_{k_2-1,t}^{\asd 2})^* (f_{k_2-1,t}^{\asd 2}), (h_{k_1,k_3,t}^{\asd 2})^* (h_{k_1,k_3,t}^{\asd 2}))^{\frac{1}{4}} \\
\cos((f_{k_1-1,t}^{\asd 2})^* (f_{k_1-1,t}^{\asd 2}), (j_{k_2,k_3-1,t}^{\asd 2})^* (j_{k_2,k_3-1,t}^{\asd 2}))^{\frac{1}{4}} 
\end{multline*}
\begin{multline*}
+ \frac{ \Vert b_{k_3,t}^{\asd 2} \Vert^{\frac{1}{2}} \Vert j_{k_1,k_2-1,t}^{\asd 2} \Vert^{\frac{1}{2}} \Vert f_{k_1-1,t}^{\asd 2} \Vert^{\frac{1}{2}} \Vert j_{k_2,k_3-1,t}^{\asd 2} \Vert^{\frac{1}{2}} }{ R(b_{k_3,t}^{\asd 2})^{\frac{1}{8}} R(j_{k_1,k_2-1,t}^{\asd 2})^{\frac{1}{8}} R(f_{k_1-1,t}^{\asd 2})^{\frac{1}{8}} R(j_{k_2,k_3-1,t}^{\asd 2})^{\frac{1}{8}} } \\
\cos((b_{k_3,t} \asd j_{k_1,k_2-1,t})^* (b_{k_1,t} \otimes f_{k_2-1,t} \otimes b_{k_3,t})^{(5,1,3),(6,2,4)}, \\
(f_{k_1-1,t} \asd j_{k_2,k_3-1,t})^* (f_{k_1-1,t} \otimes b_{k_2,t} \otimes f_{k_3-1,t})^{(1, 3, 5),(2, 4, 6)}) \\
\cos((b_{k_3,t} \asd j_{k_1,k_2-1,t}) (b_{k_3,t} \asd j_{k_1,k_2-1,t})^*, \\
((b_{k_1,t} b_{k_1,t}^*) \otimes (f_{k_2-1,t} f_{k_2-1,t}^*) \otimes (b_{k_3,t} b_{k_3,t}^*))^{(5,1,3),(6,2,4)})^{\frac{1}{2}} \\
\cos((f_{k_1-1,t} \asd j_{k_2,k_3-1,t}) (f_{k_1-1,t} \asd j_{k_2,k_3-1,t})^*, \\
((f_{k_1-1,t} f_{k_1-1,t}^*) \otimes (b_{k_2,t} b_{k_2,t}^*) \otimes (f_{k_3-1,t} f_{k_3-1,t}^*))^{(1, 3, 5),(2, 4, 6)})^{\frac{1}{2}} \\
\cos((b_{k_3,t}^{\asd 2})^* (b_{k_3,t}^{\asd 2}), (j_{k_1,k_2-1,t}^{\asd 2})^* (j_{k_1,k_2-1,t}^{\asd 2}))^{\frac{1}{4}} \\
\cos((f_{k_1-1,t}^{\asd 2})^* (f_{k_1-1,t}^{\asd 2}), (j_{k_2,k_3-1,t}^{\asd 2})^* (j_{k_2,k_3-1,t}^{\asd 2}))^{\frac{1}{4}},
\end{multline*}
$T^{(3)}_{t,k_1,k_3,k_2} = T^{(3)}_{t,k_1,k_2,k_3}$, $T^{(3)}_{t,k_3,k_2,k_1} = T^{(3)}_{t,k_1,k_2,k_3}$ and $T^{(3)}_{t,k_2,k_1,k_3} = T^{(3)}_{t,k_1,k_2,k_3}$ for $k_1 < k_2 < k_3 \in [1:l+1]$.
\end{proposition}
\begin{proof}
We focus on $\partial\partial\nabla_{k_1,k_2,k_3}$ from Proposition~\ref{prop:partial_partial_nabla_k_1_k_2_k_3_t_factorization}, and the rest will follow by dropping terms and replacing indices. We have
\begin{multline*}
\langle \nabla_{k_1,t} \otimes \nabla_{k_2,t} \otimes \nabla_{k_3,t}, \partial\partial\nabla_{k_1,k_2,k_3,t} \rangle \\
= \frac{1}{n^4} \sum_{i_1=1}^n \sum_{i_2=1}^n \sum_{i_3=1}^n \sum_{i_4=1}^n \langle b_{k_1,i_1,t} \otimes b_{k_2,i_2,t} \otimes b_{k_3,i_3,t}, z_{k_1,k_2,k_3,i_4,t} \rangle \\
\langle f_{k_1-1,i_1,t} \otimes f_{k_2-1,i_2,t} \otimes f_{k_3-1,i_3,t}, f_{k_1-1,i_4,t} \otimes f_{k_2-1,i_4,t} \otimes f_{k_3-1,i_4,t} \rangle \\
+ \langle f_{k_1-1,i_1,t} \otimes b_{k_2,i_2,t} \otimes b_{k_3,i_3,t}, f_{k_1-1,i_4,t} \otimes h_{k_2,k_3,i_4,t} \rangle \\
\langle b_{k_1,i_1,t} \otimes f_{k_2-1,i_2,t} \otimes f_{k_3-1,i_3,t}, f_{k_2-1,i_4,t} \otimes j_{k_1,k_3-1,i_4,t} \rangle \\
+ \langle f_{k_1-1,i_1,t} \otimes b_{k_2,i_2,t} \otimes b_{k_3,i_3,t}, f_{k_1-1,i_4,t} \otimes h_{k_2,k_3,i_4,t} \rangle \\
\langle b_{k_1,i_1,t} \otimes f_{k_2-1,i_2,t} \otimes f_{k_3-1,i_3,t}, f_{k_3-1,i_4,t} \otimes j_{k_1,k_2-1,i_4,t} \rangle \\
+ \langle b_{k_1,i_1,t} \otimes f_{k_2-1,i_2,t} \otimes b_{k_3,i_3,t}, f_{k_2-1,i_4,t} \otimes h_{k_1,k_3,i_4,t} \rangle \\
\langle f_{k_1-1,i_1,t} \otimes b_{k_2,i_2,t} \otimes f_{k_3-1,i_3,t}, f_{k_1-1,i_4,t} \otimes j_{k_2,k_3-1,i_4,t} \rangle \\
+ \langle b_{k_1,i_1,t} \otimes f_{k_2-1,i_2,t} \otimes b_{k_3,i_3,t}, b_{k_3,i_4,t} \otimes j_{k_1,k_2-1,i_4,t} \rangle \\
\langle f_{k_1-1,i_1,t} \otimes b_{k_2,i_2,t} \otimes f_{k_3-1,i_3,t}, f_{k_1-1,i_4,t} \otimes j_{k_2,k_3-1,i_4,t} \rangle \\
= \langle z_{k_1,k_2,k_3,t}^* (b_{k_1,t} \otimes b_{k_2,t} \otimes b_{k_3,t})^{(1,3,5),(2,4,6)}, \\
(f_{k_1-1,t} \asd f_{k_2-1,t} \asd f_{k_3-1,t})^* (f_{k_1-1,t} \otimes f_{k_2-1,t} \otimes f_{k_3-1,t})^{(1,3,5),(2,4,6)} \rangle \\
+ \langle (f_{k_1-1,t} \asd h_{k_2,k_3,t})^* (f_{k_1-1,t} \otimes b_{k_2,t} \otimes b_{k_3,t})^{(1,3,5),(2,4,6)}, \\
(f_{k_2-1,t} \asd j_{k_1,k_3-1,t})^* (b_{k_1,t} \otimes f_{k_2-1,t} \otimes f_{k_3-1,t})^{(3, 1, 5),(4, 2, 6)} \rangle \\
+ \langle (f_{k_1-1,t} \asd h_{k_2,k_3,t})^* (f_{k_1-1,t} \otimes b_{k_2,t} \otimes b_{k_3,t})^{(1,3,5),(2,4,6)}, \\
(f_{k_3-1,t} \asd j_{k_1,k_2-1,t})^* (b_{k_1,t} \otimes f_{k_2-1,t} \otimes f_{k_3-1,t})^{(5, 1, 3),(6, 2, 4)} \rangle \\
+ \langle (f_{k_2-1,t} \asd h_{k_1,k_3,t})^* (b_{k_1,t} \otimes f_{k_2-1,t} \otimes b_{k_3,t})^{(3,1,5),(4,2,6)}, \\
(f_{k_1-1,t} \asd j_{k_2,k_3-1,t})^* (f_{k_1-1,t} \otimes b_{k_2,t} \otimes f_{k_3-1,t})^{(1, 3, 5),(2, 4, 6)} \rangle \\
+ \langle (b_{k_3,t} \asd j_{k_1,k_2-1,t})^* (b_{k_1,t} \otimes f_{k_2-1,t} \otimes b_{k_3,t})^{(5,1,3),(6,2,4)}, \\
(f_{k_1-1,t} \asd j_{k_2,k_3-1,t})^* (f_{k_1-1,t} \otimes b_{k_2,t} \otimes f_{k_3-1,t})^{(1, 3, 5),(2, 4, 6)} \rangle.
\end{multline*}
We rewrite the first, second and last inner products above, as the third and fourth are analogous to the second. For the first, we have
\begin{multline*}
\langle z_{k_1,k_2,k_3,t}^* (b_{k_1,t} \otimes b_{k_2,t} \otimes b_{k_3,t})^{(1,3,5),(2,4,6)}, \\
(f_{k_1-1,t} \asd f_{k_2-1,t} \asd f_{k_3-1,t})^* (f_{k_1-1,t} \otimes f_{k_2-1,t} \otimes f_{k_3-1,t})^{(1,3,5),(2,4,6)} \rangle \\
= \Vert z_{k_1,k_2,k_3,t}^* (b_{k_1,t} \otimes b_{k_2,t} \otimes b_{k_3,t})^{(1,3,5),(2,4,6)} \Vert \\
\Vert (f_{k_1-1,t} \asd f_{k_2-1,t} \asd f_{k_3-1,t})^* (f_{k_1-1,t} \otimes f_{k_2-1,t} \otimes f_{k_3-1,t})^{(1,3,5),(2,4,6)} \Vert \\
\cos(z_{k_1,k_2,k_3,t}^* (b_{k_1,t} \otimes b_{k_2,t} \otimes b_{k_3,t})^{(1,3,5),(2,4,6)}, \\
(f_{k_1-1,t} \asd f_{k_2-1,t} \asd f_{k_3-1,t})^* (f_{k_1-1,t} \otimes f_{k_2-1,t} \otimes f_{k_3-1,t})^{(1,3,5),(2,4,6)}).
\end{multline*}
We rewrite the norms as
\begin{multline*}
\Vert z_{k_1,k_2,k_3,t}^* (b_{k_1,t} \otimes b_{k_2,t} \otimes b_{k_3,t})^{(1,3,5),(2,4,6)} \Vert^2 \\
= \langle z_{k_1,k_2,k_3,t} z_{k_1,k_2,k_3,t}^*, ((b_{k_1,t} b_{k_1,t}^*) \otimes (b_{k_2,t} b_{k_2,t}^*) \otimes (b_{k_3,t} b_{k_3,t}^*)^{(1,3,5),(2,4,6)} \rangle \\
= \Vert z_{k_1,k_2,k_3,t} z_{k_1,k_2,k_3,t}^* \Vert \Vert b_{k_1,t} b_{k_1,t}^* \Vert \Vert b_{k_2,t} b_{k_2,t}^* \Vert \Vert b_{k_3,t} b_{k_3,t}^* \Vert \\
\cos(z_{k_1,k_2,k_3,t} z_{k_1,k_2,k_3,t}^*, ((b_{k_1,t} b_{k_1,t}^*) \otimes (b_{k_2,t} b_{k_2,t}^*) \otimes (b_{k_3,t} b_{k_3,t}^*)^{(1,3,5),(2,4,6)}) \\
= \frac{ \Vert z_{k_1,k_2,k_3,t} \Vert^2 \Vert b_{k_1,t} \Vert^2 \Vert b_{k_2,t} \Vert^2 \Vert b_{k_3,t} \Vert^2 }{ R(z_{k_1,k_2,k_3,t})^{\frac{1}{2}} R(b_{k_1,t})^{\frac{1}{2}} R(b_{k_2,t})^{\frac{1}{2}} R(b_{k_3,t})^{\frac{1}{2}} } \\
\cos(z_{k_1,k_2,k_3,t} z_{k_1,k_2,k_3,t}^*, ((b_{k_1,t} b_{k_1,t}^*) \otimes (b_{k_2,t} b_{k_2,t}^*) \otimes (b_{k_3,t} b_{k_3,t}^*)^{(1,3,5),(2,4,6)})
\end{multline*}
and
\begin{multline*}
\Vert (f_{k_1-1,t} \asd f_{k_2-1,t} \asd f_{k_3-1,t})^* (f_{k_1-1,t} \otimes f_{k_2-1,t} \otimes f_{k_3-1,t})^{(1,3,5),(2,4,6)} \Vert^2 \\
= \langle (f_{k_1-1,t} \asd f_{k_2-1,t} \asd f_{k_3-1,t}) (f_{k_1-1,t} \asd f_{k_2-1,t} \asd f_{k_3-1,t})^*, \\ 
((f_{k_1-1,t} f_{k_1-1,t}^*) \otimes (f_{k_2-1,t} f_{k_2-1,t}^*) \otimes (f_{k_3-1,t} f_{k_3-1,t}^*)^{(1,3,5),(2,4,6)} \rangle \\
= \Vert (f_{k_1-1,t} \asd f_{k_2-1,t} \asd f_{k_3-1,t}) (f_{k_1-1,t} \asd f_{k_2-1,t} \asd f_{k_3-1,t})^* \Vert \\
\Vert f_{k_1-1,t} f_{k_1-1,t}^* \Vert \Vert f_{k_2-1,t} f_{k_2-1,t}^* \Vert \Vert f_{k_3-1,t} f_{k_3-1,t}^* \Vert \\
\cos((f_{k_1-1,t} \asd f_{k_2-1,t} \asd f_{k_3-1,t}) (f_{k_1-1,t} \asd f_{k_2-1,t} \asd f_{k_3-1,t})^*, \\ 
((f_{k_1-1,t} f_{k_1-1,t}^*) \otimes (f_{k_2-1,t} f_{k_2-1,t}^*) \otimes (f_{k_3-1,t} f_{k_3-1,t}^*)^{(1,3,5),(2,4,6)}) \\
= \frac{ \Vert f_{k_1-1,t} \Vert^2 \Vert f_{k_2-1,t} \Vert^2 \Vert f_{k_3-1,t} \Vert^2 \Vert f_{k_1-1,t}^{\asd 3} \Vert^{\frac{2}{3}} \Vert f_{k_2-1,t}^{\asd 3} \Vert^{\frac{2}{3}} \Vert f_{k_3-1,t}^{\asd 3} \Vert^{\frac{2}{3}} }{ R(f_{k_1-1,t})^{\frac{1}{2}} R(f_{k_2-1,t})^{\frac{1}{2}} R(f_{k_3-1,t})^{\frac{1}{2}} R(f_{k_1-1,t}^{\asd 3})^{\frac{1}{6}} R(f_{k_2-1,t}^{\asd 3})^{\frac{1}{6}} R(f_{k_3-1,t}^{\asd 3})^{\frac{1}{6}} } \\
\cos((f_{k_1-1,t} \asd f_{k_2-1,t} \asd f_{k_3-1,t}) (f_{k_1-1,t} \asd f_{k_2-1,t} \asd f_{k_3-1,t})^*, \\ 
((f_{k_1-1,t} f_{k_1-1,t}^*) \otimes (f_{k_2-1,t} f_{k_2-1,t}^*) \otimes (f_{k_3-1,t} f_{k_3-1,t}^*)^{(1,3,5),(2,4,6)}) \\
\cos((f_{k_1-1,t}^{\asd 2})^* (f_{k_1-1,t}^{\asd 2}), (f_{k_2-1,t}^{\asd 2})^* (f_{k_2-1,t}^{\asd 2}), (f_{k_3-1,t}^{\asd 2})^* (f_{k_3-1,t}^{\asd 2}))^{\frac{1}{2}},
\end{multline*}
where we used that
\begin{multline*}
\Vert (f_{k_1-1,t} \asd f_{k_2-1,t} \asd f_{k_3-1,t}) (f_{k_1-1,t} \asd f_{k_2-1,t} \asd f_{k_3-1,t})^* \Vert^2 \\
= \tr( ((f_{k_1-1,t}^{\asd 2})^* (f_{k_1-1,t}^{\asd 2})) ((f_{k_2-1,t}^{\asd 2})^* (f_{k_2-1,t}^{\asd 2})) ((f_{k_3-1,t}^{\asd 2})^* (f_{k_3-1,t}^{\asd 2})) ) \\
= \Vert (f_{k_1-1,t}^{\asd 2})^* (f_{k_1-1,t}^{\asd 2}) \Vert_3 \Vert (f_{k_2-1,t}^{\asd 2})^* (f_{k_2-1,t}^{\asd 2}) \Vert_3 \Vert (f_{k_3-1,t}^{\asd 2})^* (f_{k_3-1,t}^{\asd 2}) \Vert_3 \\
\cos((f_{k_1-1,t}^{\asd 2})^* (f_{k_1-1,t}^{\asd 2}), (f_{k_2-1,t}^{\asd 2})^* (f_{k_2-1,t}^{\asd 2}), (f_{k_3-1,t}^{\asd 2})^* (f_{k_3-1,t}^{\asd 2})) \\
= \Vert (f_{k_1-1,t}^{\asd 3})^* (f_{k_1-1,t}^{\asd 3}) \Vert^{\frac{2}{3}} \Vert (f_{k_2-1,t}^{\asd 3})^* (f_{k_2-1,t}^{\asd 3}) \Vert^{\frac{2}{3}} \Vert (f_{k_3-1,t}^{\asd 3})^* (f_{k_3-1,t}^{\asd 3}) \Vert^{\frac{2}{3}} \\
\cos((f_{k_1-1,t}^{\asd 2})^* (f_{k_1-1,t}^{\asd 2}), (f_{k_2-1,t}^{\asd 2})^* (f_{k_2-1,t}^{\asd 2}), (f_{k_3-1,t}^{\asd 2})^* (f_{k_3-1,t}^{\asd 2})) \\
= \frac{ \Vert f_{k_1-1,t}^{\asd 3} \Vert^{\frac{4}{3}} \Vert f_{k_2-1,t}^{\asd 3} \Vert^{\frac{4}{3}} \Vert f_{k_3-1,t}^{\asd 3} \Vert^{\frac{4}{3}} }{ R(f_{k_1-1,t}^{\asd 3})^{\frac{1}{3}} R(f_{k_2-1,t}^{\asd 3})^{\frac{1}{3}} R(f_{k_3-1,t}^{\asd 3})^{\frac{1}{3}} } \\
\cos((f_{k_1-1,t}^{\asd 2})^* (f_{k_1-1,t}^{\asd 2}), (f_{k_2-1,t}^{\asd 2})^* (f_{k_2-1,t}^{\asd 2}), (f_{k_3-1,t}^{\asd 2})^* (f_{k_3-1,t}^{\asd 2})).
\end{multline*}
The second term can be written as
\begin{multline*}
\langle (f_{k_1-1,t} \asd h_{k_2,k_3,t})^* (f_{k_1-1,t} \otimes b_{k_2,t} \otimes b_{k_3,t})^{(1,3,5),(2,4,6)}, \\
(f_{k_2-1,t} \asd j_{k_1,k_3-1,t})^* (b_{k_1,t} \otimes f_{k_2-1,t} \otimes f_{k_3-1,t})^{(3, 1, 5),(4, 2, 6)} \rangle \\
= \Vert (f_{k_1-1,t} \asd h_{k_2,k_3,t})^* (f_{k_1-1,t} \otimes b_{k_2,t} \otimes b_{k_3,t})^{(1,3,5),(2,4,6)} \Vert \\
\Vert (f_{k_2-1,t} \asd j_{k_1,k_3-1,t})^* (b_{k_1,t} \otimes f_{k_2-1,t} \otimes f_{k_3-1,t})^{(3, 1, 5),(4, 2, 6)} \Vert \\
\cos( (f_{k_1-1,t} \asd h_{k_2,k_3,t})^* (f_{k_1-1,t} \otimes b_{k_2,t} \otimes b_{k_3,t})^{(1,3,5),(2,4,6)}, \\
(f_{k_2-1,t} \asd j_{k_1,k_3-1,t})^* (b_{k_1,t} \otimes f_{k_2-1,t} \otimes f_{k_3-1,t})^{(3, 1, 5),(4, 2, 6)}).
\end{multline*}
We rewrite the first norm as
\begin{multline*}
\Vert (f_{k_1-1,t} \asd h_{k_2,k_3,t})^* (f_{k_1-1,t} \otimes b_{k_2,t} \otimes b_{k_3,t})^{(1,3,5),(2,4,6)} \Vert^2 \\
= \langle (f_{k_1-1,t} \asd h_{k_2,k_3,t}) (f_{k_1-1,t} \asd h_{k_2,k_3,t})^*, \\
((f_{k_1-1,t} f_{k_1-1,t}^*) \otimes (b_{k_2,t} b_{k_2,t}^*) \otimes (b_{k_3,t} b_{k_3,t}^*))^{(1,3,5),(2,4,6)} \rangle \\
= \Vert (f_{k_1-1,t} \asd h_{k_2,k_3,t}) (f_{k_1-1,t} \asd h_{k_2,k_3,t})^* \Vert 
\Vert f_{k_1-1,t} f_{k_1-1,t}^* \Vert \Vert b_{k_2,t} b_{k_2,t}^* \Vert \Vert b_{k_3,t} b_{k_3,t}^* \Vert \\
\cos((f_{k_1-1,t} \asd h_{k_2,k_3,t}) (f_{k_1-1,t} \asd h_{k_2,k_3,t})^*, \\
((f_{k_1-1,t} f_{k_1-1,t}^*) \otimes (b_{k_2,t} b_{k_2,t}^*) \otimes (b_{k_3,t} b_{k_3,t}^*))^{(1,3,5),(2,4,6)}) \\
= \frac{ \Vert f_{k_1-1,t} \Vert^2 \Vert b_{k_2,t} \Vert^2 \Vert b_{k_3,t} \Vert^2 \Vert f_{k_1-1,t}^{\asd 2} \Vert \Vert h_{k_2,k_3,t}^{\asd 2} \Vert }{ R(f_{k_1-1,t})^{\frac{1}{2}} R(b_{k_2,t})^{\frac{1}{2}} R(b_{k_3,t})^{\frac{1}{2}} R(f_{k_1-1,t}^{\asd 2})^{\frac{1}{4}} R(h_{k_2,k_3,t}^{\asd 2})^{\frac{1}{4}} } \\
\cos((f_{k_1-1,t} \asd h_{k_2,k_3,t}) (f_{k_1-1,t} \asd h_{k_2,k_3,t})^*, \\
((f_{k_1-1,t} f_{k_1-1,t}^*) \otimes (b_{k_2,t} b_{k_2,t}^*) \otimes (b_{k_3,t} b_{k_3,t}^*))^{(1,3,5),(2,4,6)}) \\
\cos((f_{k_1-1,t}^{\asd 2})^* (f_{k_1-1,t}^{\asd 2}), (h_{k_2,k_3,t}^{\asd 2})^* (h_{k_2,k_3,t}^{\asd 2}))^{\frac{1}{2}}
\end{multline*}
where we used that
\begin{multline*}
\Vert (f_{k_1-1,t} \asd h_{k_2,k_3,t}) (f_{k_1-1,t} \asd h_{k_2,k_3,t})^* \Vert^2
= \langle (f_{k_1-1,t}^{\asd 2})^* (f_{k_1-1,t}^{\asd 2}), (h_{k_2,k_3,t}^{\asd 2})^* (h_{k_2,k_3,t}^{\asd 2}) \rangle \\
= \Vert (f_{k_1-1,t}^{\asd 2})^* (f_{k_1-1,t}^{\asd 2}) \Vert \Vert (h_{k_2,k_3,t}^{\asd 2})^* (h_{k_2,k_3,t}^{\asd 2}) \Vert \cos((f_{k_1-1,t}^{\asd 2})^* (f_{k_1-1,t}^{\asd 2}), (h_{k_2,k_3,t}^{\asd 2})^* (h_{k_2,k_3,t}^{\asd 2})) \\
= \frac{ \Vert f_{k_1-1,t}^{\asd 2} \Vert^2 \Vert h_{k_2,k_3,t}^{\asd 2} \Vert^2 }{ R(f_{k_1-1,t}^{\asd 2})^{\frac{1}{2}} R(h_{k_2,k_3,t}^{\asd 2})^{\frac{1}{2}} } \cos((f_{k_1-1,t}^{\asd 2})^* (f_{k_1-1,t}^{\asd 2}), (h_{k_2,k_3,t}^{\asd 2})^* (h_{k_2,k_3,t}^{\asd 2}))
\end{multline*}
and the second norm as
\begin{multline*}
\Vert (f_{k_2-1,t} \asd j_{k_1,k_3-1,t})^* (b_{k_1,t} \otimes f_{k_2-1,t} \otimes f_{k_3-1,t})^{(3, 1, 5),(4, 2, 6)} \Vert^2 \\
= \langle (f_{k_2-1,t} \asd j_{k_1,k_3-1,t}) (f_{k_2-1,t} \asd j_{k_1,k_3-1,t})^*, \\
((b_{k_1,t} b_{k_1,t}^*) \otimes (f_{k_2-1,t} f_{k_2-1,t}^*) \otimes (f_{k_3-1,t} f_{k_3-1,t}^*))^{(3, 1, 5),(4, 2, 6)} \rangle \\
= \Vert f_{k_2-1,t} \asd j_{k_1,k_3-1,t}) (f_{k_2-1,t} \asd j_{k_1,k_3-1,t})^* \Vert \\
\Vert b_{k_1,t} b_{k_1,t}^* \Vert \Vert f_{k_2-1,t} f_{k_2-1,t}^* \Vert \Vert f_{k_3-1,t} f_{k_3-1,t}^* \Vert \\
\cos((f_{k_2-1,t} \asd j_{k_1,k_3-1,t}) (f_{k_2-1,t} \asd j_{k_1,k_3-1,t})^*, \\
((b_{k_1,t} b_{k_1,t}^*) \otimes (f_{k_2-1,t} f_{k_2-1,t}^*) \otimes (f_{k_3-1,t} f_{k_3-1,t}^*))^{(3, 1, 5),(4, 2, 6)}) \\
= \frac{ \Vert b_{k_1,t} \Vert^2 \Vert f_{k_2-1,t} \Vert^2 \Vert f_{k_3-1,t} \Vert^2 \Vert f_{k_2-1,t}^{\asd 2} \Vert^2 \Vert j_{k_1,k_3-1,t}^{\asd 2} \Vert^2 }{ R(b_{k_1,t})^{\frac{1}{2}} R(f_{k_2-1,t})^{\frac{1}{2}} R(f_{k_3-1,t})^{\frac{1}{2}} R(f_{k_2-1,t}^{\asd 2})^{\frac{1}{4}} R(j_{k_1,k_3-1,t}^{\asd 2})^{\frac{1}{4}} } \\
\cos((f_{k_2-1,t} \asd j_{k_1,k_3-1,t}) (f_{k_2-1,t} \asd j_{k_1,k_3-1,t})^*, \\
((b_{k_1,t} b_{k_1,t}^*) \otimes (f_{k_2-1,t} f_{k_2-1,t}^*) \otimes (f_{k_3-1,t} f_{k_3-1,t}^*))^{(3, 1, 5),(4, 2, 6)}) \\
\cos((f_{k_2-1,t}^{\asd 2})^* (f_{k_2-1,t}^{\asd 2}), (j_{k_1,k_3-1,t}^{\asd 2})^* (j_{k_1,k_3-1,t}^{\asd 2}))^{\frac{1}{2}}
\end{multline*}
where we used that
\begin{multline*}
\Vert (f_{k_2-1,t} \asd j_{k_1,k_3-1,t}) (f_{k_2-1,t} \asd j_{k_1,k_3-1,t})^* \Vert^2
= \langle (f_{k_2-1,t}^{\asd 2})^* (f_{k_2-1,t}^{\asd 2}), (j_{k_1,k_3-1,t}^{\asd 2})^* (j_{k_1,k_3-1,t}^{\asd 2}) \rangle \\
= \Vert (f_{k_2-1,t}^{\asd 2})^* (f_{k_2-1,t}^{\asd 2}) \Vert \Vert (j_{k_1,k_3-1,t}^{\asd 2})^* (j_{k_1,k_3-1,t}^{\asd 2}) \Vert \cos((f_{k_2-1,t}^{\asd 2})^* (f_{k_2-1,t}^{\asd 2}), (j_{k_1,k_3-1,t}^{\asd 2})^* (j_{k_1,k_3-1,t}^{\asd 2})) \\
= \frac{ \Vert f_{k_2-1,t}^{\asd 2} \Vert^2 \Vert j_{k_1,k_3-1,t}^{\asd 2} \Vert^2 }{ R(f_{k_2-1,t}^{\asd 2})^{\frac{1}{2}} R(j_{k_1,k_3-1,t}^{\asd 2})^{\frac{1}{2}} } \cos((f_{k_2-1,t}^{\asd 2})^* (f_{k_2-1,t}^{\asd 2}), (j_{k_1,k_3-1,t}^{\asd 2})^* (j_{k_1,k_3-1,t}^{\asd 2})).
\end{multline*}
We rewrite the last inner product as
\begin{multline*}
\langle (b_{k_3,t} \asd j_{k_1,k_2-1,t})^* (b_{k_1,t} \otimes f_{k_2-1,t} \otimes b_{k_3,t})^{(5,1,3),(6,2,4)}, \\
(f_{k_1-1,t} \asd j_{k_2,k_3-1,t})^* (f_{k_1-1,t} \otimes b_{k_2,t} \otimes f_{k_3-1,t})^{(1, 3, 5),(2, 4, 6)} \rangle \\
= \Vert (b_{k_3,t} \asd j_{k_1,k_2-1,t})^* (b_{k_1,t} \otimes f_{k_2-1,t} \otimes b_{k_3,t})^{(5,1,3),(6,2,4)} \Vert \\
\Vert (f_{k_1-1,t} \asd j_{k_2,k_3-1,t})^* (f_{k_1-1,t} \otimes b_{k_2,t} \otimes f_{k_3-1,t})^{(1, 3, 5),(2, 4, 6)} \Vert \\
\cos((b_{k_3,t} \asd j_{k_1,k_2-1,t})^* (b_{k_1,t} \otimes f_{k_2-1,t} \otimes b_{k_3,t})^{(5,1,3),(6,2,4)}, \\
(f_{k_1-1,t} \asd j_{k_2,k_3-1,t})^* (f_{k_1-1,t} \otimes b_{k_2,t} \otimes f_{k_3-1,t})^{(1, 3, 5),(2, 4, 6)}).
\end{multline*}
We then rewrite the norms as
\begin{multline*}
\Vert (b_{k_3,t} \asd j_{k_1,k_2-1,t})^* (b_{k_1,t} \otimes f_{k_2-1,t} \otimes b_{k_3,t})^{(5,1,3),(6,2,4)} \Vert^2 \\
= \langle (b_{k_3,t} \asd j_{k_1,k_2-1,t}) (b_{k_3,t} \asd j_{k_1,k_2-1,t})^*, \\
((b_{k_1,t} b_{k_1,t}^*) \otimes (f_{k_2-1,t} f_{k_2-1,t}^*) \otimes (b_{k_3,t} b_{k_3,t}^*))^{(5,1,3),(6,2,4)} \rangle \\
= \Vert (b_{k_3,t} \asd j_{k_1,k_2-1,t}) (b_{k_3,t} \asd j_{k_1,k_2-1,t})^* \Vert \\
\Vert b_{k_1,t} b_{k_1,t}^* \Vert \Vert f_{k_2-1,t} f_{k_2-1,t}^* \Vert \Vert b_{k_3,t} b_{k_3,t}^* \Vert \\
\cos((b_{k_3,t} \asd j_{k_1,k_2-1,t}) (b_{k_3,t} \asd j_{k_1,k_2-1,t})^*, \\
((b_{k_1,t} b_{k_1,t}^*) \otimes (f_{k_2-1,t} f_{k_2-1,t}^*) \otimes (b_{k_3,t} b_{k_3,t}^*))^{(5,1,3),(6,2,4)}) \\
= \frac{ \Vert b_{k_1,t} \Vert^2 \Vert f_{k_2-1,t} \Vert^2 \Vert b_{k_3,t} \Vert^2 \Vert b_{k_3,t}^{\asd 2} \Vert \Vert j_{k_1,k_2-1,t}^{\asd 2} \Vert }{ R(b_{k_1,t})^{\frac{1}{2}} R(f_{k_2-1,t})^{\frac{1}{2}} R(b_{k_3,t})^{\frac{1}{2}} R(b_{k_3,t}^{\asd 2})^{\frac{1}{4}} R(j_{k_1,k_2-1,t}^{\asd 2})^{\frac{1}{4}} } \\
\cos((b_{k_3,t} \asd j_{k_1,k_2-1,t}) (b_{k_3,t} \asd j_{k_1,k_2-1,t})^*, \\
((b_{k_1,t} b_{k_1,t}^*) \otimes (f_{k_2-1,t} f_{k_2-1,t}^*) \otimes (b_{k_3,t} b_{k_3,t}^*))^{(5,1,3),(6,2,4)}) \\
\cos((b_{k_3,t}^{\asd 2})^* (b_{k_3,t}^{\asd 2}), (j_{k_1,k_2-1,t}^{\asd 2})^* (j_{k_1,k_2-1,t}^{\asd 2}))^{\frac{1}{2}}
\end{multline*}
where we used that
\begin{multline*}
\Vert (b_{k_3,t} \asd j_{k_1,k_2-1,t}) (b_{k_3,t} \asd j_{k_1,k_2-1,t})^* \Vert^2 
= \langle (b_{k_3,t}^{\asd 2})^* (b_{k_3,t}^{\asd 2}), (j_{k_1,k_2-1,t}^{\asd 2})^* (j_{k_1,k_2-1,t}^{\asd 2}) \rangle \\
= \Vert (b_{k_3,t}^{\asd 2})^* (b_{k_3,t}^{\asd 2}) \Vert \Vert (j_{k_1,k_2-1,t}^{\asd 2})^* (j_{k_1,k_2-1,t}^{\asd 2}) \Vert \cos((b_{k_3,t}^{\asd 2})^* (b_{k_3,t}^{\asd 2}), (j_{k_1,k_2-1,t}^{\asd 2})^* (j_{k_1,k_2-1,t}^{\asd 2})) \\
= \frac{ \Vert b_{k_3,t}^{\asd 2} \Vert^2 \Vert j_{k_1,k_2-1,t}^{\asd 2} \Vert^2 }{ R(b_{k_3,t}^{\asd 2})^{\frac{1}{2}} R(j_{k_1,k_2-1,t}^{\asd 2})^{\frac{1}{2}} } \cos((b_{k_3,t}^{\asd 2})^* (b_{k_3,t}^{\asd 2}), (j_{k_1,k_2-1,t}^{\asd 2})^* (j_{k_1,k_2-1,t}^{\asd 2}))
\end{multline*}
and as
\begin{multline*}
\Vert (f_{k_1-1,t} \asd j_{k_2,k_3-1,t})^* (f_{k_1-1,t} \otimes b_{k_2,t} \otimes f_{k_3-1,t})^{(1, 3, 5),(2, 4, 6)} \Vert^2 \\
= \langle (f_{k_1-1,t} \asd j_{k_2,k_3-1,t}) (f_{k_1-1,t} \asd j_{k_2,k_3-1,t})^*, \\
((f_{k_1-1,t} f_{k_1-1,t}^*) \otimes (b_{k_2,t} b_{k_2,t}^*) \otimes (f_{k_3-1,t} f_{k_3-1,t}^*))^{(1, 3, 5),(2, 4, 6)} \rangle \\
= \Vert (f_{k_1-1,t} \asd j_{k_2,k_3-1,t}) (f_{k_1-1,t} \asd j_{k_2,k_3-1,t})^* \Vert \\
\Vert f_{k_1-1,t} f_{k_1-1,t}^* \Vert \Vert b_{k_2,t} b_{k_2,t}^* \Vert \Vert f_{k_3-1,t} f_{k_3-1,t}^* \Vert \\
\cos((f_{k_1-1,t} \asd j_{k_2,k_3-1,t}) (f_{k_1-1,t} \asd j_{k_2,k_3-1,t})^*, \\
((f_{k_1-1,t} f_{k_1-1,t}^*) \otimes (b_{k_2,t} b_{k_2,t}^*) \otimes (f_{k_3-1,t} f_{k_3-1,t}^*))^{(1, 3, 5),(2, 4, 6)}) \\
= \frac{ \Vert f_{k_1-1,t} \Vert^2 \Vert b_{k_2,t} \Vert^2 \Vert f_{k_3-1,t} \Vert^2 \Vert f_{k_1-1,t}^{\asd 2} \Vert \Vert j_{k_2,k_3-1,t}^{\asd 2} \Vert }{ R(f_{k_1-1,t})^{\frac{1}{2}} R(b_{k_2,t})^{\frac{1}{2}} R(f_{k_3-1,t})^{\frac{1}{2}} R(f_{k_1-1,t}^{\asd 2})^{\frac{1}{4}} R(j_{k_2,k_3-1,t}^{\asd 2})^{\frac{1}{4}} } \\
\cos((f_{k_1-1,t} \asd j_{k_2,k_3-1,t}) (f_{k_1-1,t} \asd j_{k_2,k_3-1,t})^*, \\
((f_{k_1-1,t} f_{k_1-1,t}^*) \otimes (b_{k_2,t} b_{k_2,t}^*) \otimes (f_{k_3-1,t} f_{k_3-1,t}^*))^{(1, 3, 5),(2, 4, 6)}) \\
\cos((f_{k_1-1,t}^{\asd 2})^* (f_{k_1-1,t}^{\asd 2}), (j_{k_2,k_3-1,t}^{\asd 2})^* (j_{k_2,k_3-1,t}^{\asd 2}))^{\frac{1}{2}}
\end{multline*}
where we used that
\begin{multline*}
\Vert f_{k_1-1,t} \asd j_{k_2,k_3-1,t}) (f_{k_1-1,t} \asd j_{k_2,k_3-1,t})^* \Vert^2
= \langle (f_{k_1-1,t}^{\asd 2})^* (f_{k_1-1,t}^{\asd 2}), (j_{k_2,k_3-1,t}^{\asd 2})^* (j_{k_2,k_3-1,t}^{\asd 2}) \rangle \\
= \Vert (f_{k_1-1,t}^{\asd 2})^* (f_{k_1-1,t}^{\asd 2}) \Vert \Vert (j_{k_2,k_3-1,t}^{\asd 2})^* (j_{k_2,k_3-1,t}^{\asd 2}) \Vert \cos((f_{k_1-1,t}^{\asd 2})^* (f_{k_1-1,t}^{\asd 2}), (j_{k_2,k_3-1,t}^{\asd 2})^* (j_{k_2,k_3-1,t}^{\asd 2})) \\
= \frac{ \Vert f_{k_1-1,t}^{\asd 2} \Vert^2 \Vert j_{k_2,k_3-1,t}^{\asd 2} \Vert^2 }{ R(f_{k_1-1,t}^{\asd 2})^{\frac{1}{2}} R(j_{k_2,k_3-1,t}^{\asd 2})^{\frac{1}{2}} } \cos((f_{k_1-1,t}^{\asd 2})^* (f_{k_1-1,t}^{\asd 2}), (j_{k_2,k_3-1,t}^{\asd 2})^* (j_{k_2,k_3-1,t}^{\asd 2})).
\end{multline*}

Summarizing the above, we then have
\[
\langle (\Xi_t \nabla_t)^\otimes 3, \partial\partial\nabla_t \rangle
= \sum_{k_1=1}^{l+1} \sum_{k_2=1}^{l+1} \sum_{k_3=1}^{l+1} \xi_{k_1,t} \xi_{k_2,t} \xi_{k_3,t} \langle \nabla_{k_1,t} \otimes \nabla_{k_2,t} \otimes \nabla_{k_3,t}, \partial\partial\nabla_{k_1,k_2,k_3,t} \rangle
\]
with
\begin{multline*}
\langle \nabla_{k_1,t} \otimes \nabla_{k_2,t} \otimes \nabla_{k_3,t}, \partial\partial\nabla_{k_1,k_2,k_3,t} \rangle \\
= \frac{ \Vert f_{k_1-1,t} \Vert \Vert f_{k_2-1,t} \Vert \Vert f_{k_3-1,t} \Vert \Vert b_{k_1,t} \Vert \Vert b_{k_2,t} \Vert \Vert b_{k_3,t} \Vert }{ R(f_{k_1-1,t})^{\frac{1}{4}} R(f_{k_2-1,t})^{\frac{1}{4}} R(f_{k_3-1,t})^{\frac{1}{4}} R(b_{k_1,t})^{\frac{1}{4}} R(b_{k_2,t})^{\frac{1}{4}} R(b_{k_3,t})^{\frac{1}{4}} } \\
\frac{ \Vert f_{k_1-1,t}^{\asd 3} \Vert^{\frac{1}{3}} \Vert f_{k_2-1,t}^{\asd 3} \Vert^{\frac{1}{3}} \Vert f_{k_3-1,t}^{\asd 3} \Vert^{\frac{1}{3}} \Vert z_{k_1,k_2,k_3,t} \Vert }{ R(f_{k_1-1,t}^{\asd 3})^{\frac{1}{12}} R(f_{k_2-1,t}^{\asd 3})^{\frac{1}{12}} R(f_{k_3-1,t}^{\asd 3})^{\frac{1}{12}} R(z_{k_1,k_2,k_3,t})^{\frac{1}{4}} } \\
\cos(z_{k_1,k_2,k_3,t}^* (b_{k_1,t} \otimes b_{k_2,t} \otimes b_{k_3,t})^{(1,3,5),(2,4,6)}, \\
(f_{k_1-1,t} \asd f_{k_2-1,t} \asd f_{k_3-1,t})^* (f_{k_1-1,t} \otimes f_{k_2-1,t} \otimes f_{k_3-1,t})^{(1,3,5),(2,4,6)}) \\
\cos(z_{k_1,k_2,k_3,t} z_{k_1,k_2,k_3,t}^*, ((b_{k_1,t} b_{k_1,t}^*) \otimes (b_{k_2,t} b_{k_2,t}^*) \otimes (b_{k_3,t} b_{k_3,t}^*))^{(1,3,5),(2,4,6)})^{\frac{1}{2}} \\
\cos((f_{k_1-1,t} \asd f_{k_2-1,t} \asd f_{k_3-1,t}) (f_{k_1-1,t} \asd f_{k_2-1,t} \asd f_{k_3-1,t})^*, \\ 
((f_{k_1-1,t} f_{k_1-1,t}^*) \otimes (f_{k_2-1,t} f_{k_2-1,t}^*) \otimes (f_{k_3-1,t} f_{k_3-1,t}^*))^{(1,3,5),(2,4,6)})^{\frac{1}{2}} \\
\cos((f_{k_1-1,t}^{\asd 2})^* (f_{k_1-1,t}^{\asd 2}), (f_{k_2-1,t}^{\asd 2})^* (f_{k_2-1,t}^{\asd 2}), (f_{k_3-1,t}^{\asd 2})^* (f_{k_3-1,t}^{\asd 2}))^{\frac{1}{4}} 
\end{multline*}
\begin{multline*}
+ \frac{ \Vert f_{k_1-1,t} \Vert \Vert f_{k_2-1,t} \Vert \Vert f_{k_3-1,t} \Vert \Vert b_{k_1,t} \Vert \Vert b_{k_2,t} \Vert \Vert b_{k_3,t} \Vert }{ R(f_{k_1-1,t})^{\frac{1}{4}} R(f_{k_2-1,t})^{\frac{1}{4}} R(f_{k_3-1,t})^{\frac{1}{4}} R(b_{k_1,t})^{\frac{1}{4}} R(b_{k_2,t})^{\frac{1}{4}} R(b_{k_3,t})^{\frac{1}{4}} } \\
\frac{ \Vert f_{k_1-1,t}^{\asd 2} \Vert^{\frac{1}{2}} \Vert f_{k_2-1,t}^{\asd 2} \Vert^{\frac{1}{2}} \Vert h_{k_2,k_3,t}^{\asd 2} \Vert^{\frac{1}{2}} \Vert j_{k_1,k_3-1,t}^{\asd 2} \Vert^{\frac{1}{2}} }{ R(f_{k_1-1,t}^{\asd 2})^{\frac{1}{8}} R(f_{k_2-1,t}^{\asd 2})^{\frac{1}{8}} R(h_{k_2,k_3,t}^{\asd 2})^{\frac{1}{8}} R(j_{k_1,k_3-1,t}^{\asd 2})^{\frac{1}{8}} } \\
\cos( (f_{k_1-1,t} \asd h_{k_2,k_3,t})^* (f_{k_1-1,t} \otimes b_{k_2,t} \otimes b_{k_3,t})^{(1,3,5),(2,4,6)}, \\
(f_{k_2-1,t} \asd j_{k_1,k_3-1,t})^* (b_{k_1,t} \otimes f_{k_2-1,t} \otimes f_{k_3-1,t})^{(3, 1, 5),(4, 2, 6)}) \\
\cos((f_{k_1-1,t} \asd h_{k_2,k_3,t}) (f_{k_1-1,t} \asd h_{k_2,k_3,t})^*, \\
((f_{k_1-1,t} f_{k_1-1,t}^*) \otimes (b_{k_2,t} b_{k_2,t}^*) \otimes (b_{k_3,t} b_{k_3,t}^*))^{(1,3,5),(2,4,6)})^{\frac{1}{2}} \\
\cos((f_{k_2-1,t} \asd j_{k_1,k_3-1,t}) (f_{k_2-1,t} \asd j_{k_1,k_3-1,t})^*, \\
((b_{k_1,t} b_{k_1,t}^*) \otimes (f_{k_2-1,t} f_{k_2-1,t}^*) \otimes (f_{k_3-1,t} f_{k_3-1,t}^*))^{(3, 1, 5),(4, 2, 6)})^{\frac{1}{2}} \\
\cos((f_{k_1-1,t}^{\asd 2})^* (f_{k_1-1,t}^{\asd 2}), (h_{k_2,k_3,t}^{\asd 2})^* (h_{k_2,k_3,t}^{\asd 2}))^{\frac{1}{4}} \\
\cos((f_{k_2-1,t}^{\asd 2})^* (f_{k_2-1,t}^{\asd 2}), (j_{k_1,k_3-1,t}^{\asd 2})^* (j_{k_1,k_3-1,t}^{\asd 2}))^{\frac{1}{4}} 
\end{multline*}
\begin{multline*}
+ \frac{ \Vert f_{k_1-1,t} \Vert \Vert f_{k_2-1,t} \Vert \Vert f_{k_3-1,t} \Vert \Vert b_{k_1,t} \Vert \Vert b_{k_2,t} \Vert \Vert b_{k_3,t} \Vert }{ R(f_{k_1-1,t})^{\frac{1}{4}} R(f_{k_2-1,t})^{\frac{1}{4}} R(f_{k_3-1,t})^{\frac{1}{4}} R(b_{k_1,t})^{\frac{1}{4}} R(b_{k_2,t})^{\frac{1}{4}} R(b_{k_3,t})^{\frac{1}{4}} } \\
\frac{ \Vert f_{k_1-1,t}^{\asd 2} \Vert^{\frac{1}{2}} \Vert f_{k_3-1,t}^{\asd 2} \Vert^{\frac{1}{2}} \Vert h_{k_2,k_3,t}^{\asd 2} \Vert^{\frac{1}{2}} \Vert j_{k_1,k_2-1,t}^{\asd 2} \Vert^{\frac{1}{2}} }{ R(f_{k_1-1,t}^{\asd 2})^{\frac{1}{8}} R(f_{k_3-1,t}^{\asd 2})^{\frac{1}{8}} R(h_{k_2,k_3,t}^{\asd 2})^{\frac{1}{8}} R(j_{k_1,k_2-1,t}^{\asd 2})^{\frac{1}{8}} } \\
\cos( (f_{k_1-1,t} \asd h_{k_2,k_3,t})^* (f_{k_1-1,t} \otimes b_{k_2,t} \otimes b_{k_3,t})^{(1,3,5),(2,4,6)}, \\
(f_{k_3-1,t} \asd j_{k_1,k_2-1,t})^* (b_{k_1,t} \otimes f_{k_2-1,t} \otimes f_{k_3-1,t})^{(5, 1, 3),(6, 2, 4)}) \\
\cos((f_{k_1-1,t} \asd h_{k_2,k_3,t}) (f_{k_1-1,t} \asd h_{k_2,k_3,t})^*, \\
((f_{k_1-1,t} f_{k_1-1,t}^*) \otimes (b_{k_2,t} b_{k_2,t}^*) \otimes (b_{k_3,t} b_{k_3,t}^*))^{(1,3,5),(2,4,6)})^{\frac{1}{2}} \\
\cos((f_{k_3-1,t} \asd j_{k_1,k_2-1,t}) (f_{k_3-1,t} \asd j_{k_1,k_2-1,t})^*, \\
((b_{k_1,t} b_{k_1,t}^*) \otimes (f_{k_2-1,t} f_{k_2-1,t}^*) \otimes (f_{k_3-1,t} f_{k_3-1,t}^*))^{(3, 1, 5),(4, 2, 6)})^{\frac{1}{2}} \\
\cos((f_{k_1-1,t}^{\asd 2})^* (f_{k_1-1,t}^{\asd 2}), (h_{k_2,k_3,t}^{\asd 2})^* (h_{k_2,k_3,t}^{\asd 2}))^{\frac{1}{4}} \\
\cos((f_{k_3-1,t}^{\asd 2})^* (f_{k_3-1,t}^{\asd 2}), (j_{k_1,k_2-1,t}^{\asd 2})^* (j_{k_1,k_2-1,t}^{\asd 2}))^{\frac{1}{4}} 
\end{multline*}
\begin{multline*}
+ \frac{ \Vert f_{k_1-1,t} \Vert \Vert f_{k_2-1,t} \Vert \Vert f_{k_3-1,t} \Vert \Vert b_{k_1,t} \Vert \Vert b_{k_2,t} \Vert \Vert b_{k_3,t} \Vert }{ R(f_{k_1-1,t})^{\frac{1}{4}} R(f_{k_2-1,t})^{\frac{1}{4}} R(f_{k_3-1,t})^{\frac{1}{4}} R(b_{k_1,t})^{\frac{1}{4}} R(b_{k_2,t})^{\frac{1}{4}} R(b_{k_3,t})^{\frac{1}{4}} } \\
\frac{ \Vert f_{k_1-1,t}^{\asd 2} \Vert^{\frac{1}{2}} \Vert f_{k_2-1,t}^{\asd 2} \Vert^{\frac{1}{2}} \Vert h_{k_1,k_3,t}^{\asd 2} \Vert^{\frac{1}{2}} \Vert j_{k_2,k_3-1,t}^{\asd 2} \Vert^{\frac{1}{2}} }{ R(f_{k_1-1,t}^{\asd 2})^{\frac{1}{8}} R(f_{k_2-1,t}^{\asd 2})^{\frac{1}{8}} R(h_{k_1,k_3,t}^{\asd 2})^{\frac{1}{8}} R(j_{k_2,k_3-1,t}^{\asd 2})^{\frac{1}{8}} } \\
\cos( (f_{k_2-1,t} \asd h_{k_1,k_3,t})^* (b_{k_1,t} \otimes f_{k_2-1,t} \otimes b_{k_3,t})^{(3,1,5),(4,2,6)}, \\
(f_{k_1-1,t} \asd j_{k_2,k_3-1,t})^* (f_{k_1-1,t} \otimes b_{k_2,t} \otimes f_{k_3-1,t})^{(1, 3, 5),(2, 4, 6)}) \\
\cos((f_{k_2-1,t} \asd h_{k_1,k_3,t}) (f_{k_2-1,t} \asd h_{k_1,k_3,t})^*, \\
((b_{k_1,t} b_{k_1,t}^*) \otimes (f_{k_2-1,t} f_{k_2-1,t}^*) \otimes (b_{k_3,t} b_{k_3,t}^*))^{(1,3,5),(2,4,6)})^{\frac{1}{2}} \\
\cos((f_{k_1-1,t} \asd j_{k_2,k_3-1,t}) (f_{k_1-1,t} \asd j_{k_2,k_3-1,t})^*, \\
((f_{k_1-1,t} f_{k_1-1,t}^*) \otimes (b_{k_2,t} b_{k_2,t}^*) \otimes (f_{k_3-1,t} f_{k_3-1,t}^*))^{(3, 1, 5),(4, 2, 6)})^{\frac{1}{2}} \\
\cos((f_{k_2-1,t}^{\asd 2})^* (f_{k_2-1,t}^{\asd 2}), (h_{k_1,k_3,t}^{\asd 2})^* (h_{k_1,k_3,t}^{\asd 2}))^{\frac{1}{4}} \\
\cos((f_{k_1-1,t}^{\asd 2})^* (f_{k_1-1,t}^{\asd 2}), (j_{k_2,k_3-1,t}^{\asd 2})^* (j_{k_2,k_3-1,t}^{\asd 2}))^{\frac{1}{4}} 
\end{multline*}
\begin{multline*}
+ \frac{ \Vert f_{k_1-1,t} \Vert \Vert f_{k_2-1,t} \Vert \Vert f_{k_3-1,t} \Vert \Vert b_{k_1,t} \Vert \Vert b_{k_2,t} \Vert \Vert b_{k_3,t} \Vert }{ R(f_{k_1-1,t})^{\frac{1}{4}} R(f_{k_2-1,t})^{\frac{1}{4}} R(f_{k_3-1,t})^{\frac{1}{4}} R(b_{k_1,t})^{\frac{1}{4}} R(b_{k_2,t})^{\frac{1}{4}} R(b_{k_3,t})^{\frac{1}{4}} } \\
\frac{ \Vert b_{k_3,t}^{\asd 2} \Vert^{\frac{1}{2}} \Vert j_{k_1,k_2-1,t}^{\asd 2} \Vert^{\frac{1}{2}} \Vert f_{k_1-1,t}^{\asd 2} \Vert^{\frac{1}{2}} \Vert j_{k_2,k_3-1,t}^{\asd 2} \Vert^{\frac{1}{2}} }{ R(b_{k_3,t}^{\asd 2})^{\frac{1}{8}} R(j_{k_1,k_2-1,t}^{\asd 2})^{\frac{1}{8}} R(f_{k_1-1,t}^{\asd 2})^{\frac{1}{8}} R(j_{k_2,k_3-1,t}^{\asd 2})^{\frac{1}{8}} } \\
\cos((b_{k_3,t} \asd j_{k_1,k_2-1,t})^* (b_{k_1,t} \otimes f_{k_2-1,t} \otimes b_{k_3,t})^{(5,1,3),(6,2,4)}, \\
(f_{k_1-1,t} \asd j_{k_2,k_3-1,t})^* (f_{k_1-1,t} \otimes b_{k_2,t} \otimes f_{k_3-1,t})^{(1, 3, 5),(2, 4, 6)}) \\
\cos((b_{k_3,t} \asd j_{k_1,k_2-1,t}) (b_{k_3,t} \asd j_{k_1,k_2-1,t})^*, \\
((b_{k_1,t} b_{k_1,t}^*) \otimes (f_{k_2-1,t} f_{k_2-1,t}^*) \otimes (b_{k_3,t} b_{k_3,t}^*))^{(5,1,3),(6,2,4)})^{\frac{1}{2}} \\
\cos((f_{k_1-1,t} \asd j_{k_2,k_3-1,t}) (f_{k_1-1,t} \asd j_{k_2,k_3-1,t})^*, \\
((f_{k_1-1,t} f_{k_1-1,t}^*) \otimes (b_{k_2,t} b_{k_2,t}^*) \otimes (f_{k_3-1,t} f_{k_3-1,t}^*))^{(1, 3, 5),(2, 4, 6)})^{\frac{1}{2}} \\
\cos((b_{k_3,t}^{\asd 2})^* (b_{k_3,t}^{\asd 2}), (j_{k_1,k_2-1,t}^{\asd 2})^* (j_{k_1,k_2-1,t}^{\asd 2}))^{\frac{1}{4}} \\
\cos((f_{k_1-1,t}^{\asd 2})^* (f_{k_1-1,t}^{\asd 2}), (j_{k_2,k_3-1,t}^{\asd 2})^* (j_{k_2,k_3-1,t}^{\asd 2}))^{\frac{1}{4}} 
\end{multline*}
The rest follows analogously.
\end{proof}

\begin{proof}[Proof of Theorem~\ref{thm:taylor_coeffs}]
The result is immediate from Proposition~\ref{prop:first_order_term}, Proposition~\ref{prop:second_order_term} and Proposition~\ref{prop:third_order_term} after rewriting the formulas in terms of the Gram matrices and the other tensors consisting of inner products.
\end{proof}

\end{document}